\pdfoutput=1
\def\arxivversion{}

\documentclass[10pt]{article}

\ifdefined\arxivversion
\usepackage[preprint]{tmlr}
\else
\usepackage{tmlr}
\fi

\usepackage{amsmath}

\usepackage{url}

\usepackage[utf8]{inputenc} 
\usepackage[T1]{fontenc}    
\usepackage{booktabs}       
\usepackage{amsfonts}       
\usepackage{microtype}      
\usepackage{xcolor}         

\usepackage{graphicx}
\usepackage{bm}

\usepackage{amsthm}
\usepackage{caption}
\usepackage{tikz}
\usepackage{array}
\usepackage{flafter}
\usepackage[below]{placeins}   
\usepackage{pifont}
\newcommand{\cmark}{\textcolor{green!45!black}{\ding{51}}}
\newcommand{\xmark}{\textcolor{red!70!black}{\ding{55}}}
\makeatletter
\newcommand{\codeavailability}{\if@accepted The code, the trained policies, and a notebook that reproduces all figures and tables are available at \url{https://github.com/danialsj77/Dissimilarity-Based-MDP}.\else The code, the trained policies, and a notebook that reproduces all figures and tables are provided in the supplementary material and will be made publicly available.\fi}
\makeatother
\usetikzlibrary{arrows.meta,positioning,shapes.geometric,fit,calc,backgrounds}
\usepackage{amsmath,amssymb}
\usepackage{mathtools}
\usepackage{thmtools}
\usepackage{thm-restate}
\usepackage[hypertexnames=false]{hyperref}
\hypersetup{colorlinks=true, citecolor=blue, linkcolor=blue, urlcolor=black}
\usepackage{etoc}
\makeatletter
\def\addcontentsline#1#2#3{\addtocontents{#1}{\protect\contentsline{#2}{#3}{\thepage}{\@currentHref}}}
\makeatother
\etocsettagdepth{mainmatter}{none}
\etocsettagdepth{appendix}{subsubsection}
\etocsettocstyle{\section*{Contents of the Appendices}}{}
\makeatletter
\DeclareRobustCommand{\eqref}[1]{\begingroup\hypersetup{linkcolor=blue}\textcolor{blue}{\textup{\tagform@{\ref{#1}}}}\endgroup}
\makeatother
\newcommand{\namedref}[2]{\hyperref[#2]{#1~\ref*{#2}}}

\newtheorem{mydef}{Definition}

\newtheorem{mycoro}{Corollary}

\newtheorem{mylemma}{Lemma}

\newtheorem{myassump}{Assumption}

\title{Dissimilarity-Based Markov Decision Processes for Stabilizing State Distributions}

\ifdefined\arxivversion
\author{\name Xun Shen$^{*}$ \email shen@go.tuat.ac.jp \\
       \addr Tokyo University of Agriculture and Technology, Japan
       \AND
       \name Yuepeng Wang$^{*}$ \email wangadam38@gmail.com \\
       \addr Tokyo University of Agriculture and Technology, Japan
       \AND
       \name Daniyal Sajedi-Hosseini$^{*\dagger}$ \email danialsj77@gmail.com, s265021u@st.go.tuat.ac.jp \\
       \addr Tokyo University of Agriculture and Technology, Japan
       \AND
       \name Akifumi Wachi \email akifumi.wachi@lycorp.co.jp \\
       \addr LY Corporation, Japan
       \AND
       \name Ken Kawano \email kawano.k.4e8f@m.isct.ac.jp \\
       \addr Institute of Science Tokyo, Japan
       \AND
       \name Yoshihiko Fujisawa \email fujisawa.y.5bdc@m.isct.ac.jp \\
       \addr Institute of Science Tokyo, Japan
       \AND
       \name Mehrshad Sadria \email mehrshad.sadria@uwaterloo.ca \\
       \addr University of Waterloo, Canada
       \AND
       \name S{\'e}bastien Gros \email sebastien.gros@ntnu.no \\
       \addr Norwegian University of Science and Technology, Norway
       \AND
       \name Katsuki Fujisawa \email fujisawa.k.2110@m.isct.ac.jp \\
       \addr Institute of Science Tokyo, Japan}
\else
\author{\name Anonymous authors}
\fi

\begin{document}
\etocdepthtag.toc{mainmatter}

\maketitle
\ifdefined\arxivversion
{\renewcommand{\thefootnote}{\fnsymbol{footnote}}%
\footnotetext[1]{Equal contribution.}%
\footnotetext[2]{Corresponding author.}}
\fi

\begin{abstract}
Reinforcement learning for stochastic systems typically optimizes expected cumulative rewards, while reliable and repeatable behavior also depends on how the state distribution evolves over time.
Optimizing expected return alone does not guarantee that this distribution converges to a desired target or remains stable.
This paper develops a theoretical framework for learning policies that stabilize state distributions.
We formulate the problem as a Dissimilarity-based Markov Decision Process (DMDP), in which the policy acts on the current state distribution and the objective measures its dissimilarity to a target distribution.
We establish sufficient conditions under which the optimal policy makes the target distribution asymptotically stable.
To support policy learning, we introduce a finite-dimensional representation of distributions and show how stability guarantees for its parameter dynamics extend to the underlying state distributions.
We further develop certificates for the learned policies, including their deployment using estimated distribution parameters.
These results quantify how learning and estimation errors affect the guarantees and identify conditions under which learned policies retain asymptotic stability.
For practical implementation, we combine distribution estimates and target-dependent rewards with PPO and SAC. 
A computable example on a linear system illustrates the proposed theoretical certificates.
Furthermore, experiments on a HalfCheetah task demonstrate the benefits of distribution feedback and explicit target matching for regulating the velocity distribution.
\end{abstract}

\section{Introduction}
\label{sec:introduction}

Reinforcement learning for stochastic systems raises a stability question beyond the optimization of expected return: under what conditions does a learned policy keep the evolving state distribution close to a desired distribution, for example, a distribution concentrated on a high-reward region, and drive it toward that target?
This question matters when the desired behavior concerns the variability and concentration of states across repeated executions.
For example, a robot's mean velocity alone does not determine how often it exceeds a prescribed velocity limit.
Conventional Markov decision processes (MDPs) and reinforcement learning optimize expected cumulative rewards or costs~\citep{Puterman1994,Sutton2018}, but the resulting expected return can leave important features of individual trajectories unresolved~\citep{Baumann2025}.
Moreover, even convergence of an expected stage cost need not imply convergence of the state distribution: distinct distributions can have the same expected cost while the closed-loop distribution continues to oscillate.
A distributional stability guarantee therefore requires conditions that connect the learning objective and the feedback policy to the evolution of the distribution itself.

Several research directions address parts of this problem.
Measure-functional MDPs formulate costs and stability on the space of probability measures, while mean-field formulations incorporate distributions into the decision state.
Dissipativity and turnpike notions have been extended to random variables and probability measures~\citep{Schiessl2025TAC,Schiessl2025SICON} and yield closed-loop stability and performance guarantees for stochastic economic model predictive control~\citep{Schiessl2024CDC,Schiessl2025arXiv}.
Lyapunov-based learning control provides certificates for learned policies, and distribution steering specifies a desired terminal distribution.
Distributional reinforcement learning models the distribution of returns, while state-distribution objectives concern distributions of visited states.
These perspectives motivate a more specific theoretical objective: to connect feedback from the current state distribution with infinite-horizon distributional stability, and to characterize when this connection survives learning and estimation errors. 
\namedref{Table}{tab:related_work_comparison} summarizes the relationship to the cited approaches along six criteria.
The first three concern a stage cost that depends nonlinearly on the state distribution (distribution cost), a policy that acts on the distribution or its parameterization (distribution feedback), and a distributional stability notion with sufficient conditions (distributional stability).
The remaining criteria concern certificates for learned policies under learning, model, and estimation errors (certified learned policies), applicability to nonlinear dynamics, and an infinite-horizon objective or stability property.
Under the definitions detailed in \namedref{Appendix}{appendix:comparison_state_feedback}, the DMDP framework combines all six properties.
Its central contribution is the accompanying theory connecting distribution feedback, optimal-policy stability, and certificates for policies obtained from data.

\begin{table}[t]
    \centering
    \caption{Comparison with related approaches (\cmark: yes; (\cmark): not adopted but not excluded by the theory; \xmark: no).}
    \label{tab:related_work_comparison}
    \small
    \renewcommand{\arraystretch}{1.15}
    \setlength{\tabcolsep}{4.25pt}
    \begin{tabular}{@{}>{\raggedright\arraybackslash}p{0.33\linewidth}cccccc@{}}
        \toprule
        Approach & \shortstack{Distribution\\cost} & \shortstack{Distribution\\feedback} & \shortstack{Distributional\\stability} & \shortstack{Certified\\learned\\policies} & \shortstack{Nonlinear\\dynamics} & \shortstack{Infinite\\horizon} \\
        \midrule
        Norm-based MDPs~\citep{Puterman1994,Sutton2018} & \xmark & \xmark & \xmark & \xmark & \cmark & \cmark \\
        Lyapunov-based learning control~\citep{Perkins2002,Berkenkamp2017,Chang2019} & \xmark & \xmark & \xmark & \cmark & \cmark & \cmark \\
        Measure-functional MDPs~\citep{Gros_2022,Kordabad2022} & \cmark & (\cmark) & \cmark & \xmark & \cmark & \cmark \\
        Mean-field MDPs and control~\citep{Motte2022,Bauerle2023,Carmona2023} & \cmark & \cmark & \xmark & \xmark & \cmark & \cmark \\
        Finite-horizon distribution steering~\citep{Chen2016,Balci2021} & \cmark & \xmark & \xmark & \xmark & \xmark & \xmark \\
        Belief-state MDPs~\citep{Kaelbling1998Planning} & \xmark & \cmark & \xmark & \xmark & \cmark & \cmark \\
        Distributional RL~\citep{Bellemare2017,Bellemare2023} & \xmark & \xmark & \xmark & \xmark & \cmark & \cmark \\
        State-distribution objectives~\citep{Hazan2019,Zahavy2021} & \cmark & \xmark & \xmark & \xmark & \cmark & \cmark \\
        \midrule
        DMDP (this paper) & \cmark & \cmark & \cmark & \cmark & \cmark & \cmark \\
        \bottomrule
    \end{tabular}
\end{table}

We formulate this problem as a Dissimilarity-based MDP (DMDP).
The decision state is the current probability measure, the policy maps that measure to an action, and the stage cost can directly quantify its dissimilarity to a target distribution.
For a fixed transition kernel, the distribution evolves deterministically once the current distribution and action are specified.
This observation makes it possible to analyze distributional stabilization through deterministic dynamics on the space of probability measures.
The resulting problem nevertheless presents two difficulties: its state space is infinite-dimensional, and both the distribution and its evolution must generally be inferred from finite data.

Our theoretical development addresses these difficulties through finite-dimensional representations and Lyapunov certificates.
The state distribution $\rho_k$ at time $k$ is represented by a member $\rho_{\omega_k}$ of a parameterized family with a finite-dimensional parameter $\omega_k$.
The deterministic dynamics of $\omega_k$ are exact when the family is closed under the controlled transition and approximate otherwise.
A compatibility condition between parameter distance and distributional dissimilarity then transfers stability from the parameter dynamics to the induced distributions.
For learned policies, the key issue is whether a certified decrease survives the errors introduced by learning and deployment.
Our results express this relationship through inequalities of the form $W(\omega_{k+1})-W(\omega_k)
    \leq
    -\sigma\!\left(\|\omega_k-\omega_\star\|\right)
    +\overline{\Delta}$,
where $W$ is a Lyapunov function on the parameter space, $\omega_\star$ is the parameter of the target distribution, $\sigma$ is a class-$\mathcal K$ decrease margin, and $\overline{\Delta}$ collects the relevant error bounds.
This inequality identifies where the decrease remains certified.
Asymptotic distributional stability follows when the residual is zero or when state-dependent errors leave a strictly positive decrease margin away from the target.
With nonzero uniform residuals, the guarantee is a Lyapunov decrease outside an explicitly characterized residual set.
Our contributions are as follows:
\begin{itemize}

\item 
    \textbf{Distribution-level formulation and optimal-policy stability.}
    In \namedref{Section}{sec:problem}, we formulate distributional regulation as a DMDP, building on the dissimilarity-based stability notion of measure-functional MDPs~\citep{Gros_2022}.
    For the undiscounted infinite-horizon problem, we establish sufficient conditions linking the stage cost and optimal value to asymptotic stability of the target distribution.
    \namedref{Appendix}{appendix:recovery_nmdp} further shows that the formulation recovers successful norm-based distributional stabilization whenever the corresponding distribution dynamics are representable by an admissible DMDP policy.
\item 
    \textbf{Transfer of Lyapunov certificates to distributional stability.}
    In \namedref{Section}{sec:error_distribution}, we introduce a finite-dimensional distribution representation and a compatibility condition relating parameter distance to distributional dissimilarity.
    \namedref{Lemma}{lemma:lyapunov_distribution_stability} establishes a policy-independent certificate that transfers a Lyapunov decrease in the parameter space to asymptotic stability of the induced distributions.
    For the ideal undiscounted problem, dissipativity-based cost rotation connects this certificate to optimal control by making the rotated optimal value a Lyapunov function under the stated conditions.
\item 
    \textbf{Error-dependent certificates for learned policies.}
    In \namedref{Section}{sec:stablization_policy_learning}, we extend the analysis to learning and deployment.
    \namedref{Theorems}{theo:parameter_stability_means_distribution_stability}--\ref{theo:model_free_online_policy_learning} connect the ideal optimal-policy guarantee to certificates for policies learned offline with an approximate model, their deployment using estimated distribution parameters, and policies learned from transition data through an estimated Lyapunov decrease.
    Under the stated high-probability uniform error bounds, the results quantify how approximation and estimation errors enter the Lyapunov decrease and determine the residual sets.
    We also establish asymptotic recovery of the ideal optimizers within the chosen policy class under uniform convergence of the learned models and infinite-horizon objectives.
\item 
    \textbf{Computable illustration and empirical evaluation.}
    In \namedref{Appendix}{appendix:toy_example}, a computable example on a linear system illustrates the Lyapunov certificates and numerically evaluates their error terms and residual sets.
    In \namedref{Section}{sec:practical_implementation}, we implement distribution feedback and target-dependent rewards using sliding-window estimates with proximal policy optimization (PPO)~\citep{Schulman2017} and soft actor-critic (SAC)~\citep{Haarnoja2018}, with and without a safety-cost constraint.
    The HalfCheetah experiments in \namedref{Section}{sec:validations} provide empirical evidence that distribution feedback and explicit target matching improve regulation toward a desired velocity distribution.
    These experiments illustrate the practical approach; the Lyapunov certificates are examined in the scalar example and are not established for the learned HalfCheetah policies.

\end{itemize}

\section{Dissimilarity-Based MDPs}
\label{sec:problem}

This section introduces the Dissimilarity-based MDP and a notion of stability defined on the space of probability measures.
\namedref{Section}{sec:dmdp_formulation} defines the distribution dynamics, the dissimilarity measure, the resulting stability notion, and the associated optimal control problem; \namedref{Section}{sec:optimal_policy_stability} gives a sufficient condition under which an optimal policy renders the DMDP $\mathbb{D}$-asymptotically stable; and \namedref{Section}{sec:challenges} explains why \namedref{Problem}{eq:maintext_stability_problem} is not directly solvable in practice.
\subsection{State-Distribution Dynamics and Dissimilarity Measures}
\label{sec:dmdp_formulation}
Let $\bm{\mathrm{s}} \in \mathcal{S} \subseteq \mathbb{R}^{d_{\mathsf{s}}}$ and $\bm{\mathrm{a}} \in \mathcal{A} \subseteq \mathbb{R}^{d_{\mathsf{a}}}$ denote the system state and an admissible action, respectively.
We consider a continuous-state, continuous-action setting in which both $\mathcal{S}$ and $\mathcal{A}$ are continuous spaces.
Moreover, we assume that $\mathcal{A}$ is compact; this assumption is natural in many real-world applications in which admissible actions are restricted to a closed and bounded set.
The evolution of the controlled stochastic system is characterized by a Markov transition kernel $\tau$ on $\mathcal{S}$~\citep{HernandezLerma1996,Meyn2009}.
For any measurable set $\mathcal{B} \subseteq \mathcal{S}$, $\tau\left[\mathcal{B} \mid \bm{\mathrm{s}}_k, \bm{\mathrm{a}}_k\right]$ denotes the probability that the next state $\bm{\mathrm{s}}_{k+1}$ belongs to $\mathcal{B}$ given the current state-action pair $(\bm{\mathrm{s}}_k,\bm{\mathrm{a}}_k)$.
Given an initial distribution $\rho_0$ over $\mathcal{S}$ and a policy that selects actions from $\mathcal{A}$, the transition kernel $\tau$ induces a sequence of state distributions $\left\{\rho_k\right\}_{k=0}^{\infty}$ with $\bm{\mathrm{s}}_k \sim \rho_k$.
Let $M(\mathcal{S})$ denote the set of all probability measures defined on $\mathcal{S}$.
For a fixed action $\bm{\mathrm{a}} \in \mathcal{A}$, we define the distribution transition operator $\mathcal{T}_{\tau,\bm{\mathrm{a}}}: M(\mathcal{S}) \rightarrow M(\mathcal{S})$ as the mapping from a probability measure $\rho$ to its successor $\rho_{+}$ under the action $\bm{\mathrm{a}}$:
\begin{equation}
    \label{eq:def_state_transition_operator}
    \rho_{+}(\mathcal{B}) = \mathcal{T}_{\tau,\bm{\mathrm{a}}}\left[\rho\right](\mathcal{B}) = \int_{\mathcal{S}} \tau(\mathcal{B} \mid \bm{\mathrm{s}}, \bm{\mathrm{a}}) \, \rho(\mathsf{d}\bm{\mathrm{s}}).
\end{equation}

Conventional MDP formulations~\citep{Puterman1994,Sutton2018}, such as the norm-based MDP (NMDP) described in \namedref{Appendix}{appendix:norm_mdp}, aim to optimize the expected cumulative objective associated with state-action pairs through a policy that selects actions based on the instantaneous state.
In such formulations, the instantaneous state $\bm{\mathrm{s}}_k$ is used as the decision variable that characterizes the system at time index $k$.
However, in a stochastic system, even under the same initial distribution and input history, the realized state $\bm{\mathrm{s}}_k$ remains random, whereas the stochastic dynamics uniquely determine its distribution $\rho_k$.
Therefore, $\bm{\mathrm{s}}_k$ should be regarded as a sample drawn from $\rho_k$, rather than a complete characterization of the system's distributional behavior.
Motivated by this observation, we target the state distribution $\rho_k$ rather than an individual realization $\bm{\mathrm{s}}_k$.
For deterministic state values, a norm such as $\|\bm{\mathrm{s}}_k-\bm{\mathrm{s}}_\star\|$ is commonly used to quantify the distance from a target state.
For probability distributions, we instead introduce a dissimilarity measure $\mathbb{D}(\cdot\|\cdot)$ to quantify the discrepancy between two state distributions as follows:
\begin{mydef}[\textbf{Dissimilarity}]
    \label{def:dissimilarity_measure}
    A dissimilarity measure $\mathbb{D}(\cdot\|\cdot):M(\mathcal{S})\times M(\mathcal{S})\rightarrow\mathbb{R}$ is a mapping from a pair of probability measures to a nonnegative real number such that
    \begin{equation}
        \mathbb{D}(\rho\|\tilde{\rho})\geq 0\ \text{and}\ \mathbb{D}(\rho\|\rho)=0,\ \forall \rho,\tilde{\rho}\in M(\mathcal{S}).
    \end{equation}
\end{mydef}
\noindent
Typical choices of $\mathbb{D}(\cdot\|\cdot)$ include the Kullback--Leibler (KL) divergence~\citep{Kullback1951} and the Wasserstein distance~\citep{Villani2009}, restricted to pairs of measures for which they are finite.
From this perspective, sequential decision-making for stochastic systems can be formulated as the regulation of the dissimilarity between state distributions over time.
This formulation leads to a Dissimilarity-based Markov Decision Process (DMDP), defined by the tuple
$\mathcal{M}_{\mathsf{D}} := \langle M(\mathcal{S}), \mathcal{A}, \tau, \mathcal{L}, \gamma, \rho_0 \rangle$,
where $\mathcal{L}: M(\mathcal{S}) \times \mathcal{A} \rightarrow \mathbb{R}$ is a stage cost functional defined over the space of probability measures and the action space, and $\gamma$ is the discount factor.
A representative example of such a stage cost functional is
$\mathcal{L}[\rho_k, \bm{\mathrm{a}}_k] = \mathbb{D}(\rho_k \| \rho_\star)$,
where $\rho_\star$ denotes a target distribution.
We consider a deterministic causal policy $\psi: M(\mathcal{S}) \rightarrow \mathcal{A}$ such that $\bm{\mathrm{a}}_k = \psi(\rho_k)$.
Let $\bm{\Psi}$ denote the set of all such policies.
For a given policy $\psi$, a fixed transition kernel $\tau$, and a fixed initial distribution $\rho_0$, the sequence of probability measures $\{\rho_k\}_{k=0}^\infty$ associated with the state sequence $\{\bm{\mathrm{s}}_k\}_{k=0}^\infty$, where $\bm{\mathrm{s}}_k \sim \rho_k$, is uniquely determined through the distribution transition operator in~\eqref{eq:def_state_transition_operator}.
With the dissimilarity measure $\mathbb{D}(\cdot\|\cdot)$ in \namedref{Definition}{def:dissimilarity_measure}, the stability of the DMDP can be formalized directly on the space of probability measures as follows:
\begin{mydef}[\textbf{Dissimilarity-based stability}]
    \label{def:D_stability}
    A Dissimilarity-based Markov Decision Process is $\mathbb{D}$-stable with respect to a probability measure $\rho_\star$ and a dissimilarity measure $\mathbb{D}$ if, for any $\varepsilon>0$, there exists a $\delta(\varepsilon)>0$ such that $\mathbb{D}(\rho_0\|\rho_\star)\leq \delta$ implies $\mathbb{D}(\rho_k\|\rho_\star)\leq \varepsilon$ for all $k\geq 0$.
    If, moreover, the probability measure $\rho_\star$ is $\mathbb{D}$-attractive, i.e.,
    \begin{equation}
        \label{eq:def_D_attractive}
        \lim_{k\rightarrow\infty} \mathbb{D}(\rho_k\|\rho_\star)=0
    \end{equation}
    holds, then the Dissimilarity-based Markov Decision Process is called $\mathbb{D}$-asymptotically stable.
\end{mydef}
\noindent
Our objective is to design a distribution-dependent policy that stabilizes the stochastic system to a desired target distribution $\rho_\star$ in the sense of $\mathbb{D}$-asymptotic stability.
This objective leads to the following problem:
\begin{equation}
    \label{eq:maintext_stability_problem}
    \tag{$\mathsf{OCP}$}
\min_{\psi \in \bm{\Psi}}\ \sum_{k=0}^{\infty} \gamma^k\mathcal{L}[\rho_k, \psi(\rho_k)] \quad
    \text{s.t.}\ \rho_{k+1} = \mathcal{T}_{\tau,\bm{\mathrm{a}}_k} \rho_k,\ \bm{\mathrm{a}}_k= \psi(\rho_k).
\end{equation}

\begin{figure}[t]
    \centering
    \definecolor{figState}{HTML}{0072B2}
    \definecolor{figDist}{HTML}{D55E00}
    \definecolor{figGreen}{HTML}{009E73}
    \definecolor{figGray}{HTML}{595959}
    \begin{tikzpicture}[
        font=\footnotesize,
        >={Stealth[length=4pt,width=3.4pt]},
        box/.style={draw=figGray, rounded corners=2pt, line width=0.5pt, fill=white, align=center, inner sep=3pt, minimum height=1.35cm},
        flow/.style={->, line width=0.6pt, draw=black!80},
        panel/.style={font=\footnotesize, anchor=west, inner sep=0pt},
        sub/.style={font=\scriptsize, text=figGray},
    ]
        \begin{scope}
            \node[panel] at (0,1.25) {\textbf{(a)} norm-based MDP: the policy acts on one realized state $\bm{\mathrm{s}}_k$};
            \node[box, minimum width=3.7cm, minimum height=1.55cm] (pa) at (1.85,0) {};
            \draw[line width=0.4pt, black!70] (0.2,-0.42) -- (3.5,-0.42);
            \fill[black!12] plot[domain=0.25:3.45, samples=60, smooth] (\x, {-0.42+0.78*exp(-((\x-2.05)^2)/(2*0.5^2))}) -- (3.45,-0.42) -- (0.25,-0.42) -- cycle;
            \draw[black!45, line width=0.6pt] plot[domain=0.25:3.45, samples=60, smooth] (\x, {-0.42+0.78*exp(-((\x-2.05)^2)/(2*0.5^2))});
            \node[text=black!55, anchor=west] at (2.62,0.22) {$\rho_k$};
            \draw[black!70, line width=0.5pt] (0.95,-0.47) -- (0.95,-0.37);
            \node[anchor=north, inner sep=1pt] at (0.95,-0.47) {$\bm{\mathrm{s}}_\star$};
            \fill[figState] (2.45,-0.42) circle (2.1pt);
            \node[anchor=north, inner sep=1pt, text=figState] at (2.45,-0.47) {$\bm{\mathrm{s}}_k$};
            \node[box, minimum width=2.3cm] (pol_a) at (5.55,0) {policy\\[2pt] $\bm{\mathrm{a}}_k=\pi(\bm{\mathrm{s}}_k)$};
            \node[box, minimum width=3.6cm] (tr_a) at (9.45,0) {Markov transition kernel\\[2pt] $\bm{\mathrm{s}}_{k+1}\sim\tau(\cdot\mid\bm{\mathrm{s}}_k,\bm{\mathrm{a}}_k)$};
            \node[box, minimum width=4.1cm, fill=figState!7, draw=figState] (c_a) at (14.05,0) {expected stage cost\\[2pt] $\mathbb{E}_{\bm{\mathrm{s}}_k\sim\rho_k}\big[L(\bm{\mathrm{s}}_k,\bm{\mathrm{a}}_k)\big]$\\[2pt] {\scriptsize e.g., $L(\bm{\mathrm{s}},\bm{\mathrm{a}})=\|\bm{\mathrm{s}}-\bm{\mathrm{s}}_\star\|_2^2$}};
            \draw[flow, draw=figState] (pa.east) -- node[above, text=figState] {$\bm{\mathrm{s}}_k$} (pol_a.west);
            \draw[flow] (pol_a.east) -- node[above] {$\bm{\mathrm{a}}_k$} (tr_a.west);
            \draw[flow, draw=figState] (tr_a.south) -- ++(0,-0.3) -| node[pos=0.25, below, text=figState] {$\bm{\mathrm{s}}_{k+1}$} (pa.south);
        \end{scope}
        \begin{scope}[yshift=-3.25cm]
            \node[panel] at (0,1.25) {\textbf{(b)} DMDP: the policy acts on the state distribution $\rho_k$};
            \node[box, minimum width=3.7cm, minimum height=1.55cm] (pb) at (1.85,0) {};
            \draw[line width=0.4pt, black!70] (0.2,-0.42) -- (3.5,-0.42);
            \fill[figDist!15] plot[domain=0.75:3.45, samples=70, smooth] (\x, {-0.42+0.58*exp(-((\x-2.25)^2)/(2*0.36^2))+0.22*exp(-((\x-1.55)^2)/(2*0.3^2))}) -- (3.45,-0.42) -- (0.75,-0.42) -- cycle;
            \draw[figDist, line width=0.7pt] plot[domain=0.75:3.45, samples=70, smooth] (\x, {-0.42+0.58*exp(-((\x-2.25)^2)/(2*0.36^2))+0.22*exp(-((\x-1.55)^2)/(2*0.3^2))});
            \draw[black!80, line width=0.7pt, dash pattern=on 2.2pt off 1.4pt] plot[domain=0.2:2.0, samples=60, smooth] (\x, {-0.42+0.78*exp(-((\x-1.0)^2)/(2*0.3^2))});
            \node[text=figDist, anchor=west] at (2.62,0.22) {$\rho_k$};
            \node[anchor=east] at (0.72,0.3) {$\rho_\star$};
            \node[box, minimum width=2.3cm] (pol_b) at (5.55,0) {policy\\[2pt] $\bm{\mathrm{a}}_k=\psi(\rho_k)$};
            \node[box, minimum width=3.6cm] (tr_b) at (9.45,0) {distribution transition\\[2pt] $\rho_{k+1}=\mathcal{T}_{\tau,\bm{\mathrm{a}}_k}\rho_k$};
            \node[box, minimum width=4.1cm, fill=figDist!7, draw=figDist] (c_b) at (14.05,0) {stage cost\\[2pt] $\mathcal{L}[\rho_k,\bm{\mathrm{a}}_k]$\\[2pt] {\scriptsize e.g., $\mathcal{L}[\rho_k,\bm{\mathrm{a}}_k]=\mathbb{D}(\rho_k\|\rho_\star)$}};
            \draw[flow, draw=figDist] (pb.east) -- node[above, text=figDist] {$\rho_k$} (pol_b.west);
            \draw[flow] (pol_b.east) -- node[above] {$\bm{\mathrm{a}}_k$} (tr_b.west);
            \draw[flow, draw=figDist] (tr_b.south) -- ++(0,-0.3) -| node[pos=0.25, below, text=figDist] {$\rho_{k+1}$} (pb.south);
        \end{scope}
    \end{tikzpicture}
    \caption{Norm-based MDP (a) and DMDP (b).}
    \label{fig:dmdp_overview}
\end{figure}

\namedref{Figure}{fig:dmdp_overview} contrasts the two formulations.
In the state-feedback loop of the norm-based MDP of \namedref{Appendix}{appendix:norm_mdp}, the policy $\pi$ acts on the realized state $\bm{\mathrm{s}}_k$, which is one sample from the state distribution $\rho_k$ (gray), and the objective accumulates the expected stage cost.
In the distribution-feedback loop of the DMDP, the policy $\psi$ acts on the state distribution $\rho_k$ itself, the distribution evolves through the distribution transition operator $\mathcal{T}_{\tau,\bm{\mathrm{a}}_k}$ in~\eqref{eq:def_state_transition_operator}, and a representative stage cost is the dissimilarity $\mathbb{D}(\rho_k\|\rho_\star)$ between $\rho_k$ and a target distribution $\rho_\star$ (dashed).
The densities in the figure are schematic.

\subsection{Optimal-Policy Stability}
\label{sec:optimal_policy_stability}
For a policy $\psi\in\bm{\Psi}$ and an initial distribution $\rho_0\in M(\mathcal{S})$, let $J[\psi;\rho_0]:=\sum_{k=0}^{\infty}\gamma^k\mathcal{L}[\rho_k,\psi(\rho_k)]$ denote the objective of \namedref{Problem}{eq:maintext_stability_problem} along the closed-loop sequence $\rho_{k+1}=\mathcal{T}_{\tau,\psi(\rho_k)}\rho_k$, whenever the series is defined (possibly as $+\infty$).
Its optimal value is $V_\star[\rho_0]:=\inf_{\psi\in\bm{\Psi}}J[\psi;\rho_0]$, and a policy $\psi_\star\in\bm{\Psi}$ is optimal for $\rho_0$ if $J[\psi_\star;\rho_0]=V_\star[\rho_0]$.
The following theorem provides a sufficient condition under which $\psi_\star$ stabilizes the distribution dynamics of $\rho_k$.
\begin{restatable}{mytheo}{restateThmMainStability}
    \label{theo:main_psi_star_stability}
    Assume $\gamma=1$.
    Let $\mathcal{P}:=\{\rho\in M(\mathcal{S})\mid V_\star[\rho]<\infty\}$.
    Suppose that $\psi_\star$ attains $V_\star[\rho]$ for every initial distribution $\rho\in\mathcal{P}$, that $\mathcal{P}$ is forward invariant under the closed-loop dynamics induced by $\psi_\star$, that $\rho_\star=\mathcal{T}_{\tau,\psi_\star(\rho_\star)}\rho_\star$, and that there exist class-$\mathcal{K}$ functions~\citep{Khalil2002,Kellett2014} $\kappa_1$ and $\kappa_2$ such that
    \begin{equation}
        \mathcal{L}\left[\rho_k,\psi_\star(\rho_k)\right]\geq \kappa_1\left(\mathbb{D}(\rho_k\|\rho_\star)\right),\
        V_\star\left[\rho_k\right]\leq\kappa_2\left(\mathbb{D}(\rho_k\|\rho_\star)\right),\ \forall \rho_k\in\mathcal{P}\ \text{generated}\ \text{by}\ \psi_\star.
    \end{equation}
    Then, the Dissimilarity-based Markov Decision Process under the optimal policy $\psi_\star$ is $\mathbb{D}$-asymptotically stable with respect to the target probability measure $\rho_\star$ and the dissimilarity measure $\mathbb{D}$ for any initial distribution $\rho_0\in\mathcal{P}$.
\end{restatable}
\noindent
The proof of \namedref{Theorem}{theo:main_psi_star_stability} is provided in \namedref{Appendix}{app:proof_main_psi_star_stability}.
The forward-invariance condition on $\mathcal{P}$ is a mild well-posedness requirement rather than a restrictive structural assumption.
For example, this condition holds when $\mathcal{P}$ is chosen as a continuous admissible region, such as a compact subset of $M(\mathcal{S})$, on which the optimal closed-loop transition $\rho\mapsto\mathcal{T}_{\tau,\psi_\star(\rho)}\rho$ is well defined and positively invariant, i.e., $\mathcal{T}_{\tau,\psi_\star(\rho)}\rho\in\mathcal{P}$ for all $\rho\in\mathcal{P}$.
Such compact invariant regions commonly arise from bounded admissible distribution families or from bounded parameter domains in the parameterized formulation introduced in \namedref{Section}{sec:parameterized_distributions}.
\subsection{Challenges in Solving DMDPs}
\label{sec:challenges}
Despite its conceptual appeal, \namedref{Problem}{eq:maintext_stability_problem} is not directly solvable in practice.
The value function $V_\star[\rho_0]$ is defined over the infinite-dimensional space of probability measures $M(\mathcal{S})$, which makes direct computation intractable.
In addition, the formulation requires access to the evolving state distributions $\{\rho_k\}_{k=0}^\infty$ and the distribution transition operator $\mathcal{T}_{\tau,\bm{\mathrm{a}}_k}$, which are generally unknown in practical applications.
Instead, only finite trajectory data are available, and both the state distributions and their distributional dynamics must be inferred from data.
\namedref{Sections}{sec:error_distribution} and~\ref{sec:stablization_policy_learning} address these two challenges in turn.

\begin{figure}[t]
    \centering
    \definecolor{figState}{HTML}{0072B2}
    \definecolor{figDist}{HTML}{D55E00}
    \definecolor{figGreen}{HTML}{009E73}
    \definecolor{figGray}{HTML}{595959}
    \begin{tikzpicture}[
        font=\footnotesize,
        >={Stealth[length=4pt,width=3.4pt]},
        base/.style={rounded corners=2pt, align=center, inner sep=3pt, minimum height=0.95cm, minimum width=4.8cm},
        problem/.style={base, line width=0.5pt, draw=figGray, fill=white},
        result/.style={base, line width=0.9pt, draw=figDist, fill=figDist!9},
        smallresult/.style={result, minimum width=3.0cm, minimum height=1.6cm},
        dep/.style={->, line width=0.6pt, draw=black!75},
        frame/.style={draw=figGray, line width=0.4pt, rounded corners=3pt, dash pattern=on 1.5pt off 1.5pt, inner sep=4pt},
        framelabel/.style={font=\scriptsize, text=figGray, fill=black!4, inner sep=1.5pt},
        band/.style={fill=black!4, draw=none},
        bandlabel/.style={font=\scriptsize, text=figGray, rotate=90, anchor=south, inner sep=1pt},
    ]
        \node[problem] (ocp) at (3.1,0) {\textbf{Problem~$\mathsf{OCP}$}\\ DMDP on $M(\mathcal{S})$ with $\bm{\mathrm{a}}_k=\psi(\rho_k)$};
        \node[result] (t1) at (8.5,0) {\textbf{Theorem~\ref*{theo:main_psi_star_stability}}\\ optimal $\psi_\star$ is $\mathbb{D}$-asymptotically stable};
        \node[problem] (pocp) at (3.1,-1.65) {\textbf{Assumption~\ref*{assump:parametric_rho}, Problem~$\mathsf{P\text{-}OCP}$}\\ $\rho_\omega$, \ $\omega_{k+1}=f(\omega_k,\bm{\mathrm{a}}_k)$, \ $\bm{\mathrm{a}}_k=\phi(\omega_k)$};
        \node[problem] (a2) at (13.9,-1.65) {\textbf{Assumption~\ref*{assump:D_compatible_parameterization}}\\ $\mathbb{D}$-compatible parameterization};
        \node[problem] (rot) at (3.1,-3.0) {\textbf{Cost rotation} (Section~\ref*{sec:cost_rotation})\\ $\ell^{\mathsf{R}}$, \ $U_\star^{\mathsf{R}}$, \ rotated Bellman identity};
        \node[result] (t2) at (8.5,-3.0) {\textbf{Theorem~\ref*{theo:parameter_stability_means_distribution_stability}}\\ optimal $\phi_\star$ is $\mathbb{D}$-asymptotically stable};
        \node[result] (l1) at (13.9,-3.0) {\textbf{Lemma~\ref*{lemma:lyapunov_distribution_stability}}\\ Lyapunov certificate transfer\\ for any admissible policy $\phi$};
        \node[smallresult] (t3) at (2.3,-5.0) {\textbf{Theorem~\ref*{theo:stability_finite_samples}}\\ finite-sample certificate\\ for $\hat{\phi}_{\mathsf{off}}$};
        \node[smallresult] (t4) at (5.85,-5.0) {\textbf{Theorem~\ref*{theo:online_plugin_stability}}\\ online deployment\\ with estimates $\hat{\omega}_t$};
        \node[smallresult] (p1) at (9.4,-5.0) {\textbf{Proposition~\ref*{prop:asymptotic_recovery_rotated_optimizers}}\\ asymptotic recovery\\ as $N\rightarrow\infty$};
        \node[result, minimum height=1.6cm] (t5) at (13.9,-5.0) {\textbf{Theorem~\ref*{theo:model_free_online_policy_learning}}\\ certificate from an estimated\\ Lyapunov decrease for $\phi_n$\\ with the surrogate $W_n$};
        \coordinate (mbpad) at ([yshift=-0.17cm]t4.south);
        \coordinate (mfpad) at ([yshift=-0.17cm]t5.south);
        \node[frame, fit=(t3)(t4)(p1)(mbpad)] (mb) {};
        \node[frame, fit=(t5)(mfpad)] (mf) {};
        \node[framelabel] at (mb.south -| t4) {approximate dynamics model (Sections~\ref*{sec:approximate_dynamics_model},~\ref*{sec:deployment_estimated_distribution}, and~\ref*{sec:asymptotic_recovery})};
        \node[framelabel] at (mf.south) {estimated decrease (Section~\ref*{sec:certification_estimated_lyapunov})};
        \begin{scope}[on background layer]
            \fill[band] (0.55,0.62) rectangle (16.45,-0.62);
            \fill[band] (0.55,-1.05) rectangle (16.45,-3.72);
            \fill[band] (0.55,-4.0) rectangle (16.45,-6.4);
        \end{scope}
        \node[bandlabel] at (0.5,0) {Section~\ref*{sec:problem}};
        \node[bandlabel] at (0.5,-2.385) {Section~\ref*{sec:error_distribution}};
        \node[bandlabel] at (0.5,-5.16) {Section~\ref*{sec:stablization_policy_learning}};
        \draw[dep] (ocp) -- (t1);
        \draw[dep] (ocp) -- (pocp);
        \draw[dep] (pocp) -- (rot);
        \draw[dep] (pocp) -- (a2);
        \draw[dep] (a2) -- (l1);
        \draw[dep] (rot) -- (t2);
        \draw[dep] (l1) -- (t2);
        \draw[dep] (rot.south) -- (rot.south |- mb.north);
        \draw[dep] (l1.south) -- (l1.south |- mf.north);
        \coordinate (junction) at ([yshift=-0.3cm]l1.south);
        \draw[dep] (junction) -| ([xshift=-0.55cm]mb.north east);
        \fill[black!75] (junction) circle (1.3pt);
        \draw[dep] (t3) -- (t4);
    \end{tikzpicture}
    \caption{Roadmap of the theoretical development.}
    \label{fig:roadmap}
\end{figure}
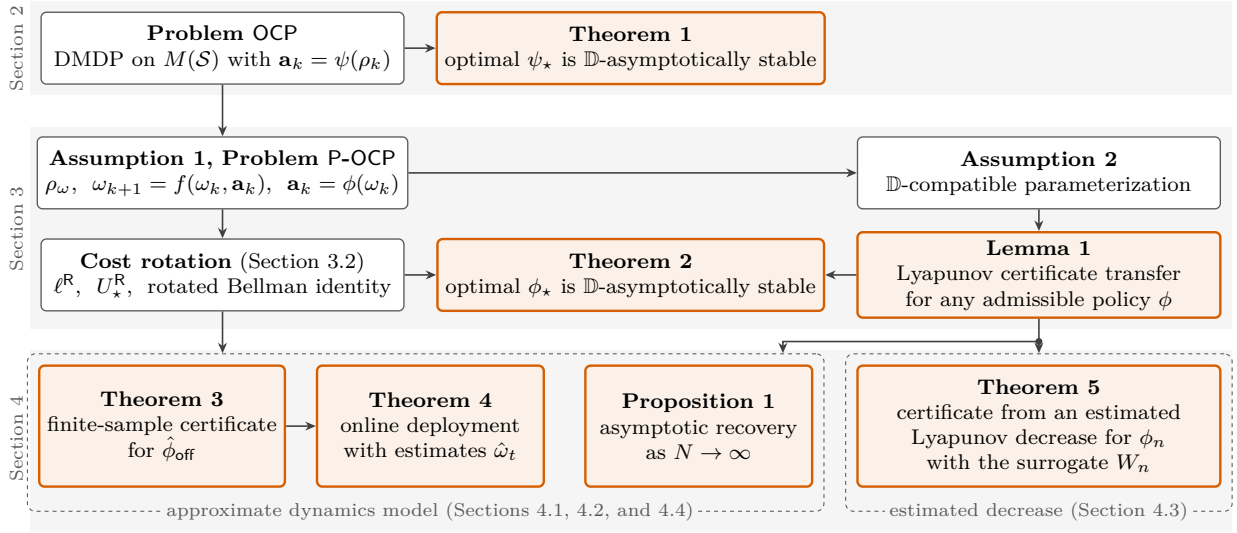

\section{Finite-Dimensional Representation and Stability Analysis}
\label{sec:error_distribution}
Although the DMDP in \namedref{Problem}{eq:maintext_stability_problem} provides a distribution-level stability notion for stochastic systems, it remains infinite-dimensional because the state $\rho_k$ belongs to the space of probability measures $M(\mathcal{S})$.
We therefore introduce a finite-dimensional parameterization of the state distributions.
\namedref{Section}{sec:parameterized_distributions} introduces this parameterization and the resulting finite-dimensional optimal control problem, \namedref{Section}{sec:cost_rotation} introduces a cost rotation for the undiscounted problem, and \namedref{Section}{sec:lyapunov_certificates} establishes Lyapunov certificates that transfer stability from the parameter space to the space of probability measures.
\namedref{Figure}{fig:roadmap} summarizes the theoretical development of \namedref{Sections}{sec:problem}--\ref{sec:stablization_policy_learning} and the dependencies among the assumptions and results.
In the figure, white boxes denote problems, assumptions, and constructions, and orange boxes with thick borders denote results; each result holds under the hypotheses stated in it.
An arrow from one box to another indicates that the latter is formulated in terms of the former or that its proof uses the former, and an arrow into a dashed frame applies to every result in that frame.
\namedref{Lemma}{lemma:lyapunov_distribution_stability} transfers a Lyapunov decrease in the parameter space to $\mathbb{D}$-asymptotic stability of the induced distributions; it, or the Lyapunov argument in its proof, is used by every result in \namedref{Section}{sec:stablization_policy_learning}, with $W=U_\star^{\mathsf{R}}$ in \namedref{Theorems}{theo:stability_finite_samples} and~\ref{theo:online_plugin_stability} and with $W=W_n$ in \namedref{Theorem}{theo:model_free_online_policy_learning}.

\subsection{Parameterized Distributions}
\label{sec:parameterized_distributions}
The parameterization rests on the following assumption.
\begin{myassump}
    \label{assump:parametric_rho}
    Each probability measure $\rho\in M(\mathcal{S})$ considered in the analysis admits a density representation parameterized by a finite-dimensional vector $\omega\in\Omega\subseteq\mathbb{R}^{n_\omega}$.
    We write the corresponding measure as $\rho_\omega$.
\end{myassump}

\namedref{Assumption}{assump:parametric_rho} allows the distributional state $\rho_k$ to be represented by a parameter vector $\omega_k$ estimated from finite data.
This parameterization may be obtained, for example, by kernel density estimation~\citep{Parzen1962,Silverman1986}, Gaussian mixture models~\citep{McLachlan2000}, normalizing flows~\citep{Rezende2015,Papamakarios2021}, or neural density estimators~\citep{Papamakarios2017}.
Thus, the family need not be a simple parametric one, and the assumption serves as a computational interface between distribution estimation and policy learning.
Under this parameterization, the distributional transition
$\rho_{k+1}=\mathcal{T}_{\tau,\bm{\mathrm{a}}_k}\rho_k$
is represented by the parameter dynamics
\begin{equation}
    \label{eq:parameter_dynamics}
    \omega_{k+1}=f(\omega_k,\bm{\mathrm{a}}_k),
\end{equation}
where $f:\Omega\times\mathcal{A}\rightarrow\Omega$ is an unknown transition map induced by the underlying stochastic dynamics and the chosen density parameterization.
Although the sample-level system dynamics are stochastic, the evolution of the distribution parameter $\omega_k$ is deterministic once the current distribution parameter and the action are fixed.
The representation~\eqref{eq:parameter_dynamics} is exact when the parameterized family is closed under the controlled transition, i.e., $\mathcal{T}_{\tau,\bm{\mathrm{a}}}\rho_\omega=\rho_{f(\omega,\bm{\mathrm{a}})}$ for all $\omega\in\Omega$ and $\bm{\mathrm{a}}\in\mathcal{A}$; otherwise, $f$ describes the evolution of the fitted parameter only approximately.
Closure holds in special cases, such as linear dynamics with additive Gaussian noise and a Gaussian family in \namedref{Appendix}{appendix:toy_example}, but rarely in general, e.g., not for the product of uniform families under the HalfCheetah dynamics in \namedref{Appendix}{appendix:implementation}.
\namedref{Assumption}{assump:parametric_rho} and the exact representation are thus idealizations that provide a tractable finite-dimensional framework and hold only approximately in practice; \namedref{Section}{sec:stablization_policy_learning} retains them and bounds the errors of the learned parameter dynamics, value functions, and estimated Lyapunov decrease, as well as of the estimated distribution parameters.
We then introduce a parameter-based policy $\phi:\Omega\rightarrow\mathcal{A}$ such that $\bm{\mathrm{a}}_k=\phi(\omega_k)$, and let $\Phi$ denote the set of admissible policies.
The distributional stage cost functional $\mathcal{L}:M(\mathcal{S})\times\mathcal{A}\rightarrow\mathbb{R}$ is equivalently represented in the parameter space by $\ell(\omega,\bm{\mathrm{a}})
    :=
    \mathcal{L}[\rho_\omega,\bm{\mathrm{a}}].$
When $\mathcal{L}[\rho_k,\bm{\mathrm{a}}_k]=\mathbb{D}(\rho_k\|\rho_\star)$, we have
$\ell(\omega_k,\bm{\mathrm{a}}_k)=\mathbb{D}(\rho_{\omega_k}\|\rho_\star)$.

This parameterization yields the following finite-dimensional optimal control problem:
\begin{equation}
\label{eq:parameter_based}
\tag{$\mathsf{P\text{-}OCP}$}
    \min_{\phi \in \Phi} \ \sum_{k=0}^{\infty} \gamma^k\left[ \ell(\omega_k,\phi(\omega_k))-\ell_0 \right] \quad
    \text{s.t.} \quad \omega_{k+1}=f(\omega_k,\phi(\omega_k)),\quad \omega_0=\omega,
\end{equation}
where $U_\star(\omega)$ denotes the optimal value of \namedref{Problem}{eq:parameter_based}.
The constant $\ell_0$ is defined by
\begin{equation}
\label{eq:parameter_based_optimal_steady}
    \ell_0
    :=
    \min_{\phi \in \Phi,\ \omega \in \Omega} \ \ell(\omega,\phi(\omega))  \quad
    \text{s.t.} \quad \omega=f(\omega,\phi(\omega)).
\end{equation}
The optimal steady-state parameter is denoted by $\omega_\star$, and we assume that it exists and that its associated probability measure is $\rho_{\omega_\star}=\rho_\star$.
Because the state distribution itself is the controlled object, the steady-state problem in~\eqref{eq:parameter_based_optimal_steady} is posed in the parameterized space.
This problem identifies a stationary parameter $\omega_\star$ satisfying $\omega_\star=f(\omega_\star,\phi(\omega_\star))$, whose associated probability measure $\rho_{\omega_\star}$ represents the target distribution $\rho_\star$.
Thus, $\omega_\star$ is not an additional target beyond $\rho_\star$; it is a finite-dimensional representation of the invariant target distribution and is used to normalize the cost through $\ell_0$ and to formulate stabilization in the parameter space.
Solving \namedref{Problem}{eq:parameter_based} yields an optimal parameter-based policy $\phi_\star\in\Phi$.
\subsection{Cost Rotation}
\label{sec:cost_rotation}
Since the parameter dynamics in~\eqref{eq:parameter_dynamics} are deterministic, the stability of the parameter trajectory can be analyzed through a dissipativity-based cost rotation~\citep{Diehl2011,Angeli2012,Faulwasser2018}.
For the undiscounted case $\gamma=1$, let $\lambda:\Omega\rightarrow\mathbb{R}$ be a storage function~\citep{Willems1972} and define the rotated stage cost by
\begin{equation}
    \label{eq:rotated_stage_cost_parameter}
    \ell^{\mathsf{R}}(\omega,\bm{\mathrm{a}})
    :=
    \ell(\omega,\bm{\mathrm{a}})
    -
    \ell_0
    -
    \lambda(f(\omega,\bm{\mathrm{a}}))
    +
    \lambda(\omega).
\end{equation}
When the storage function $\lambda$ is bounded along admissible trajectories in $\Omega$ and the terminal storage term $\lambda(\omega_H)$ is kept in the rotated objective, as in~\eqref{eq:rotated_telescope_parameter} below, the rotated and unrotated undiscounted objectives have the same optimal policies.
For any policy $\phi$ and any finite horizon $H$, the rotated cost satisfies
\begin{equation}
    \label{eq:rotated_telescope_parameter}
    \lambda(\omega_H)
    +
    \sum_{k=0}^{H-1}
    \ell^{\mathsf{R}}\left(\omega_k,\phi(\omega_k)\right)
    =
    \lambda(\omega_0)
    +
    \sum_{k=0}^{H-1}
    \left[
    \ell\left(\omega_k,\phi(\omega_k)\right)-\ell_0
    \right]\quad\text{(telescoping identity)},
\end{equation}
where $\omega_{k+1}=f(\omega_k,\phi(\omega_k))$.
Therefore, in the undiscounted setting, when the storage function is bounded along admissible trajectories and the infinite-horizon limits are well defined, the rotation preserves the set of optimal policies because it changes the objective only by the initial storage term.
We define the rotated optimal value by $U_\star^{\mathsf{R}}(\omega)
    :=
    \min_{\phi\in\Phi}
    \lim_{H\rightarrow\infty}
    \left[
    \lambda(\omega_H)
    +
    \sum_{k=0}^{H-1}
    \ell^{\mathsf{R}}\left(\omega_k,\phi(\omega_k)\right)
    \right],$
subject to $\omega_{k+1}=f(\omega_k,\phi(\omega_k))$ and $\omega_0=\omega$.
Under the standard dynamic programming principle~\citep{Puterman1994,HernandezLerma1996} for the rotated infinite-horizon problem, if an optimal stationary selector $\phi_\star$ exists and $U_\star^{\mathsf{R}}$ is finite on an operating set $\Omega_{\mathsf{c}}$, then the rotated optimal value satisfies the Bellman identity
\begin{equation}
    \label{eq:rotated_bellman_identity_parameter}
    U_\star^{\mathsf{R}}(\omega)
    =
    \ell^{\mathsf{R}}\left(\omega,\phi_\star(\omega)\right)
    +
    U_\star^{\mathsf{R}}\left(f(\omega,\phi_\star(\omega))\right),
    \quad
    \forall \omega\in\Omega_{\mathsf{c}}.
\end{equation}
This identity is the rotated analogue of the usual Bellman equation.
\subsection{Lyapunov Certificates for Distributional Stability}
\label{sec:lyapunov_certificates}
The Lyapunov certificates in this subsection rely on the following compatibility condition between the parameterization and the dissimilarity measure.
\begin{myassump}[$\mathbb{D}$-compatible parameterization]
    \label{assump:D_compatible_parameterization}
    Let $\Omega_{\mathsf{c}}\subseteq\Omega$ be an operating set containing $\omega_\star$.
    The parameterization $\omega\mapsto\rho_\omega$ is said to be $\mathbb{D}$-compatible on $\Omega_{\mathsf{c}}$ with respect to $\omega_\star$ if there exist class-$\mathcal{K}$ functions $\alpha_1$ and $\alpha_2$ such that
    \begin{equation}
        \label{eq:D_compatible_parameterization}
        \alpha_1\left(\|\omega-\omega_\star\|\right)
        \leq
        \mathbb{D}(\rho_\omega\|\rho_{\omega_\star})
        \leq
        \alpha_2\left(\|\omega-\omega_\star\|\right),
        \quad
        \forall \omega\in\Omega_{\mathsf{c}}.
    \end{equation}
\end{myassump}

\namedref{Assumption}{assump:D_compatible_parameterization} requires the parameter distance $\|\omega-\omega_\star\|$ and the distributional dissimilarity $\mathbb{D}(\rho_\omega\|\rho_{\omega_\star})$ to be locally comparable on the operating set under consideration.
This condition is not a restriction on the dynamics; rather, it ensures that convergence in the finite-dimensional parameter space implies convergence of the corresponding state distribution in the chosen dissimilarity measure.
We first state a policy-agnostic Lyapunov certificate~\citep[cf.][]{Khalil2002,Jiang2002}.
This result is independent of optimality and is used later to certify stability not only for the ideal optimal policy $\phi_\star$ but also for learned policies.

\begin{restatable}{mylemma}{restateLemLyapunovTransfer}
    \label{lemma:lyapunov_distribution_stability}
    Let $\phi:\Omega\rightarrow\mathcal{A}$ be any admissible policy and define the closed-loop parameter map $F_\phi(\omega):=f(\omega,\phi(\omega)).$
    Let $\Omega_{\mathsf{c}}\subseteq\Omega$ be positively invariant under $F_\phi$, and suppose that $F_\phi(\omega_\star)=\omega_\star$.
    Assume that there exist a continuous function $W:\Omega_{\mathsf{c}}\rightarrow\mathbb{R}_{\geq 0}$ and class-$\mathcal{K}$ functions $\underline{\alpha}$, $\overline{\alpha}$, and $\sigma$ such that
    \begin{equation}
        \label{eq:generic_lyapunov_bounds_parameter}
        \underline{\alpha}\left(\|\omega-\omega_\star\|\right)
        \leq
        W(\omega)
        \leq
        \overline{\alpha}\left(\|\omega-\omega_\star\|\right),
        \quad
        \forall \omega\in\Omega_{\mathsf{c}},
    \end{equation}
    \begin{equation}
        \label{eq:generic_lyapunov_decrease_parameter}
        W(F_\phi(\omega))-W(\omega)
        \leq
        -\sigma\left(\|\omega-\omega_\star\|\right),
        \quad
        \forall \omega\in\Omega_{\mathsf{c}}.
    \end{equation}
    Suppose further that the parameterization is $\mathbb{D}$-compatible on $\Omega_{\mathsf{c}}$ in the sense of \namedref{Assumption}{assump:D_compatible_parameterization}.
    Then, the induced distribution dynamics generated by
    $\omega_{k+1}=F_\phi(\omega_k),\ \rho_k=\rho_{\omega_k}$
    are $\mathbb{D}$-asymptotically stable with respect to $\rho_\star=\rho_{\omega_\star}$ on $\{\rho_\omega\mid \omega\in\Omega_{\mathsf{c}}\}$.
\end{restatable}

The next theorem recovers the optimal-policy stability guarantee from the rotated-cost dissipativity condition.
It shows that, for the ideal parameterized optimal control problem, the rotated optimal value $U_\star^{\mathsf{R}}$ can serve as the Lyapunov function in \namedref{Lemma}{lemma:lyapunov_distribution_stability}.

\begin{restatable}{mytheo}{restateThmParameterStability}
    \label{theo:parameter_stability_means_distribution_stability}
    Assume $\gamma=1$.
    Let $\phi_\star$ be an optimal policy of \namedref{Problem}{eq:parameter_based}, and define $F_{\phi_\star}(\omega)
        :=
        f(\omega,\phi_\star(\omega)).$
    Let $\Omega_{\mathcal{P}}\subseteq\Omega$ be positively invariant under $F_{\phi_\star}$, and suppose that $F_{\phi_\star}(\omega_\star)=\omega_\star$.
    Suppose that $U_\star^{\mathsf{R}}$ is finite, nonnegative, and continuous on $\Omega_{\mathcal{P}}$ and satisfies the rotated Bellman identity~\eqref{eq:rotated_bellman_identity_parameter} for all $\omega\in\Omega_{\mathcal{P}}$.
    Suppose further that there exist class-$\mathcal{K}$ functions $\kappa_1$ and $\kappa_2$ such that
    \begin{equation}
        \label{eq:rotated_dissipativity_parameter}
        \ell^{\mathsf{R}}\left(\omega,\phi_\star(\omega)\right)
        \geq
        \kappa_1\left(\|\omega-\omega_\star\|\right),
        \quad
        \forall \omega\in\Omega_{\mathcal{P}},
    \end{equation}
    \begin{equation}
        \label{eq:rotated_value_upper_bound_parameter}
        U_\star^{\mathsf{R}}(\omega)
        \leq
        \kappa_2\left(\|\omega-\omega_\star\|\right),
        \quad
        \forall \omega\in\Omega_{\mathcal{P}}.
    \end{equation}
    Suppose that the parameterization is $\mathbb{D}$-compatible on $\Omega_{\mathcal{P}}$ in the sense of \namedref{Assumption}{assump:D_compatible_parameterization}.
    Then, the optimal policy $\phi_\star$ renders the induced distribution dynamics $\mathbb{D}$-asymptotically stable with respect to $\rho_\star=\rho_{\omega_\star}$ on $\{\rho_\omega\mid \omega\in\Omega_{\mathcal{P}}\}$.
\end{restatable}
\noindent
The proofs of \namedref{Lemma}{lemma:lyapunov_distribution_stability} and \namedref{Theorem}{theo:parameter_stability_means_distribution_stability} are provided in \namedref{Appendices}{app:proof_lyapunov_distribution_stability} and~\ref{app:proof_parameter_stability_means_distribution_stability}, respectively.
For the representative stage cost $\ell(\omega,\bm{\mathrm{a}})=\mathbb{D}(\rho_\omega\|\rho_\star)$, which is nonnegative and vanishes at $\omega_\star$, it follows that $\ell_0=0$; hence, the storage function $\lambda=0$ gives $\ell^{\mathsf{R}}=\ell\geq\alpha_1(\|\omega-\omega_\star\|)$ on $\Omega_{\mathcal{P}}$ by \namedref{Assumption}{assump:D_compatible_parameterization}, i.e.,~\eqref{eq:rotated_dissipativity_parameter} with $\kappa_1=\alpha_1$; other stage costs may require a nonzero storage function.

\noindent
\textbf{Intuition behind \namedref{Lemma}{lemma:lyapunov_distribution_stability} and \namedref{Theorem}{theo:parameter_stability_means_distribution_stability}.}
The key insight is that distributional stabilization does not require directly solving a problem over the infinite-dimensional space of measures.
Instead, by parameterizing each distribution $\rho_k$ as $\rho_{\omega_k}$ and analyzing the deterministic parameter dynamics $\omega_{k+1}=f(\omega_k,\phi(\omega_k))$, we can certify stability through a Lyapunov decrease in the finite-dimensional parameter space.

\section{Stability Certification for Learned Policies}
\label{sec:stablization_policy_learning}

The results of the preceding section concern policies and value functions for the exact parameter dynamics.
Here, we study how approximation and estimation errors affect their Lyapunov certificates.
\namedref{Section}{sec:approximate_dynamics_model} transfers a decrease condition computed with an approximate dynamics model to the true dynamics.
\namedref{Section}{sec:deployment_estimated_distribution} then accounts for deployment with estimated distribution parameters.
\namedref{Section}{sec:certification_estimated_lyapunov} gives a complementary certificate based on an estimated Lyapunov decrease.
These results concern the information available for certification and do not prescribe a particular policy-learning algorithm or data-collection schedule.
\namedref{Section}{sec:asymptotic_recovery} studies the corresponding limiting behavior, and \namedref{Section}{sec:practical_implementation} describes the practical implementation settings. 
We use the policy class $\Phi_\Theta
    :=
    \left\{
    \phi(\cdot;\theta):\Omega\rightarrow\mathcal A
    \mid \theta\in\Theta\subseteq\mathbb R^{n_\theta}
    \right\}
    \subseteq\Phi,$
where $\phi(\omega;\theta)$ is continuous in $(\omega,\theta)$.
We omit $\theta$ when its role is clear.
The certificates below apply to a fixed learned policy; stability during policy updates and exploration requires additional analysis.

\subsection{Certification with an Approximate Dynamics Model}
\label{sec:approximate_dynamics_model}

Suppose that an approximate parameter transition map $\hat f_{\mathsf{off}}$ is available for the dynamics~\eqref{eq:parameter_dynamics}.
One way to obtain it is to use $M$ independent trajectories generated from an initial distribution under an action sequence:
\begin{equation}
    \label{eq:data_time_series}
    \mathcal D_{\mathsf{off}}
    :=\left\{\mathcal D_{\mathsf{off}}^{(i)}\right\}_{i=1}^{M},
    \qquad
    \mathcal D_{\mathsf{off}}^{(i)}
    :=
    \left\{
    \left(
    \bm{\mathrm s}_{k}^{(i)},
    \bm{\mathrm a}_{k},
    \bm{\mathrm s}_{k+1}^{(i)}
    \right)
    \right\}_{k=0}^{T-1}.
\end{equation}
At each $k$, the samples $\{\bm{\mathrm s}_{k}^{(i)}\}_{i=1}^{M}$ share the distribution $\rho_k$ and yield an estimate $\hat\omega_k$.
The resulting tuples $(\hat\omega_k,\bm{\mathrm a}_k,\hat\omega_{k+1})$ are used to learn $\hat f_{\mathsf{off}}$.
Here, $N=TM$ is the number of state-transition samples.
The density models discussed after
\namedref{Assumption}{assump:parametric_rho} provide possible estimators.
This construction is one source of the approximate model; the certificate below depends on its approximation errors.

\paragraph{Candidate policy and value function.}
Define the approximate rotated stage cost $\widehat\ell_{\mathsf{off}}^{\mathsf R}
    (\omega,\bm{\mathrm a})
    :=
    \ell(\omega,\bm{\mathrm a})-\ell_0
    -\lambda\!\left(\hat f_{\mathsf{off}}
    (\omega,\bm{\mathrm a})\right)+\lambda(\omega).$
The corresponding rotated optimization problem is
\begin{equation}
    \label{eq:parameter_based_approximation}
    \tag{$\mathsf{AP\text{-}OCP}^{\mathsf R}$}
    \begin{aligned}
        \min_{\phi\in\Phi_\Theta}\quad&
        \lim_{H\rightarrow\infty}
        \left[
        \lambda(\omega_H)
        +\sum_{k=0}^{H-1}
        \widehat\ell_{\mathsf{off}}^{\mathsf R}
        \left(\omega_k,\phi(\omega_k)\right)
        \right]\quad
        \text{s.t.}\quad
        \omega_{k+1}
        =\hat f_{\mathsf{off}}
        \left(\omega_k,\phi(\omega_k)\right),
        \ \omega_0=\omega.
    \end{aligned}
\end{equation}
Let $\widehat U_{\mathsf{off}}^{\mathsf R}$ denote its optimal value.
A policy $\hat\phi_{\mathsf{off}}$ may be obtained from
\begin{equation}
    \label{eq:parameter_based_approximation_no_rotation}
    \tag{$\mathsf{AP\text{-}OCP}$}
    \begin{aligned}
        \min_{\phi\in\Phi_\Theta}\quad&
        \sum_{k=0}^{\infty}
        \gamma^k\ell\left(\omega_k,\phi(\omega_k)\right)\quad
        \text{s.t.}\quad
        \omega_{k+1}
        =\hat f_{\mathsf{off}}
        \left(\omega_k,\phi(\omega_k)\right),
        \ \omega_0=\omega.
    \end{aligned}
\end{equation}
For $\gamma=1$, replacing $\ell$ by $\ell-\ell_0$ in \namedref{Problem}{eq:parameter_based_approximation_no_rotation} and retaining the terminal storage term in \namedref{Problem}{eq:parameter_based_approximation} gives the policy equivalence described in \namedref{Section}{sec:cost_rotation}.
For discounted or approximate optimization, stability instead requires the certificate below.
\namedref{Problem}{eq:parameter_based_approximation} serves the analysis: the certificate uses its optimal value $\widehat U_{\mathsf{off}}^{\mathsf R}$ as a Lyapunov candidate but not the optimality of $\hat\phi_{\mathsf{off}}$, which may therefore be computed from \namedref{Problem}{eq:parameter_based_approximation_no_rotation}, without $\lambda$, or obtained by reinforcement learning.
For the representative stage cost and $\lambda=0$, $\widehat U_{\mathsf{off}}^{\mathsf R}$ is the optimal value of \namedref{Problem}{eq:parameter_based_approximation_no_rotation} with $\gamma=1$, and the certificate of \namedref{Section}{sec:certification_estimated_lyapunov} requires no storage function.

\paragraph{Certificate quantities.}
Define $\widehat F_{\mathsf{off}}(\omega)
    :=\hat f_{\mathsf{off}}
    \left(\omega,\hat\phi_{\mathsf{off}}(\omega)\right),
    \
    F_{\mathsf{off}}(\omega)
    :=f\left(\omega,\hat\phi_{\mathsf{off}}(\omega)\right).$
Fix a compact operating region $\Omega_{\mathsf{com}}\subseteq\Omega$ containing $\omega_\star$.
The Lyapunov bounds used below are
\begin{equation}
    \label{eq:finite_sample_value_bounds}
    \underline\alpha\left(\|\omega-\omega_\star\|\right)
    \leq U_\star^{\mathsf R}(\omega)
    \leq\overline\alpha\left(\|\omega-\omega_\star\|\right),
    \qquad \forall\omega\in\Omega_{\mathsf{com}},
\end{equation}
where $\underline\alpha,\overline\alpha\in\mathcal K$.
Define the value-function and successor-induced errors by
\begin{equation}
    \label{eq:def_delta_U_finite_sample}
    \Delta_U
    :=\sup_{\omega\in\Omega_{\mathsf{com}}}
    \left|
    \widehat U_{\mathsf{off}}^{\mathsf R}(\omega)
    -U_\star^{\mathsf R}(\omega)
    \right|,
\end{equation}
\begin{equation}
    \label{eq:def_delta_f_finite_sample}
    \Delta_f
    :=\sup_{\omega\in\Omega_{\mathsf{com}}}
    \left|
    U_\star^{\mathsf R}\left(F_{\mathsf{off}}(\omega)\right)
    -U_\star^{\mathsf R}
    \left(\widehat F_{\mathsf{off}}(\omega)\right)
    \right|,
    \qquad
    \Delta:=\Delta_f+2\Delta_U.
\end{equation}
The approximate decrease condition is
\begin{equation}
    \label{eq:learned_lyapunov_decrease_margin_finite_sample}
    \widehat U_{\mathsf{off}}^{\mathsf R}
    \left(\widehat F_{\mathsf{off}}(\omega)\right)
    -\widehat U_{\mathsf{off}}^{\mathsf R}(\omega)
    \leq-\sigma\left(\|\omega-\omega_\star\|\right),
    \qquad \forall\omega\in\Omega_{\mathsf{com}},
\end{equation}
with $\sigma\in\mathcal K$.
For an error bound $\bar\Delta\geq0$, define the residual set
\begin{equation}
    \label{eq:finite_sample_residual_set}
    \Omega_{\bar\Delta}
    :=
    \left\{
    \omega\in\Omega_{\mathsf{com}}
    \ \middle|\
    \sigma\left(\|\omega-\omega_\star\|\right)
    \leq\bar\Delta
    \right\}.
\end{equation}

Since the data $\mathcal D_{\mathsf{off}}$ are random, so are $\hat f_{\mathsf{off}}$, $\hat\phi_{\mathsf{off}}$, $\widehat U_{\mathsf{off}}^{\mathsf R}$, and $\Delta$.
An event is a set of outcomes of the random data or samples, not of parameters $\omega$, and a statement holds on an event if it holds for every outcome in it; $\mathcal E_{\mathsf{off}}$ denotes the event on which $\Delta\leq\bar\Delta$ and~\eqref{eq:learned_lyapunov_decrease_margin_finite_sample} hold.

\begin{restatable}{mytheo}{restateThmFiniteSample}
    \label{theo:stability_finite_samples}
    With the quantities defined in \namedref{Section}{sec:approximate_dynamics_model}, suppose that $\Omega_{\mathsf{com}}$ is positively invariant under $F_{\mathsf{off}}$ and $\widehat F_{\mathsf{off}}$,
    $F_{\mathsf{off}}(\omega_\star)=\omega_\star$, and $U_\star^{\mathsf R}$ is continuous and satisfies \eqref{eq:finite_sample_value_bounds}.
    Assume \namedref{Assumption}{assump:D_compatible_parameterization} on $\Omega_{\mathsf{com}}$.
    If the event $\mathcal E_{\mathsf{off}}$ has probability at least $1-\alpha$, $\alpha\in(0,1)$, then, on $\mathcal E_{\mathsf{off}}$,
    \begin{equation}
        \label{eq:true_lyapunov_decrease_with_residual}
        U_\star^{\mathsf R}\left(F_{\mathsf{off}}(\omega)\right)
        -U_\star^{\mathsf R}(\omega)
        \leq-\sigma\left(\|\omega-\omega_\star\|\right)
        +\bar\Delta,
        \qquad \forall\omega\in\Omega_{\mathsf{com}}.
    \end{equation}
    Thus, $U_\star^{\mathsf R}$ strictly decreases outside $\Omega_{\bar\Delta}$, which shrinks to
    $\{\omega_\star\}$ as $\bar\Delta\rightarrow0$.
    If $\bar\Delta=0$, or, more generally, if state-dependent errors yield $U_\star^{\mathsf R}(F_{\mathsf{off}}(\omega))-U_\star^{\mathsf R}(\omega)\leq-\tilde\sigma(\|\omega-\omega_\star\|)$ throughout $\Omega_{\mathsf{com}}$ for some $\tilde\sigma\in\mathcal K$, the induced distribution dynamics are $\mathbb D$-asymptotically stable with respect to $\rho_\star$ on
    $\{\rho_\omega\mid\omega\in\Omega_{\mathsf{com}}\}$, on $\mathcal E_{\mathsf{off}}$.
\end{restatable}
\noindent
The proof is given in \namedref{Appendix}{app:proof_stability_finite_samples}.
The theorem transfers an approximate decrease certificate to the true dynamics with residual $\Delta_f+2\Delta_U$.
It assumes suitable uniform error bounds and invariance on the operating region; these must be established for the chosen estimation and learning procedure.

\subsection{Deployment with Estimated Distribution Parameters}
\label{sec:deployment_estimated_distribution}

We now hold $\hat\phi_{\mathsf{off}}$ fixed and account for estimation of its input.
At time $t$, an estimate $\hat\omega_t$ is formed from the available history $\mathcal H_t
    :=
    \left\{
    (\bm{\mathrm s}_k,\bm{\mathrm a}_k)
    \right\}_{k=0}^{t-1}
    \cup\{\bm{\mathrm s}_t\},$ and the action is $\bm{\mathrm a}_t
    =\hat\phi_{\mathsf{off}}(\hat\omega_t).$
A single state observation generally does not determine $\omega_t$.
Possible constructions use a sliding window, $K$-nearest neighbors, or importance-weighted observations.
The certificate assumes an explicit error bound for the resulting estimator.
Retain $F_{\mathsf{off}}$ for exact-parameter feedback and define $F_{\mathsf{on}}(\omega,\tilde\omega)
    :=f\left(\omega,\hat\phi_{\mathsf{off}}(\tilde\omega)\right).$
Suppose that we have the certificate
\begin{equation}
    \label{eq:exact_parameter_online_certificate}
    U_\star^{\mathsf R}\left(F_{\mathsf{off}}(\omega)\right)
    -U_\star^{\mathsf R}(\omega)
    \leq-\sigma_{\mathsf{on}}
    \left(\|\omega-\omega_\star\|\right)
    +\bar\Delta_{\mathsf{off}},
    \qquad \forall\omega\in\Omega_{\mathsf{com}},
\end{equation}
where $\sigma_{\mathsf{on}}\in\mathcal K$ and
$\bar\Delta_{\mathsf{off}}\geq0$.
This may be supplied by \namedref{Theorem}{theo:stability_finite_samples}.
Define the additional perturbation
\begin{equation}
    \label{eq:online_plugin_perturbation_bound}
    \Delta_{\mathsf{on}}(\delta)
    :=
    \sup_{\substack{
    \omega,\tilde\omega\in\Omega_{\mathsf{com}}\\
    \|\tilde\omega-\omega\|\leq\delta}}
    \left|
    U_\star^{\mathsf R}
    \left(F_{\mathsf{on}}(\omega,\tilde\omega)\right)
    -U_\star^{\mathsf R}
    \left(F_{\mathsf{off}}(\omega)\right)
    \right|.
\end{equation}
For $\delta_{\mathsf{on}}\geq0$, the required estimation event is
\begin{equation}
    \label{eq:online_estimator_uniform_error}
    \|\hat\omega_t-\omega_t\|\leq\delta_{\mathsf{on}},
    \qquad
    \hat\omega_t\in\Omega_{\mathsf{com}},
    \qquad \forall t\geq0,
\end{equation}
simultaneously for every initial $\omega_0\in\Omega_{\mathsf{com}}$.
The associated residual set is
\begin{equation}
    \label{eq:online_residual_set}
    \Omega_{\mathsf{on}}
    :=
    \left\{
    \omega\in\Omega_{\mathsf{com}}
    \ \middle|\
    \sigma_{\mathsf{on}}\left(\|\omega-\omega_\star\|\right)
    \leq
    \bar\Delta_{\mathsf{off}}
    +\Delta_{\mathsf{on}}(\delta_{\mathsf{on}})
    \right\}.
\end{equation}

\begin{restatable}{mytheo}{restateThmOnlinePlugin}
    \label{theo:online_plugin_stability}
    Suppose that $U_\star^{\mathsf R}$ is continuous on the compact set $\Omega_{\mathsf{com}}$ and satisfies \eqref{eq:finite_sample_value_bounds} and \eqref{eq:exact_parameter_online_certificate}.
    Assume \namedref{Assumption}{assump:D_compatible_parameterization} on $\Omega_{\mathsf{com}}$ and that $F_{\mathsf{on}}(\omega,\tilde\omega)\in\Omega_{\mathsf{com}}$ whenever $\omega,\tilde\omega\in\Omega_{\mathsf{com}}$ and $\|\tilde\omega-\omega\|\leq\delta_{\mathsf{on}}$.
    If the uniform event \eqref{eq:online_estimator_uniform_error} has probability at least $1-\alpha_{\mathsf{on}}$, $\alpha_{\mathsf{on}}\in(0,1)$, then, on that event, every trajectory $\omega_{t+1}=F_{\mathsf{on}}(\omega_t,\hat\omega_t)$ initialized in $\Omega_{\mathsf{com}}$ satisfies
    \begin{equation}
        \label{eq:online_true_lyapunov_decrease_residual}
        U_\star^{\mathsf R}(\omega_{t+1})
        -U_\star^{\mathsf R}(\omega_t)
        \leq
        -\sigma_{\mathsf{on}}\left(\|\omega_t-\omega_\star\|\right)
        +\bar\Delta_{\mathsf{off}}
        +\Delta_{\mathsf{on}}(\delta_{\mathsf{on}}).
    \end{equation}
    Thus, $U_\star^{\mathsf R}$ strictly decreases outside $\Omega_{\mathsf{on}}$.
    If the total residual is zero, or, more generally, if state-dependent errors yield $U_\star^{\mathsf R}(\omega_{t+1})-U_\star^{\mathsf R}(\omega_t)\leq-\tilde\sigma_{\mathsf{on}}(\|\omega_t-\omega_\star\|)$ for all $t\geq0$ and some
    $\tilde\sigma_{\mathsf{on}}\in\mathcal K$, then the induced distribution dynamics are $\mathbb D$-asymptotically stable with respect to $\rho_\star$ on $\{\rho_\omega\mid\omega\in\Omega_{\mathsf{com}}\}$, on the same event.
\end{restatable}
\noindent
The proof is given in
\namedref{Appendix}{app:proof_online_plugin_stability}.
Parameter estimation adds $\Delta_{\mathsf{on}}(\delta_{\mathsf{on}})$ to the exact-parameter residual.
If $U_\star^{\mathsf R}$, $f$, and $\hat\phi_{\mathsf{off}}$ are continuous on the relevant compact sets, then $\Delta_{\mathsf{on}}(\delta)\rightarrow0$ as $\delta\rightarrow0$.
Thus, a vanishing uniform estimation error recovers the exact-parameter certificate.
When \eqref{eq:exact_parameter_online_certificate} holds only on an event, such as $\mathcal E_{\mathsf{off}}$ of \namedref{Theorem}{theo:stability_finite_samples}, the two certificates apply on the intersection of this event and the event~\eqref{eq:online_estimator_uniform_error}.

\subsection{Certification from an Estimated Lyapunov Decrease}
\label{sec:certification_estimated_lyapunov}

A decrease certificate can also be obtained directly from parameter-level transition data, without constructing an explicit approximation of $f$.
At learning iteration $n$, let
\begin{equation}
    \label{eq:model_free_online_dataset}
    \mathcal D_{\mathsf{on}}^{(n)}
    :=
    \left\{
    \left(
    \hat\omega_t^{(j)},
    \bm{\mathrm a}_t^{(j)},
    \hat\omega_{t+1}^{(j)}
    \right)
    \right\}_{j=1}^{N_n}
\end{equation}
contain estimated distribution parameters before and after the recorded actions.
Let $\phi_n:\Omega\rightarrow\mathcal A$ and $W_n:\Omega\rightarrow\mathbb R_{\geq0}$ be a candidate policy and Lyapunov surrogate.
Possible learning procedures include actor-critic methods~\citep{Konda1999}, fitted value iteration~\citep{Munos2008}, and policy-gradient methods~\citep{Sutton1999} with Lyapunov regularization~\citep{Chang2019,Han2020}.
The result below certifies the resulting pair through its decrease and error bounds.

\paragraph{Certificate quantities.}
For analysis, define the true closed-loop maps
\begin{equation}
    \label{eq:model_free_exact_closed_loop}
    F_n(\omega):=f(\omega,\phi_n(\omega)),
\end{equation}
\begin{equation}
    \label{eq:model_free_plugin_closed_loop}
    F_n^{\mathsf{plug}}(\omega,\tilde\omega)
    :=f(\omega,\phi_n(\tilde\omega)).
\end{equation}
These maps specify the dynamics being certified.
Let $\widehat{\mathcal D}_n$ estimate the exact-parameter
decrease
\begin{equation}
    \label{eq:true_exact_decrease_model_free}
    \mathcal D_n(\omega)
    :=W_n(F_n(\omega))-W_n(\omega).
\end{equation}
On a compact region $\Omega_{\mathsf{com}}$ containing
$\omega_\star$, the required Lyapunov bounds and estimated
decrease are
\begin{equation}
    \label{eq:model_free_Wn_bounds}
    \underline\alpha\left(\|\omega-\omega_\star\|\right)
    \leq W_n(\omega)
    \leq\overline\alpha\left(\|\omega-\omega_\star\|\right),
    \qquad \forall\omega\in\Omega_{\mathsf{com}},
\end{equation}
\begin{equation}
    \label{eq:model_free_empirical_decrease_certificate}
    \widehat{\mathcal D}_n(\omega)
    \leq-\sigma_n\left(\|\omega-\omega_\star\|\right),
    \qquad \forall\omega\in\Omega_{\mathsf{com}},
\end{equation}
where $\underline\alpha,\overline\alpha,\sigma_n\in\mathcal K$.
Define the decrease-estimation error and the perturbation
caused by estimating the policy input:
\begin{equation}
    \label{eq:model_free_decrease_estimation_error}
    \Delta_{\mathsf{dec},n}
    :=
    \sup_{\omega\in\Omega_{\mathsf{com}}}
    \left|
    \widehat{\mathcal D}_n(\omega)-\mathcal D_n(\omega)
    \right|,
\end{equation}
\begin{equation}
    \label{eq:model_free_parameter_estimation_perturbation}
    \Delta_{\omega,n}(\delta)
    :=
    \sup_{\substack{
    \omega,\tilde\omega\in\Omega_{\mathsf{com}}\\
    \|\tilde\omega-\omega\|\leq\delta}}
    \left|
    W_n\left(F_n^{\mathsf{plug}}(\omega,\tilde\omega)\right)
    -W_n(F_n(\omega))
    \right|.
\end{equation}
For bounds $\bar\Delta_{\mathsf{dec},n}\geq0$ and
$\delta_{\omega,n}\geq0$, the required joint event is
\begin{equation}
    \label{eq:model_free_high_probability_event}
    \Delta_{\mathsf{dec},n}
    \leq\bar\Delta_{\mathsf{dec},n},
    \qquad
    \|\hat\omega_t-\omega_t\|\leq\delta_{\omega,n},
    \qquad
    \hat\omega_t\in\Omega_{\mathsf{com}},
    \qquad \forall t\geq0,
\end{equation}
simultaneously for every initial
$\omega_0\in\Omega_{\mathsf{com}}$.
The corresponding invariance condition is
\begin{equation}
    \label{eq:model_free_plugin_invariance}
    F_n^{\mathsf{plug}}(\omega,\tilde\omega)
    \in\Omega_{\mathsf{com}},
    \qquad
    \forall\omega,\tilde\omega\in\Omega_{\mathsf{com}}
    \text{ with }
    \|\tilde\omega-\omega\|\leq\delta_{\omega,n}.
\end{equation}
For deployment of the fixed policy $\phi_n$, write
\begin{equation}
    \label{eq:model_free_plugin_dynamics}
    \omega_{t+1}
    =F_n^{\mathsf{plug}}(\omega_t,\hat\omega_t),
\end{equation}
\begin{equation}
    \label{eq:model_free_residual_set}
    \Omega_{\mathsf{mf},n}
    :=
    \left\{
    \omega\in\Omega_{\mathsf{com}}
    \ \middle|\
    \sigma_n\left(\|\omega-\omega_\star\|\right)
    \leq
    \bar\Delta_{\mathsf{dec},n}
    +\Delta_{\omega,n}(\delta_{\omega,n})
    \right\}.
\end{equation}

\begin{restatable}[\textbf{Certificate from estimated Lyapunov decrease}]{mytheo}{restateThmModelFree}
    \label{theo:model_free_online_policy_learning}
    Fix $n$ and the candidate pair $(\phi_n,W_n)$.
    Suppose that $W_n$ is continuous on
    $\Omega_{\mathsf{com}}$ and that \eqref{eq:model_free_Wn_bounds}, \eqref{eq:model_free_empirical_decrease_certificate}, \eqref{eq:model_free_plugin_invariance}, and \namedref{Assumption}{assump:D_compatible_parameterization} hold on this compact set.
    If the joint event \eqref{eq:model_free_high_probability_event} has probability at least $1-\alpha_n$, $\alpha_n\in(0,1)$, then, on that event, every trajectory of \eqref{eq:model_free_plugin_dynamics} initialized in $\Omega_{\mathsf{com}}$ satisfies
    \begin{equation}
        \label{eq:model_free_true_decrease_with_residual}
        W_n(\omega_{t+1})-W_n(\omega_t)
        \leq
        -\sigma_n\left(\|\omega_t-\omega_\star\|\right)
        +\bar\Delta_{\mathsf{dec},n}
        +\Delta_{\omega,n}(\delta_{\omega,n}).
    \end{equation}
    Thus, $W_n$ strictly decreases outside
    $\Omega_{\mathsf{mf},n}$.
    If the total residual is zero, or, more generally, if state-dependent errors yield
    $W_n(\omega_{t+1})-W_n(\omega_t)\leq-\tilde\sigma_n(\|\omega_t-\omega_\star\|)$
    for all $t\geq0$ and some $\tilde\sigma_n\in\mathcal K$,
    then the induced distribution dynamics are $\mathbb D$-asymptotically stable with respect to $\rho_\star$ on $\{\rho_\omega\mid\omega\in\Omega_{\mathsf{com}}\}$, on the same event.
\end{restatable}
\noindent
The proof is given in
\namedref{Appendix}{app:proof_model_free_online_policy_learning}.
The certificate has two error contributions:
$\bar\Delta_{\mathsf{dec},n}$ accounts for estimating the
decrease, and $\Delta_{\omega,n}(\delta_{\omega,n})$
accounts for estimating the policy input.
The index $n$ is held fixed throughout deployment;
the theorem does not establish stability while the policy
and surrogate are being updated.
Stability during data collection and exploration requires
additional guarantees, as in safe-exploration
methods~\citep{Garcia2015,Achiam2017,Berkenkamp2017}.

\subsection{Asymptotic Recovery and Limiting Certificates}
\label{sec:asymptotic_recovery}

The preceding results certify a fixed candidate under
nonzero errors.
We next consider the limit as the available data increase,
first for optimization with an approximate dynamics model
and then for certification from an estimated decrease.

\paragraph{Objectives and optimizers.}
For each dataset size $N$, let
$\hat f_{\mathsf{off}}^N$ be the approximate model.
For a fixed initial $\omega\in\Omega_{\mathsf{com}}$
and $\theta\in\Theta$, the rotated objective of the
policy $\phi(\cdot;\theta)$ under $f$ is
$J^{\mathsf R}(\theta;\omega)
    :=\lim_{H\rightarrow\infty}
    \left[
    \lambda(\omega_H)
    +\sum_{k=0}^{H-1}
    \ell^{\mathsf R}\left(\omega_k,\phi(\omega_k;\theta)\right)
    \right],$
subject to
$\omega_{k+1}=f\left(\omega_k,\phi(\omega_k;\theta)\right)$
and $\omega_0=\omega$, and
$\widehat J_N^{\mathsf R}(\theta;\omega)$ is the objective
of \namedref{Problem}{eq:parameter_based_approximation}
for the same policy with
$\hat f_{\mathsf{off}}=\hat f_{\mathsf{off}}^N$.
Both objectives use the same storage function $\lambda$
and constant $\ell_0$, neither of which depends on
$\theta$.
By~\eqref{eq:rotated_telescope_parameter}, applied along
the trajectory of each model, each objective equals
$\lambda(\omega)$ plus the undiscounted objective of
\namedref{Problem}{eq:parameter_based}
for $\phi(\cdot;\theta)$ under that model, so the rotation
does not change the minimizers.
Define $\Theta_\star^{\mathsf R}
    :=\arg\min_{\theta\in\Theta}
    J^{\mathsf R}(\theta;\omega),
    \
    \widehat\theta_N
    \in\arg\min_{\theta\in\Theta}
    \widehat J_N^{\mathsf R}(\theta;\omega).$
\namedref{Proposition}{prop:asymptotic_recovery_rotated_optimizers}
assumes, rather than proves, two convergence conditions:
uniform model consistency,
\begin{equation}
    \label{eq:uniform_model_consistency_main}
    \sup_{\substack{
    \omega\in\Omega_{\mathsf{com}}\\
    \bm{\mathrm a}\in\mathcal A}}
    \left\|
    \hat f_{\mathsf{off}}^N(\omega,\bm{\mathrm a})
    -f(\omega,\bm{\mathrm a})
    \right\|
    \longrightarrow0
    \quad\text{almost surely as }N\rightarrow\infty,
\end{equation}
and uniform convergence of the infinite-horizon objectives,
\begin{equation}
    \label{eq:uniform_rotated_objective_consistency}
    \sup_{\theta\in\Theta}
    \left|
    \widehat J_N^{\mathsf R}(\theta;\omega)
    -J^{\mathsf R}(\theta;\omega)
    \right|
    \longrightarrow0
    \quad\text{almost surely as }N\rightarrow\infty.
\end{equation}
Fixed-horizon convergence alone does not establish
the second condition.

\begin{restatable}{myprop}{restatePropAsymptoticRecovery}
    \label{prop:asymptotic_recovery_rotated_optimizers}
    Let $\gamma=1$, let
    $\Omega_{\mathsf{com}}$, $\Theta$, and $\mathcal A$
    be compact, and suppose that
    $f$, $\ell$, $\lambda$, and $\phi(\omega;\theta)$
    are continuous.
    Assume that
    \eqref{eq:uniform_model_consistency_main}
    and \eqref{eq:uniform_rotated_objective_consistency} hold,
    and that both rotated objectives are well defined
    and continuous in $\theta$ on $\Theta$.
    Suppose that $\Omega_{\mathsf{com}}$ is positively
    invariant for every true closed loop indexed by
    $\theta\in\Theta$, and that, almost surely, there is a
    finite $N_0$ such that it is also positively invariant
    for every learned closed loop whenever
    $N\geq N_0$ and $\theta\in\Theta$.
    Then $\operatorname{dist}
        \left(\widehat\theta_N,\Theta_\star^{\mathsf R}\right)
        \longrightarrow0
        \ \text{almost surely}.$
    Consequently, every limit point $\bar\theta$ of
    $\{\widehat\theta_N\}_{N=1}^{\infty}$ belongs to
    $\Theta_\star^{\mathsf R}$ almost surely.
    If every policy indexed by
    $\bar\theta\in\Theta_\star^{\mathsf R}$
    satisfies the certificate in
    \namedref{Lemma}{lemma:lyapunov_distribution_stability}
    on $\Omega_{\mathsf{com}}$, and
    \namedref{Assumption}{assump:D_compatible_parameterization}
    holds there, then every such limit policy renders the
    induced distribution dynamics
    $\mathbb D$-asymptotically stable with respect to
    $\rho_\star$ on
    $\{\rho_\omega\mid\omega\in\Omega_{\mathsf{com}}\}$.
\end{restatable}
\noindent
The proof is given in
\namedref{Appendix}{app:proof_asymptotic_recovery_rotated_optimizers}.
Uniform model consistency requires sufficient coverage
of $\Omega_{\mathsf{com}}\times\mathcal A$
and suitable estimation conditions~\citep{Munos2008,Chen2019}.
\namedref{Appendix}{appendix:example_pipeline_model_consistency}
describes one possible pipeline based on conditional
density estimation, maximum-likelihood estimation,
and Gaussian process regression.
With partial coverage, uncertainty regularization,
such as that used in MOPO and MOBILE~\citep{Yu2020,Sun2023},
can help control optimization outside the supported
region, but does not by itself establish the uniform
consistency required by the proposition.
For \namedref{Theorem}{theo:model_free_online_policy_learning},
consider a sequence for which $\bar\Delta_{\mathsf{dec},n}
    +\Delta_{\omega,n}(\delta_{\omega,n})
    \longrightarrow 0.$
If a subsequence of $(\phi_n,W_n)$ converges uniformly to a pair $(\phi_\infty,W_\infty)$ that satisfies all conditions of \namedref{Lemma}{lemma:lyapunov_distribution_stability}, including a decrease condition with a class-$\mathcal K$ function, and the parameterization is $\mathbb D$-compatible on the operating region, $\phi_\infty$ is $\mathbb D$-asymptotically stable.
Vanishing residuals and convergence of the pair alone are insufficient because the decrease margins $\sigma_n$ may also vanish.
This is a certificate for the limiting policy. 
Optimizer recovery is established under \namedref{Proposition}{prop:asymptotic_recovery_rotated_optimizers}'s assumptions.

\subsection{Practical Implementation}
\label{sec:practical_implementation}

The experiments train PPO and SAC using distribution
estimates as policy inputs, with three departures from
the idealized problem~\eqref{eq:maintext_stability_problem}.
First, a deployed system provides one noisy trajectory,
so $\hat\omega_t$ is formed from a sliding window as
described in \namedref{Section}{sec:deployment_estimated_distribution}.
The estimate $\hat\omega_t$ contains the means and standard deviations of the
last $n_{\mathsf w}$ noisy readings of the regulated
features (the forward velocity and the torso pitch
of HalfCheetah).
Second, to represent the body's configuration within
the gait cycle, the policy also receives the latest
benchmark observation $\bm{\mathrm y}_t$ with sensor
noise of standard deviation $\eta$,
$\bm{\mathrm o}_t=\bm{\mathrm y}_t+\eta\,\bm{\xi}_t$, so that
$\bm{\mathrm a}_t=\phi(\hat\omega_t,\bm{\mathrm o}_t)$.
Third, DMDP (PPO/SAC) maximizes the return of
$r_t^{\mathsf b}+r_{t\mathbb D}-\beta\,\mathbb D_t$,
with the benchmark reward $r_t^{\mathsf b}$,
the dissimilarity $\mathbb D_t$ of the empirical
distribution of the true feature values in the window
to $\rho_\star$, its weight $\beta>0$, and the threshold
reward $r_{t\mathbb D}$, a constant bonus whenever
$\mathbb D_t\leq\mathbb D_{\mathsf{thr}}$.
\namedref{Figure}{fig:implementation_pipeline} in
\namedref{Appendix}{appendix:implementation} shows this
closed loop, to which the preceding certificates apply
once their hypotheses are verified for its actual
inputs, induced dynamics, and estimation errors.

\section{Empirical Validation}
\label{sec:validations}

Does a DMDP bring the distribution of the velocity and the pitch of the closed loop close to the target $\rho_\star$ and keep it stable?
In the HalfCheetah task of \namedref{Table}{tab:task} in \namedref{Appendix}{appendix:implementation}, the benchmark reward $r^{\mathsf{b}}_t$ favors fast forward running, and the safety cost $c_t$ indicates each step whose forward velocity $v_t$ exceeds the velocity limit $v_{\mathsf{lim}}$.
The target $\rho_\star$, uniform on the product of a velocity band below $v_{\mathsf{lim}}$ and a narrow pitch interval around zero, encodes fast and safe running with an upright torso.

The DMDP methods are compared with the baselines of \namedref{Table}{tab:methods} in \namedref{Appendix}{appendix:implementation}, with SAC in this section and with PPO in \namedref{Appendix}{appendix:ppo_results}, where DMDP-Lag (PPO) is likewise the closest of all methods to the target without noise.
The critics of the Distributional baselines learn quantiles of the return instead of its expectation~\citep{DabneyQR2018}.
Since Implicit DMDP (PPO/SAC) differs from the Distributional method of the same learner only in the policy input, it isolates the contribution of feeding back $\hat{\omega}_t$ without $\rho_\star$; its return distribution depends on the state distribution only implicitly, whereas the reward of the DMDP methods depends on it explicitly through $\mathbb{D}_t$.
The -Lag methods enforce the episodic cost constraint $\mathbb{E}\big[\sum_{t=0}^{H-1}c_t\big]\leq\bar{c}$ with the cost limit $\bar{c}$~\citep{Altman1999}, adapting a Lagrange multiplier $\lambda\geq0$ as in PPO-Lagrangian~\citep{Ray2019} and SAC-Lagrangian~\citep{Ha2020}.

Every method is trained with five seeds and a random $\eta$ in each training episode, and each policy is evaluated with deterministic actions on $64$ episodes per seed at the sensor-noise levels $\eta\in\{0,0.1,\ldots,0.5\}$.
The evaluation tables report the return of the evaluation reward $r^{\mathsf{eva}}_t$ of \namedref{Table}{tab:task}. Unlike the benchmark reward, it does not reward violating the velocity limit, and its threshold reward $r_{t\mathbb{D}}$ is also computed for the baselines, which are not trained on it.
Whereas $\hat{\omega}_t$ and $\mathbb{D}_t$ use windows of one trajectory, the evaluation uses the velocity marginal $\rho_k$ of the state distribution at step $k$ across the evaluation episodes and the stable velocity distribution $\hat{\rho}$ pooled over their steps $k\geq500$.
The target distance $\mathbb{D}(\hat{\rho}\|\rho^{v}_\star)$ to the velocity marginal $\rho^{v}_\star$ of $\rho_\star$ and the stability distance $\mathbb{D}(\rho_k\|\hat{\rho})$ are $1$-Wasserstein distances. The evaluation tables average the stability distance over $k<200$ and over $k\geq500$, and their last column is the fraction of evaluation steps $k\geq500$ with $\mathbb{D}_t\leq\mathbb{D}_{\mathsf{thr}}$.
The reliability of the return $R$ of each method's own training reward is measured by the probability $\mathbb{P}[R\geq R_{\mathsf{Thr}}]$, one point of the episode-level counterpart of performance profiles~\citep{Chan2020}, and by the $10$--$90\%$ range of $R$, with $R_{\mathsf{Thr}}=5000$ (an average reward of $5$ per step).

\subsection{Results with SAC}
\label{sec:exp_sac}

\namedref{Table}{tab:sac_main_results} shows that, at both noise levels, DMDP-Lag (SAC) brings its stable velocity distribution closest to the target and has the highest evaluation return and the largest fraction of steps with $\mathbb{D}_t\leq\mathbb{D}_{\mathsf{thr}}$. Without noise, it leads clearly in all three and also has the lowest cost of all methods.
In \namedref{Figure}{fig:sac_stable_laws}, its noise-free stable distribution is concentrated in the velocity band.
\namedref{Table}{tab:sac_reliability} shows that DMDP-Lag (SAC) reaches $R_{\mathsf{Thr}}$ most often of all methods at the end of training and in the evaluation at both noise levels, clearly more often than DMDP (SAC), whereas the six baselines never reach it.
In \namedref{Figure}{fig:sac_training_curves}, the mean $R$ of each DMDP method rises above $R_{\mathsf{Thr}}$ within about $2\times10^{5}$ training steps, whereas that of each baseline stays below it; from then on, the median window dissimilarity of DMDP-Lag (SAC) lies below $\mathbb{D}_{\mathsf{thr}}$.
The cost constraint is effective: at both noise levels, each -Lag method has a clearly lower cost than each method without a cost constraint, and in the second half of training, the costs of the -Lag methods stay near or below $\bar{c}$, whereas those of the other methods stay far above it. 
With the cost constraint, feeding $\hat{\omega}_t$ back to the policy helps even without $\rho_\star$: at both noise levels, Implicit DMDP-Lag (SAC) is second only to DMDP-Lag (SAC) in evaluation return and in distance to the target, and it is clearly better than Distributional SAC-Lag in both.
None of the four methods that feed back $\hat{\omega}_t$ has a flipped run at either noise level, whereas SAC, SAC-Lag, and especially the Distributional methods have flipped runs at both, and the low velocities of the flipped runs widen the distributions of these methods in \namedref{Figure}{fig:sac_stable_laws}. 
Knowing $\rho_\star$ helps further: DMDP (SAC) is clearly closer to the target than Implicit DMDP (SAC) at both noise levels, at a clearly lower cost.
The policy of Implicit DMDP (SAC) spends most steps above $v_{\mathsf{lim}}$, as its cost shows, so its evaluation return is negative and clearly the lowest of all methods at both noise levels, and its stable distribution is clearly the farthest from the target.
Two runs of DMDP (SAC) stay in the velocity band and reach $R_{\mathsf{Thr}}$ in every episode, forming the sharp peak in the band in \namedref{Figure}{fig:sac_stable_laws}; its other three runs do neither and form the wide part above $v_{\mathsf{lim}}$.

\begin{table}[!t]
    \centering
    \caption{Evaluation results with SAC.}
    \label{tab:sac_main_results}
    \resizebox{\linewidth}{!}{\begin{tabular}{llcccccc}
\toprule
Method & $\eta$ & \begin{tabular}[b]{@{}c@{}}return \\ $r^{\mathsf{eva}}_t$ $\uparrow$\end{tabular} & cost $\downarrow$ & \begin{tabular}[b]{@{}c@{}}$\mathbb{D}(\rho_k\|\hat{\rho})$ \\ $k<200$ $\downarrow$\end{tabular} & \begin{tabular}[b]{@{}c@{}}$\mathbb{D}(\rho_k\|\hat{\rho})$ \\ $k\geq500$ $\downarrow$\end{tabular} & \begin{tabular}[b]{@{}c@{}}$\mathbb{D}(\hat{\rho}\|\rho^{v}_\star)$ \\ $\downarrow$\end{tabular} & \begin{tabular}[b]{@{}c@{}}$\mathbb{P}[\mathbb{D}(\bar{\rho}_t\|\rho_\star)$ \\ $\leq\mathbb{D}_{\mathsf{thr}}]$ $\uparrow$\end{tabular} \\
\midrule
SAC & $0$ & $461\pm565$ & $218.9\pm126.1$ & $0.47\pm0.08$ & $0.301\pm0.065$ & $0.98\pm0.17$ & $0.01\pm0.01$ \\
 & $0.2$ & $526\pm616$ & $205.4\pm130.0$ & $0.37\pm0.07$ & $0.197\pm0.021$ & $1.06\pm0.17$ & $0.02\pm0.02$ \\
\midrule
SAC-Lag & $0$ & $3802\pm1920$ & $6.8\pm3.7$ & $0.26\pm0.02$ & $0.155\pm0.022$ & $0.73\pm0.23$ & $0.23\pm0.17$ \\
 & $0.2$ & $3467\pm1929$ & $\bm{4.9\pm2.2}$ & $0.26\pm0.03$ & $0.155\pm0.033$ & $0.81\pm0.24$ & $0.21\pm0.18$ \\
\midrule
Distributional & $0$ & $819\pm141$ & $108.7\pm48.7$ & $0.36\pm0.04$ & $0.233\pm0.026$ & $0.96\pm0.14$ & $0.00\pm0.00$ \\
SAC & $0.2$ & $837\pm103$ & $85.9\pm35.3$ & $0.33\pm0.03$ & $0.205\pm0.018$ & $1.03\pm0.13$ & $0.00\pm0.00$ \\
\midrule
Distributional & $0$ & $1248\pm85$ & $11.2\pm6.6$ & $0.41\pm0.06$ & $0.247\pm0.023$ & $1.04\pm0.12$ & $0.00\pm0.00$ \\
SAC-Lag & $0.2$ & $1185\pm95$ & $7.8\pm3.1$ & $0.37\pm0.04$ & $0.205\pm0.011$ & $1.12\pm0.10$ & $0.00\pm0.00$ \\
\midrule
Implicit DMDP & $0$ & $-2410\pm243$ & $844.6\pm46.6$ & $1.03\pm0.10$ & $0.563\pm0.171$ & $1.88\pm0.13$ & $0.00\pm0.00$ \\
(SAC) & $0.2$ & $-1824\pm329$ & $726.4\pm63.2$ & $0.82\pm0.09$ & $0.427\pm0.136$ & $1.46\pm0.17$ & $0.00\pm0.00$ \\
\midrule
Implicit DMDP-Lag & $0$ & $4729\pm1346$ & $2.2\pm1.7$ & $\bm{0.21\pm0.01}$ & $0.061\pm0.018$ & $0.30\pm0.03$ & $0.22\pm0.14$ \\
(SAC) & $0.2$ & $8824\pm1758$ & $9.2\pm4.8$ & $\bm{0.22\pm0.01}$ & $0.073\pm0.020$ & $0.21\pm0.07$ & $0.68\pm0.17$ \\
\midrule
DMDP (SAC) & $0$ & $4010\pm3325$ & $406.1\pm166.9$ & $0.66\pm0.19$ & $0.446\pm0.176$ & $0.81\pm0.27$ & $0.40\pm0.24$ \\
 & $0.2$ & $4138\pm2980$ & $329.5\pm132.4$ & $0.57\pm0.14$ & $0.311\pm0.107$ & $0.63\pm0.22$ & $0.38\pm0.23$ \\
\midrule
DMDP-Lag (SAC) & $0$ & $\bm{11981\pm105}$ & $\bm{1.0\pm0.5}$ & $\bm{0.21\pm0.01}$ & $\bm{0.056\pm0.006}$ & $\bm{0.13\pm0.01}$ & $\bm{0.99\pm0.01}$ \\
 & $0.2$ & $\bm{10617\pm1191}$ & $6.3\pm3.8$ & $\bm{0.22\pm0.02}$ & $\bm{0.062\pm0.020}$ & $\bm{0.14\pm0.05}$ & $\bm{0.87\pm0.11}$ \\
\bottomrule
\end{tabular}

}
\end{table}

\begin{figure}[!h]
    \centering
    \includegraphics[width=\linewidth]{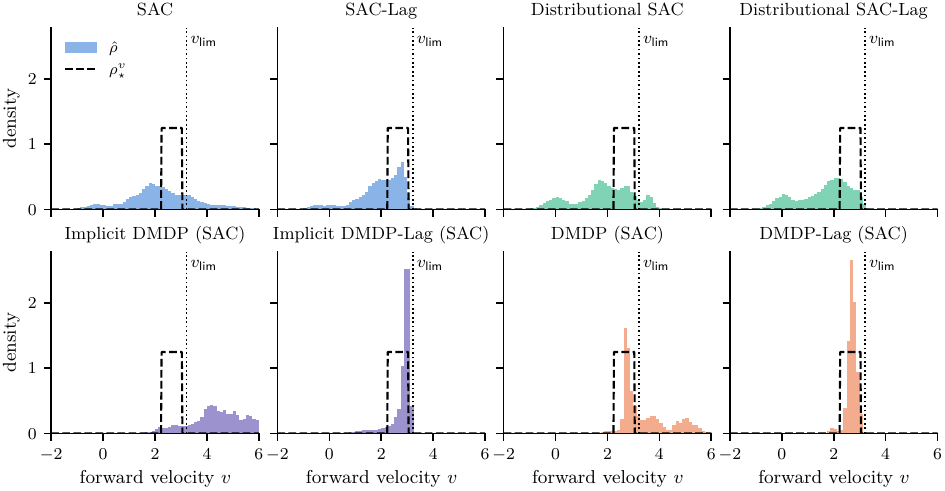}
    \caption{Stable velocity distributions with SAC.}
    \label{fig:sac_stable_laws}
\end{figure}

\begin{table}[!h]
    \centering
    \caption{Reliability of the return with SAC.}
    \label{tab:sac_reliability}
    \footnotesize\setlength{\tabcolsep}{3pt}
    \begin{tabular*}{\linewidth}{@{\extracolsep{\fill}}lcccccc@{}}
\toprule
 & \multicolumn{2}{c}{end of training} & \multicolumn{2}{c}{evaluation, $\eta=0$} & \multicolumn{2}{c}{evaluation, $\eta=0.2$} \\
\cmidrule(lr){2-3}\cmidrule(lr){4-5}\cmidrule(lr){6-7}
Method & $\mathbb{P}[R\geq R_{\mathsf{Thr}}]$ $\uparrow$ & 10--90\% range $\downarrow$ & $\mathbb{P}[R\geq R_{\mathsf{Thr}}]$ $\uparrow$ & 10--90\% range $\downarrow$ & $\mathbb{P}[R\geq R_{\mathsf{Thr}}]$ $\uparrow$ & 10--90\% range $\downarrow$ \\
\midrule
SAC & $0.00\pm0.00$ & $2231$ & $0.00\pm0.00$ & $2154$ & $0.00\pm0.00$ & $2373$ \\
SAC-Lag & $0.00\pm0.00$ & $1303$ & $0.00\pm0.00$ & $1365$ & $0.00\pm0.00$ & $1310$ \\
Distributional SAC & $0.00\pm0.00$ & $911$ & $0.00\pm0.00$ & $1197$ & $0.00\pm0.00$ & $957$ \\
Distributional SAC-Lag & $0.00\pm0.00$ & $\bm{589}$ & $0.00\pm0.00$ & $672$ & $0.00\pm0.00$ & $\bm{586}$ \\
Implicit DMDP (SAC) & $0.00\pm0.00$ & $1248$ & $0.00\pm0.00$ & $685$ & $0.00\pm0.00$ & $1104$ \\
Implicit DMDP-Lag (SAC) & $0.00\pm0.00$ & $668$ & $0.00\pm0.00$ & $\bm{548}$ & $0.00\pm0.00$ & $735$ \\
DMDP (SAC) & $0.40\pm0.24$ & $9545$ & $0.40\pm0.24$ & $8970$ & $0.40\pm0.24$ & $9132$ \\
DMDP-Lag (SAC) & $\bm{0.92\pm0.08}$ & $5088$ & $\bm{1.00\pm0.00}$ & $573$ & $\bm{0.97\pm0.02}$ & $6011$ \\
\bottomrule
\end{tabular*}

\end{table}

In \namedref{Figure}{fig:sac_stability_curves}, the stability distance of every method drops within about $100$ steps and then changes much more slowly.
For $k\geq500$, DMDP-Lag (SAC) has the smallest stability distance of all methods at both noise levels, and Implicit DMDP-Lag (SAC) has the second-smallest, within the seed variation of DMDP-Lag (SAC); both are clearly smaller than those of the other six methods.
Implicit DMDP (SAC) and DMDP (SAC) have the largest stability distances, within the seed variation of each other; without noise, these distances come mostly from the oscillation of $\rho_k$ in their synchronized runs, defined in \namedref{Appendix}{appendix:ppo_results}.

\begin{figure}[!h]
    \centering
    \includegraphics[width=\linewidth]{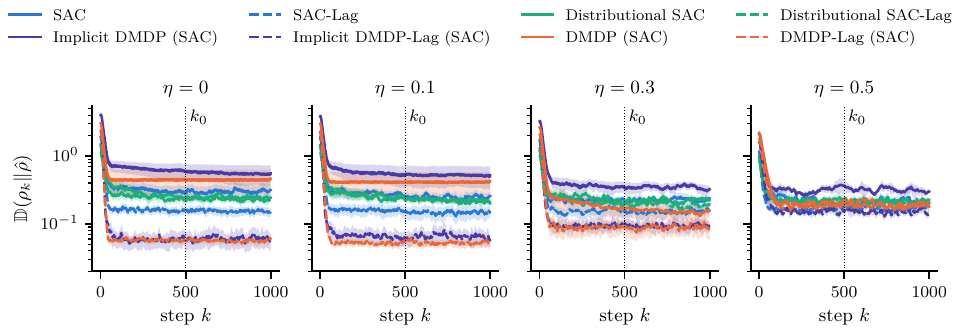}
    \caption{Stability curves with SAC.}
    \label{fig:sac_stability_curves}
\end{figure}

\section{Conclusion}
\label{sec:conclusion}
This paper introduced Dissimilarity-based Markov Decision Processes, which control the distribution of the state of a stochastic system rather than a single state: the policy acts on the state distribution, and the stage cost is its dissimilarity to a target distribution.
A finite-dimensional parameterization turns this problem into the control of deterministic parameter dynamics, for which Lyapunov certificates were derived for the optimal policy, for policies learned offline with a model of the parameter dynamics, for their online deployment with estimated parameters, and for policies whose Lyapunov decrease is estimated from transition data; in the scalar linear example of \namedref{Appendix}{appendix:toy_example}, these certificates hold wherever they are checked.
On HalfCheetah, a sliding-window estimate of the distribution, a window dissimilarity, and a threshold reward enable PPO and SAC to regulate the joint distribution of the velocity and the pitch.
Of all compared methods, DMDP-Lag (PPO/SAC) brings the velocity distribution closest to a fast and safe target and attains the highest evaluation return without sensor noise, and DMDP-Lag (SAC) achieves both also under moderate noise.
Feeding the distribution estimate back to the policy helps even without the target, and knowing the target takes the DMDP methods further.
The Lyapunov certificates of \namedref{Theorems}{theo:parameter_stability_means_distribution_stability}--\ref{theo:model_free_online_policy_learning} were not computed for the learned HalfCheetah policies; in the example of \namedref{Appendix}{appendix:toy_example}, computing them requires the true model, and the certified residual sets are much larger than the sets in which the learned closed loops settle.
Because these certificates hold for any policy whose decrease margin dominates the error terms, whichever learning method produced it, learning a Lyapunov function together with the policy and bounding the error terms from data is a natural next step toward certifying such high-dimensional tasks from data alone.
Richer density models in the parameterization of \namedref{Section}{sec:parameterized_distributions}, such as Gaussian mixtures and normalizing flows, further open the way to regulating multimodal distributions of higher-dimensional states.

\FloatBarrier
\newpage

\bibliography{references}
\bibliographystyle{tmlr}

\appendix
\clearpage
\etocdepthtag.toc{appendix}
\tableofcontents
\clearpage

\section{Proofs of All Main Results}
\label{app:proofs}
Each result of the main text is restated with its original numbering before its proof.

\subsection{Proof of Theorem~\ref*{theo:main_psi_star_stability}}
\label{app:proof_main_psi_star_stability}
\restateThmMainStability*

\begin{proof}
Let $\rho_0\in\mathcal{P}$ and consider the closed-loop dynamics $\rho_{k+1}=\mathcal{T}_{\tau,\psi_\star(\rho_k)}\rho_k$.
Since $\mathcal{P}$ is forward invariant under these dynamics, $\rho_k\in\mathcal{P}$ for all $k\geq 0$.
Since $\gamma=1$, the dynamic programming principle~\citep{Puterman1994,HernandezLerma1996} gives, along the optimal policy $\psi_\star$, the Bellman relation $V_\star[\rho_k]=\mathcal{L}[\rho_k,\psi_\star(\rho_k)]+V_\star[\rho_{k+1}]$.
With the lower bound on the stage cost functional, it follows that
\begin{equation}
    V_\star[\rho_{k+1}]-V_\star[\rho_k]
    =
    -\mathcal{L}\left[\rho_k,\psi_\star(\rho_k)\right]
    \leq
    -\kappa_1\left(\mathbb{D}(\rho_k\|\rho_\star)\right).
    \label{eq:proof_lyapunov_decrease}
\end{equation}
Moreover, $V_\star[\rho_k]=\sum_{j=k}^{\infty}\mathcal{L}[\rho_j,\psi_\star(\rho_j)]\geq 0$ along the optimal closed-loop trajectory, because $\mathcal{L}[\rho_j,\psi_\star(\rho_j)]\geq\kappa_1(\mathbb{D}(\rho_j\|\rho_\star))\geq 0$.
Hence, $V_\star[\rho_k]$ is nonnegative and nonincreasing, and, by the monotone convergence theorem, $\lim_{k\rightarrow\infty}V_\star[\rho_k]=\bar{V}$ for some constant $\bar{V}\geq 0$.
The lower bound on the stage cost functional, the Bellman relation, and $V_\star[\rho_{k+1}]\geq 0$ also give
\begin{equation}
    \kappa_1\left(\mathbb{D}(\rho_k\|\rho_\star)\right)
    \leq
    \mathcal{L}\left[\rho_k,\psi_\star(\rho_k)\right]
    \leq
    V_\star[\rho_k],
    \quad \forall k\geq 0.
\end{equation}

For $\mathbb{D}$-stability, let $\varepsilon>0$.
Since $\kappa_1$ and $\kappa_2$ are class-$\mathcal{K}$ functions, $\kappa_1(\varepsilon)>0$ and there exists $\delta(\varepsilon)>0$ such that $\kappa_2(\delta(\varepsilon))\leq\kappa_1(\varepsilon)$.
If $\mathbb{D}(\rho_0\|\rho_\star)\leq\delta(\varepsilon)$, then the previous bound, the monotonicity of $V_\star[\rho_k]$, and the upper bound on $V_\star$ give
\begin{equation}
    \kappa_1\left(\mathbb{D}(\rho_k\|\rho_\star)\right)
    \leq
    V_\star[\rho_k]
    \leq
    V_\star[\rho_0]
    \leq
    \kappa_2\left(\mathbb{D}(\rho_0\|\rho_\star)\right)
    \leq
    \kappa_2(\delta(\varepsilon))
    \leq
    \kappa_1(\varepsilon),
    \quad \forall k\geq 0.
\end{equation}
Since $\kappa_1$ is strictly increasing, $\mathbb{D}(\rho_k\|\rho_\star)\leq\varepsilon$ for all $k\geq 0$; that is, the closed-loop DMDP is $\mathbb{D}$-stable with respect to $\rho_\star$.

For $\mathbb{D}$-attractivity, suppose, for contradiction, that $\bar{V}>0$.
Since $V_\star[\rho_k]\rightarrow\bar{V}$, there exists $K\geq 0$ such that, for all $k\geq K$, $V_\star[\rho_k]\geq\bar{V}/2$ and, hence, by the upper bound on $V_\star$, $\kappa_2(\mathbb{D}(\rho_k\|\rho_\star))\geq\bar{V}/2$.
Since $\kappa_2$ is a class-$\mathcal{K}$ function, there exists a constant $\varepsilon_0>0$ such that, for all $k\geq K$, $\mathbb{D}(\rho_k\|\rho_\star)\geq\varepsilon_0$ and, hence, $\kappa_1(\mathbb{D}(\rho_k\|\rho_\star))\geq\kappa_1(\varepsilon_0)>0$.
Summing~\eqref{eq:proof_lyapunov_decrease} gives, for all $k>K$,
\begin{equation}
    V_\star[\rho_k]
    \leq
    V_\star[\rho_K]
    -
    \sum_{j=K}^{k-1}
    \kappa_1\left(\mathbb{D}(\rho_j\|\rho_\star)\right)
    \leq
    V_\star[\rho_K]
    -
    (k-K)\kappa_1(\varepsilon_0).
\end{equation}
The right-hand side tends to $-\infty$ as $k\rightarrow\infty$, which contradicts $V_\star[\rho_k]\geq 0$.
Therefore, $\bar{V}=0$, and $\kappa_1(\mathbb{D}(\rho_k\|\rho_\star))\leq V_\star[\rho_k]$ yields $\lim_{k\rightarrow\infty}\kappa_1(\mathbb{D}(\rho_k\|\rho_\star))=0$.
Since $\kappa_1$ is a class-$\mathcal{K}$ function, $\lim_{k\rightarrow\infty}\mathbb{D}(\rho_k\|\rho_\star)=0$; that is, $\rho_\star$ is $\mathbb{D}$-attractive.
Together with $\mathbb{D}$-stability, this proves that the DMDP under $\psi_\star$ is $\mathbb{D}$-asymptotically stable with respect to $\rho_\star$ and $\mathbb{D}$.
\end{proof}

\subsection{Proof of Lemma~\ref*{lemma:lyapunov_distribution_stability}}
\label{app:proof_lyapunov_distribution_stability}
\restateLemLyapunovTransfer*

\begin{proof}
Let $\omega_0\in\Omega_{\mathsf{c}}$ and $\omega_{k+1}=F_\phi(\omega_k)$ for $k\geq 0$.
Since $\Omega_{\mathsf{c}}$ is positively invariant under $F_\phi$, $\omega_k\in\Omega_{\mathsf{c}}$ for all $k\geq 0$.
With $r_k:=\|\omega_k-\omega_\star\|$ and $W_k:=W(\omega_k)$, the decrease condition~\eqref{eq:generic_lyapunov_decrease_parameter} gives
\begin{equation}
    W_{k+1}-W_k
    \leq
    -\sigma(r_k),
    \quad
    \forall k\geq 0.
    \label{eq:proof_generic_lyapunov_decrease}
\end{equation}
Since $\sigma$ is nonnegative, $\{W_k\}_{k=0}^{\infty}$ is nonincreasing, and, by~\eqref{eq:generic_lyapunov_bounds_parameter}, $W_k\geq\underline{\alpha}(r_k)\geq 0$ for all $k\geq 0$.

First, the parameter dynamics are stable with respect to $\omega_\star$.
Let $\varepsilon_\omega>0$.
Since $\underline{\alpha}$ and $\overline{\alpha}$ are class-$\mathcal K$ functions, $\underline{\alpha}(\varepsilon_\omega)>0$ and, by the continuity of $\overline{\alpha}$ at the origin, there exists $\delta_\omega>0$ such that $\overline{\alpha}(\delta_\omega)\leq\underline{\alpha}(\varepsilon_\omega)$.
If $r_0\leq\delta_\omega$, then~\eqref{eq:generic_lyapunov_bounds_parameter} and the monotonicity of $W_k$ give
\begin{equation}
    \underline{\alpha}(r_k)
    \leq
    W_k
    \leq
    W_0
    \leq
    \overline{\alpha}(r_0)
    \leq
    \overline{\alpha}(\delta_\omega)
    \leq
    \underline{\alpha}(\varepsilon_\omega),
    \quad
    \forall k\geq 0,
\end{equation}
and, since $\underline{\alpha}$ is strictly increasing, $r_k\leq\varepsilon_\omega$ for all $k\geq 0$.

Second, the parameter dynamics are attractive.
Summing~\eqref{eq:proof_generic_lyapunov_decrease} from $k=0$ to $N-1$ gives $\sum_{k=0}^{N-1}\sigma(r_k)\leq W_0-W_N\leq W_0$.
Letting $N\rightarrow\infty$ yields $\sum_{k=0}^{\infty}\sigma(r_k)<\infty$, so $\lim_{k\rightarrow\infty}\sigma(r_k)=0$.
Since $\sigma$ is a class-$\mathcal K$ function, $\lim_{k\rightarrow\infty}r_k=0$; hence, the closed-loop parameter dynamics are asymptotically stable with respect to $\omega_\star$.

It remains to transfer these properties to the distributions $\rho_{\omega_k}$.
Since the parameterization is $\mathbb{D}$-compatible on $\Omega_{\mathsf{c}}$ in the sense of \namedref{Assumption}{assump:D_compatible_parameterization} and $\rho_{\omega_\star}=\rho_\star$,
\begin{equation}
    \alpha_1(r_k)
    \leq
    \mathbb{D}(\rho_{\omega_k}\|\rho_\star)
    \leq
    \alpha_2(r_k),
    \quad
    \forall k\geq 0.
\end{equation}
Since $r_k\rightarrow 0$ and $\alpha_2$ is a class-$\mathcal K$ function, $\lim_{k\rightarrow\infty}\mathbb{D}(\rho_{\omega_k}\|\rho_\star)=0$, which is $\mathbb{D}$-attractivity.
For $\mathbb{D}$-stability, let $\varepsilon_{\mathbb{D}}>0$ and choose $\varepsilon_\omega>0$ such that $\alpha_2(\varepsilon_\omega)\leq\varepsilon_{\mathbb{D}}$.
By the stability of the parameter dynamics, there exists $\delta_\omega>0$ such that $r_0\leq\delta_\omega$ implies $r_k\leq\varepsilon_\omega$ for all $k\geq 0$.
Let $\delta_{\mathbb{D}}:=\alpha_1(\delta_\omega)>0$.
If $\mathbb{D}(\rho_{\omega_0}\|\rho_\star)\leq\delta_{\mathbb{D}}$, then the lower compatibility bound gives $\alpha_1(r_0)\leq\mathbb{D}(\rho_{\omega_0}\|\rho_\star)\leq\alpha_1(\delta_\omega)$, so $r_0\leq\delta_\omega$ because $\alpha_1$ is strictly increasing.
Hence,
\begin{equation}
    \mathbb{D}(\rho_{\omega_k}\|\rho_\star)
    \leq
    \alpha_2(r_k)
    \leq
    \alpha_2(\varepsilon_\omega)
    \leq
    \varepsilon_{\mathbb{D}},
    \quad
    \forall k\geq 0,
\end{equation}
which is $\mathbb{D}$-stability.
Together, $\mathbb{D}$-stability and $\mathbb{D}$-attractivity prove that the induced distribution dynamics are $\mathbb{D}$-asymptotically stable with respect to $\rho_\star=\rho_{\omega_\star}$ on $\{\rho_\omega\mid\omega\in\Omega_{\mathsf{c}}\}$.
\end{proof}

\subsection{Proof of Theorem~\ref*{theo:parameter_stability_means_distribution_stability}}
\label{app:proof_parameter_stability_means_distribution_stability}
\restateThmParameterStability*

\begin{proof}
\namedref{Lemma}{lemma:lyapunov_distribution_stability} is applied with $\phi=\phi_\star$, $W:=U_\star^{\mathsf{R}}$, and $\Omega_{\mathsf{c}}:=\Omega_{\mathcal{P}}$.
Its invariance, fixed-point, and $\mathbb{D}$-compatibility conditions hold by assumption, and $W$ is a continuous map into $\mathbb{R}_{\geq0}$ because $U_\star^{\mathsf{R}}$ is finite, nonnegative, and continuous on $\Omega_{\mathcal{P}}$.
It remains to verify~\eqref{eq:generic_lyapunov_bounds_parameter} and~\eqref{eq:generic_lyapunov_decrease_parameter}.
Let $\omega\in\Omega_{\mathcal{P}}$.
Since $F_{\phi_\star}(\omega)\in\Omega_{\mathcal{P}}$ by positive invariance and $U_\star^{\mathsf{R}}\geq0$ on $\Omega_{\mathcal{P}}$, the rotated Bellman identity~\eqref{eq:rotated_bellman_identity_parameter}, the rotated dissipativity condition~\eqref{eq:rotated_dissipativity_parameter}, and the upper bound~\eqref{eq:rotated_value_upper_bound_parameter} give
\begin{align}
    U_\star^{\mathsf{R}}(\omega)
    &=\ell^{\mathsf{R}}\left(\omega,\phi_\star(\omega)\right)+U_\star^{\mathsf{R}}\left(F_{\phi_\star}(\omega)\right)
    \geq\ell^{\mathsf{R}}\left(\omega,\phi_\star(\omega)\right)
    \geq\kappa_1\left(\|\omega-\omega_\star\|\right),
    \label{eq:proof_optimal_value_lower_bound}\\
    U_\star^{\mathsf{R}}(\omega)
    &\leq\kappa_2\left(\|\omega-\omega_\star\|\right).
    \label{eq:proof_optimal_value_upper_bound}
\end{align}
The rotated Bellman identity and the rotated dissipativity condition~\eqref{eq:rotated_dissipativity_parameter} also give the decrease
\begin{equation}
    U_\star^{\mathsf{R}}\left(F_{\phi_\star}(\omega)\right)-U_\star^{\mathsf{R}}(\omega)
    =-\ell^{\mathsf{R}}\left(\omega,\phi_\star(\omega)\right)
    \leq-\kappa_1\left(\|\omega-\omega_\star\|\right).
\end{equation}
Hence, the Lyapunov bounds~\eqref{eq:generic_lyapunov_bounds_parameter} hold with $\underline{\alpha}=\kappa_1$ and $\overline{\alpha}=\kappa_2$, and the decrease condition~\eqref{eq:generic_lyapunov_decrease_parameter} holds with $\sigma=\kappa_1$.
By \namedref{Lemma}{lemma:lyapunov_distribution_stability}, the induced distribution dynamics generated by $\omega_{k+1}=F_{\phi_\star}(\omega_k)$ and $\rho_k=\rho_{\omega_k}$ are $\mathbb{D}$-asymptotically stable with respect to $\rho_\star=\rho_{\omega_\star}$ on $\{\rho_\omega\mid\omega\in\Omega_{\mathcal{P}}\}$.
\end{proof}

\subsection{Proof of Theorem~\ref*{theo:stability_finite_samples}}
\label{app:proof_stability_finite_samples}
\restateThmFiniteSample*

\begin{proof}
Recall the learned and true closed-loop maps $\widehat{F}_{\mathsf{off}}(\omega)=\hat{f}_{\mathsf{off}}(\omega,\hat{\phi}_{\mathsf{off}}(\omega))$ and $F_{\mathsf{off}}(\omega)=f(\omega,\hat{\phi}_{\mathsf{off}}(\omega))$.
All statements are proved on the event $\mathcal E_{\mathsf{off}}$, of probability at least $1-\alpha$, on which $\Delta\leq\bar{\Delta}$ and the learned decrease margin~\eqref{eq:learned_lyapunov_decrease_margin_finite_sample} holds.
Let $\omega\in\Omega_{\mathsf{com}}$.
A telescoping decomposition gives
\begin{equation}
\begin{aligned}
    &U_\star^{\mathsf{R}}\left(F_{\mathsf{off}}(\omega)\right)-U_\star^{\mathsf{R}}(\omega)\\
    &=\left[U_\star^{\mathsf{R}}\left(F_{\mathsf{off}}(\omega)\right)-U_\star^{\mathsf{R}}\left(\widehat{F}_{\mathsf{off}}(\omega)\right)\right]
    +\left[U_\star^{\mathsf{R}}\left(\widehat{F}_{\mathsf{off}}(\omega)\right)-\widehat{U}_{\mathsf{off}}^{\mathsf{R}}\left(\widehat{F}_{\mathsf{off}}(\omega)\right)\right]\\
    &\quad+\left[\widehat{U}_{\mathsf{off}}^{\mathsf{R}}\left(\widehat{F}_{\mathsf{off}}(\omega)\right)-\widehat{U}_{\mathsf{off}}^{\mathsf{R}}(\omega)\right]
    +\left[\widehat{U}_{\mathsf{off}}^{\mathsf{R}}(\omega)-U_\star^{\mathsf{R}}(\omega)\right]\\
    &\leq\Delta_f+\Delta_U-\sigma\left(\|\omega-\omega_\star\|\right)+\Delta_U
    =\Delta-\sigma\left(\|\omega-\omega_\star\|\right)
    \leq\bar{\Delta}-\sigma\left(\|\omega-\omega_\star\|\right).
\end{aligned}
\end{equation}
The four brackets are bounded by~\eqref{eq:def_delta_f_finite_sample}, \eqref{eq:def_delta_U_finite_sample}, \eqref{eq:learned_lyapunov_decrease_margin_finite_sample}, and~\eqref{eq:def_delta_U_finite_sample}, respectively; the second bound applies because $\widehat{F}_{\mathsf{off}}(\omega)\in\Omega_{\mathsf{com}}$ by the positive invariance of $\Omega_{\mathsf{com}}$ under $\widehat{F}_{\mathsf{off}}$.
This proves~\eqref{eq:true_lyapunov_decrease_with_residual}.

If $\omega\in\Omega_{\mathsf{com}}$ and $\omega\notin\Omega_{\bar{\Delta}}$, with $\Omega_{\bar{\Delta}}$ the residual set~\eqref{eq:finite_sample_residual_set}, then $\sigma(\|\omega-\omega_\star\|)>\bar{\Delta}$, and~\eqref{eq:true_lyapunov_decrease_with_residual} gives $U_\star^{\mathsf{R}}(F_{\mathsf{off}}(\omega))-U_\star^{\mathsf{R}}(\omega)<0$.
Hence, $U_\star^{\mathsf{R}}$ strictly decreases along the true closed loop outside $\Omega_{\bar{\Delta}}$, which certifies the Lyapunov decrease of the learned policy there.
Since $\sigma$ is a class-$\mathcal{K}$ function, $\Omega_{\bar{\Delta}}$ shrinks to $\{\omega_\star\}$ as $\bar{\Delta}\rightarrow0$.

If $\bar{\Delta}=0$, then~\eqref{eq:true_lyapunov_decrease_with_residual} is the decrease condition~\eqref{eq:generic_lyapunov_decrease_parameter} with $W=U_\star^{\mathsf{R}}$, $\Omega_{\mathsf{c}}=\Omega_{\mathsf{com}}$, and $\phi=\hat{\phi}_{\mathsf{off}}$, for which $F_\phi=F_{\mathsf{off}}$.
The remaining conditions of \namedref{Lemma}{lemma:lyapunov_distribution_stability} are the continuity of $U_\star^{\mathsf{R}}$, the Lyapunov bounds~\eqref{eq:finite_sample_value_bounds}, the positive invariance of $\Omega_{\mathsf{com}}$ under $F_{\mathsf{off}}$, the fixed-point condition $F_{\mathsf{off}}(\omega_\star)=\omega_\star$, and the $\mathbb{D}$-compatibility of the parameterization on $\Omega_{\mathsf{com}}$ in the sense of \namedref{Assumption}{assump:D_compatible_parameterization}, all of which hold by assumption.
Hence, the induced distribution dynamics under $\hat{\phi}_{\mathsf{off}}$ are $\mathbb{D}$-asymptotically stable with respect to $\rho_\star$ on $\{\rho_\omega\mid\omega\in\Omega_{\mathsf{com}}\}$.
The same argument, with $\sigma$ replaced by $\tilde{\sigma}$, applies when the approximation error is state-dependent and dominated by the decrease margin, so that $U_\star^{\mathsf{R}}(F_{\mathsf{off}}(\omega))-U_\star^{\mathsf{R}}(\omega)\leq-\tilde{\sigma}(\|\omega-\omega_\star\|)$ for all $\omega\in\Omega_{\mathsf{com}}$ and some class-$\mathcal{K}$ function $\tilde{\sigma}$.
\end{proof}

\subsection{Proof of Theorem~\ref*{theo:online_plugin_stability}}
\label{app:proof_online_plugin_stability}
\restateThmOnlinePlugin*

\begin{proof}
Throughout, the event in~\eqref{eq:online_estimator_uniform_error}, which has probability at least $1-\alpha_{\mathsf{on}}$, is assumed to hold, so that $\|\hat{\omega}_t-\omega_t\|\leq\delta_{\mathsf{on}}$ and $\hat{\omega}_t\in\Omega_{\mathsf{com}}$ for all $t\geq0$.
Hence, the assumed invariance gives $\omega_{t+1}=F_{\mathsf{on}}(\omega_t,\hat{\omega}_t)\in\Omega_{\mathsf{com}}$ whenever $\omega_t\in\Omega_{\mathsf{com}}$, and, since $\omega_0\in\Omega_{\mathsf{com}}$, induction on $t$ gives $\omega_t\in\Omega_{\mathsf{com}}$ for all $t\geq0$.

Let $\omega,\tilde{\omega}\in\Omega_{\mathsf{com}}$ with $\|\tilde{\omega}-\omega\|\leq\delta_{\mathsf{on}}$.
Adding and subtracting $U_\star^{\mathsf{R}}(F_{\mathsf{off}}(\omega))$, and bounding the first bracket by the definition~\eqref{eq:online_plugin_perturbation_bound} of $\Delta_{\mathsf{on}}$ and the second by the exact-parameter certificate~\eqref{eq:exact_parameter_online_certificate}, gives
\begin{equation}
\begin{aligned}
    &U_\star^{\mathsf{R}}\left(F_{\mathsf{on}}(\omega,\tilde{\omega})\right)-U_\star^{\mathsf{R}}(\omega) \\
    &=
    \left[U_\star^{\mathsf{R}}\left(F_{\mathsf{on}}(\omega,\tilde{\omega})\right)-U_\star^{\mathsf{R}}\left(F_{\mathsf{off}}(\omega)\right)\right]
    +
    \left[U_\star^{\mathsf{R}}\left(F_{\mathsf{off}}(\omega)\right)-U_\star^{\mathsf{R}}(\omega)\right] \\
    &\leq
    -\sigma_{\mathsf{on}}\left(\|\omega-\omega_\star\|\right)
    +\bar{\Delta}_{\mathsf{off}}
    +\Delta_{\mathsf{on}}(\delta_{\mathsf{on}}).
\end{aligned}
\end{equation}
Setting $\omega=\omega_t$ and $\tilde{\omega}=\hat{\omega}_t$ yields~\eqref{eq:online_true_lyapunov_decrease_residual}.
If $\omega_t\notin\Omega_{\mathsf{on}}$, then $\sigma_{\mathsf{on}}(\|\omega_t-\omega_\star\|)>\bar{\Delta}_{\mathsf{off}}+\Delta_{\mathsf{on}}(\delta_{\mathsf{on}})$ by~\eqref{eq:online_residual_set}, so~\eqref{eq:online_true_lyapunov_decrease_residual} gives $U_\star^{\mathsf{R}}(\omega_{t+1})<U_\star^{\mathsf{R}}(\omega_t)$.
Hence, $U_\star^{\mathsf{R}}$ strictly decreases along the online closed loop outside $\Omega_{\mathsf{on}}$, and, since $\sigma_{\mathsf{on}}$ is a class-$\mathcal{K}$ function, $\Omega_{\mathsf{on}}$ shrinks to $\{\omega_\star\}$ as $\bar{\Delta}_{\mathsf{off}}+\Delta_{\mathsf{on}}(\delta_{\mathsf{on}})\rightarrow0$.

Now suppose that $U_\star^{\mathsf{R}}(\omega_{t+1})-U_\star^{\mathsf{R}}(\omega_t)\leq-\tilde{\sigma}_{\mathsf{on}}(\|\omega_t-\omega_\star\|)$ for all $t\geq0$ and some class-$\mathcal{K}$ function $\tilde{\sigma}_{\mathsf{on}}$; by~\eqref{eq:online_true_lyapunov_decrease_residual}, this holds with $\tilde{\sigma}_{\mathsf{on}}=\sigma_{\mathsf{on}}$ if the total residual is zero.
Since the Lyapunov argument in the proof of \namedref{Lemma}{lemma:lyapunov_distribution_stability} uses only the Lyapunov bounds and the decrease inequality along the trajectory, it applies with $W=U_\star^{\mathsf{R}}$ to every realized trajectory on the event, using this decrease, the bounds~\eqref{eq:finite_sample_value_bounds}, the invariance condition, and the $\mathbb{D}$-compatibility in \namedref{Assumption}{assump:D_compatible_parameterization}.
Hence, the induced distribution dynamics are $\mathbb{D}$-asymptotically stable with respect to $\rho_\star$ on $\{\rho_\omega\mid\omega\in\Omega_{\mathsf{com}}\}$.
\end{proof}

\subsection{Proof of Theorem~\ref*{theo:model_free_online_policy_learning}}
\label{app:proof_model_free_online_policy_learning}
\restateThmModelFree*

\begin{proof}
Throughout, the event in~\eqref{eq:model_free_high_probability_event} is assumed to hold, so that $\Delta_{\mathsf{dec},n}\leq\bar{\Delta}_{\mathsf{dec},n}$, $\|\hat{\omega}_t-\omega_t\|\leq\delta_{\omega,n}$, and $\hat{\omega}_t\in\Omega_{\mathsf{com}}$ for all $t\geq0$.
Hence, the invariance condition~\eqref{eq:model_free_plugin_invariance} gives $\omega_{t+1}=F_n^{\mathsf{plug}}(\omega_t,\hat{\omega}_t)\in\Omega_{\mathsf{com}}$ whenever $\omega_t\in\Omega_{\mathsf{com}}$, and, since $\omega_0\in\Omega_{\mathsf{com}}$, induction on $t$ gives $\omega_t\in\Omega_{\mathsf{com}}$ for all $t\geq0$.

For $\omega\in\Omega_{\mathsf{com}}$, the definition~\eqref{eq:model_free_decrease_estimation_error} of $\Delta_{\mathsf{dec},n}$ and the empirical decrease certificate~\eqref{eq:model_free_empirical_decrease_certificate} give
\begin{equation}
    W_n(F_n(\omega))-W_n(\omega)
    \leq
    \widehat{\mathcal D}_n(\omega)+\Delta_{\mathsf{dec},n}
    \leq
    -\sigma_n\left(\|\omega-\omega_\star\|\right)+\Delta_{\mathsf{dec},n}.
    \label{eq:proof_exact_model_free_decrease}
\end{equation}
Let $\tilde{\omega}\in\Omega_{\mathsf{com}}$ with $\|\tilde{\omega}-\omega\|\leq\delta_{\omega,n}$.
Adding and subtracting $W_n(F_n(\omega))$, and bounding the first bracket by the definition~\eqref{eq:model_free_parameter_estimation_perturbation} of $\Delta_{\omega,n}$ and the second by~\eqref{eq:proof_exact_model_free_decrease} with $\Delta_{\mathsf{dec},n}\leq\bar{\Delta}_{\mathsf{dec},n}$, gives
\begin{equation}
\begin{aligned}
    &W_n\left(F_n^{\mathsf{plug}}(\omega,\tilde{\omega})\right)-W_n(\omega) \\
    &=
    \left[W_n\left(F_n^{\mathsf{plug}}(\omega,\tilde{\omega})\right)-W_n(F_n(\omega))\right]
    +
    \left[W_n(F_n(\omega))-W_n(\omega)\right] \\
    &\leq
    -\sigma_n\left(\|\omega-\omega_\star\|\right)
    +\bar{\Delta}_{\mathsf{dec},n}
    +\Delta_{\omega,n}(\delta_{\omega,n}).
\end{aligned}
\end{equation}
Setting $\omega=\omega_t$ and $\tilde{\omega}=\hat{\omega}_t$ yields~\eqref{eq:model_free_true_decrease_with_residual}.
If $\omega_t\notin\Omega_{\mathsf{mf},n}$, then $\sigma_n(\|\omega_t-\omega_\star\|)>\bar{\Delta}_{\mathsf{dec},n}+\Delta_{\omega,n}(\delta_{\omega,n})$ by~\eqref{eq:model_free_residual_set}, so~\eqref{eq:model_free_true_decrease_with_residual} gives $W_n(\omega_{t+1})<W_n(\omega_t)$.
Hence, $W_n$ strictly decreases along the plug-in online dynamics outside $\Omega_{\mathsf{mf},n}$.

Now suppose that $W_n(\omega_{t+1})-W_n(\omega_t)\leq-\tilde{\sigma}_n(\|\omega_t-\omega_\star\|)$ for all $t\geq0$ and some class-$\mathcal K$ function $\tilde{\sigma}_n$; by~\eqref{eq:model_free_true_decrease_with_residual}, this holds with $\tilde{\sigma}_n=\sigma_n$ if the total residual is zero.
As in the proof of \namedref{Theorem}{theo:online_plugin_stability}, the Lyapunov argument in the proof of \namedref{Lemma}{lemma:lyapunov_distribution_stability} then applies with $W=W_n$ to every realized trajectory on the event, using this decrease, the bounds~\eqref{eq:model_free_Wn_bounds}, the invariance condition, and the $\mathbb D$-compatibility in \namedref{Assumption}{assump:D_compatible_parameterization}.
Hence, the induced distribution dynamics are $\mathbb D$-asymptotically stable with respect to $\rho_\star$ on $\{\rho_\omega\mid\omega\in\Omega_{\mathsf{com}}\}$.
\end{proof}

\subsection{Proof of Proposition~\ref*{prop:asymptotic_recovery_rotated_optimizers}}
\label{app:proof_asymptotic_recovery_rotated_optimizers}
\restatePropAsymptoticRecovery*

\begin{proof}
For $\theta\in\Theta$, write $\phi_\theta:=\phi(\cdot;\theta)$, and define the true and learned closed-loop trajectories with $\omega_0^\theta=\widehat{\omega}_0^{\theta,N}=\omega$ by
\begin{equation}
\begin{aligned}
    \omega_{k+1}^{\theta}
    &=
    f\left(\omega_k^{\theta},\phi_\theta(\omega_k^{\theta})\right), \\
    \widehat{\omega}_{k+1}^{\theta,N}
    &=
    \hat{f}_{\mathsf{off}}^{N}\left(\widehat{\omega}_k^{\theta,N},\phi_\theta(\widehat{\omega}_k^{\theta,N})\right).
\end{aligned}
\end{equation}
Let $\ell^{\mathsf{R}}$ be the rotated stage cost~\eqref{eq:rotated_stage_cost_parameter}, and let $\widehat{\ell}_{\mathsf{off}}^{\mathsf{R},N}$ be obtained from $\ell^{\mathsf{R}}$ by replacing $f$ with $\hat{f}_{\mathsf{off}}^{N}$.
For a finite horizon $H$, the true and learned rotated objectives are
\begin{equation}
\begin{aligned}
    J_H^{\mathsf{R}}(\theta;\omega)
    &:=
    \lambda(\omega_H^\theta)
    +
    \sum_{k=0}^{H-1}
    \ell^{\mathsf{R}}\left(\omega_k^\theta,\phi_\theta(\omega_k^\theta)\right), \\
    \widehat{J}_{H,N}^{\mathsf{R}}(\theta;\omega)
    &:=
    \lambda(\widehat{\omega}_H^{\theta,N})
    +
    \sum_{k=0}^{H-1}
    \widehat{\ell}_{\mathsf{off}}^{\mathsf{R},N}\left(\widehat{\omega}_k^{\theta,N},\phi_\theta(\widehat{\omega}_k^{\theta,N})\right),
\end{aligned}
\end{equation}
and $J^{\mathsf{R}}(\theta;\omega)$ and $\widehat{J}_{N}^{\mathsf{R}}(\theta;\omega)$ are their limits as $H\rightarrow\infty$, which are well defined by assumption.

Fix an event of probability one on which $N_0<\infty$ and~\eqref{eq:uniform_model_consistency_main} holds.
Since $\Omega_{\mathsf{com}}$, $\mathcal{A}$, and $\Theta$ are compact, the continuous maps $f$, $\ell$, $\lambda$, and $\phi$ are uniformly continuous on the relevant compact sets.
By the assumed positive invariance, $\omega_k^\theta\in\Omega_{\mathsf{com}}$ for all $\theta\in\Theta$ and $k\geq0$, and so is $\widehat{\omega}_k^{\theta,N}$ once $N\geq N_0$.

For every fixed $H<\infty$,
\begin{equation}
    \label{eq:proof_uniform_trajectory_convergence_prop}
    \sup_{\theta\in\Theta}
    \max_{0\leq k\leq H}
    \left\|\widehat{\omega}_k^{\theta,N}-\omega_k^\theta\right\|
    \rightarrow 0
    \quad
    \text{almost surely}.
\end{equation}
This follows by induction on $k$.
The case $k=0$ holds because $\widehat{\omega}_0^{\theta,N}=\omega_0^\theta$.
If the convergence holds at time $k<H$, the triangle inequality gives
\begin{equation}
\begin{aligned}
    \left\|\widehat{\omega}_{k+1}^{\theta,N}-\omega_{k+1}^{\theta}\right\|
    &\leq
    \left\|\left(\hat{f}_{\mathsf{off}}^{N}-f\right)\left(\widehat{\omega}_k^{\theta,N},\phi_\theta(\widehat{\omega}_k^{\theta,N})\right)\right\| \\
    &\quad+
    \left\|f\left(\widehat{\omega}_k^{\theta,N},\phi_\theta(\widehat{\omega}_k^{\theta,N})\right)-f\left(\omega_k^\theta,\phi_\theta(\omega_k^\theta)\right)\right\|.
\end{aligned}
\end{equation}
The first term tends to zero uniformly in $\theta$ by~\eqref{eq:uniform_model_consistency_main}, and the second by the induction hypothesis and the uniform continuity of $f$ and $\phi$.
By~\eqref{eq:proof_uniform_trajectory_convergence_prop}, \eqref{eq:uniform_model_consistency_main}, and the uniform continuity of $\ell$ and $\lambda$, it follows that, for every fixed $H<\infty$,
\begin{equation}
    \label{eq:proof_fixed_horizon_objective_convergence_prop}
    \sup_{\theta\in\Theta}
    \left|\widehat{J}_{H,N}^{\mathsf{R}}(\theta;\omega)-J_H^{\mathsf{R}}(\theta;\omega)\right|
    \rightarrow 0
    \quad
    \text{almost surely}.
\end{equation}
As assumed in \namedref{Proposition}{prop:asymptotic_recovery_rotated_optimizers}, this convergence extends to the infinite-horizon objectives:
\begin{equation}
    \label{eq:proof_infinite_horizon_objective_convergence_prop}
    \zeta_N
    :=
    \sup_{\theta\in\Theta}
    \left|\widehat{J}_{N}^{\mathsf{R}}(\theta;\omega)-J^{\mathsf{R}}(\theta;\omega)\right|
    \rightarrow 0
    \quad
    \text{almost surely}.
\end{equation}

Since $\Theta$ is compact and both objectives are continuous, $\Theta_\star^{\mathsf{R}}$ is nonempty and $\widehat{\theta}_{N}$ exists.
Let $J_\star^{\mathsf{R}}:=\min_{\theta\in\Theta}J^{\mathsf{R}}(\theta;\omega)$.
For any $\theta_\star\in\Theta_\star^{\mathsf{R}}$, the optimality of $\widehat{\theta}_N$ for the learned objective gives
\begin{equation}
\begin{aligned}
    J_\star^{\mathsf{R}}
    \leq
    J^{\mathsf{R}}(\widehat{\theta}_N;\omega)
    &\leq
    \widehat{J}_{N}^{\mathsf{R}}(\widehat{\theta}_N;\omega)+\zeta_N
    \leq
    \widehat{J}_{N}^{\mathsf{R}}(\theta_\star;\omega)+\zeta_N \\
    &\leq
    J^{\mathsf{R}}(\theta_\star;\omega)+2\zeta_N
    =
    J_\star^{\mathsf{R}}+2\zeta_N.
\end{aligned}
\end{equation}
Hence, $J^{\mathsf{R}}(\widehat{\theta}_N;\omega)\rightarrow J_\star^{\mathsf{R}}$ almost surely.

Suppose, for contradiction, that $\operatorname{dist}(\widehat{\theta}_N,\Theta_\star^{\mathsf{R}})\not\rightarrow0$.
Then, there exist $\varepsilon>0$ and a subsequence $\{\widehat{\theta}_{N_j}\}_{j=1}^{\infty}$ with $\operatorname{dist}(\widehat{\theta}_{N_j},\Theta_\star^{\mathsf{R}})\geq\varepsilon$ for all $j$.
By the compactness of $\Theta$, a further subsequence, still denoted by $\widehat{\theta}_{N_j}$, converges to some $\bar{\theta}\in\Theta$.
Since $J^{\mathsf{R}}(\widehat{\theta}_{N_j};\omega)\rightarrow J_\star^{\mathsf{R}}$ and $J^{\mathsf{R}}(\cdot;\omega)$ is continuous, $J^{\mathsf{R}}(\bar{\theta};\omega)=J_\star^{\mathsf{R}}$, so $\bar{\theta}\in\Theta_\star^{\mathsf{R}}$.
However, the continuity of $\operatorname{dist}(\cdot,\Theta_\star^{\mathsf{R}})$ gives $\operatorname{dist}(\bar{\theta},\Theta_\star^{\mathsf{R}})\geq\varepsilon$, a contradiction.
Therefore, $\operatorname{dist}(\widehat{\theta}_N,\Theta_\star^{\mathsf{R}})\rightarrow0$ almost surely; equivalently, every limit point of $\{\widehat{\theta}_{N}\}_{N=1}^{\infty}$ belongs to $\Theta_\star^{\mathsf{R}}$ almost surely.

Finally, suppose that every policy $\phi(\cdot;\bar{\theta})$ with $\bar{\theta}\in\Theta_\star^{\mathsf{R}}$ satisfies the Lyapunov certificate in \namedref{Lemma}{lemma:lyapunov_distribution_stability} on $\Omega_{\mathsf{com}}$ and that the parameterization is $\mathbb{D}$-compatible on $\Omega_{\mathsf{com}}$ in the sense of \namedref{Assumption}{assump:D_compatible_parameterization}.
Then, applying \namedref{Lemma}{lemma:lyapunov_distribution_stability} to each such limit policy shows that its induced distribution dynamics are $\mathbb{D}$-asymptotically stable with respect to $\rho_\star$ on $\{\rho_\omega\mid\omega\in\Omega_{\mathsf{com}}\}$.
\end{proof}

\section{Technical Comparison with Related Approaches}
\label{appendix:comparison_state_feedback}

This appendix details the comparison in \namedref{Table}{tab:related_work_comparison}.
A stage cost, objective, or constraint counts as a distribution cost only if it depends nonlinearly on the distribution of the state or of state-action pairs, not merely through the expectation of a per-state cost or through the return distribution.
A policy uses distribution feedback only if it is evaluated at the current distribution, or at a finite-dimensional parameter of it, rather than at a single state or through a time-indexed law computed offline for a fixed initial distribution.
The assumed ergodicity of a fixed Markov chain and the convergence of learning iterates, such as those of distributional reinforcement learning algorithms~\citep{Rowland2018,Antonetti2026}, do not establish distributional stability.
Optimality, regret, and sample-complexity guarantees do not certify learned policies.

\subsection{Norm-Based MDPs}
\label{appendix:norm_mdp}

The norm-based MDP (NMDP) is the tuple $\mathcal{M}_{\mathsf{N}}:=\langle \mathcal{S},\mathcal{A},\tau,L,\gamma, \rho_0 \rangle$, where $\gamma\in(0,1]$ is a discount factor, $\tau$ is the Markov transition kernel, $\rho_0$ is the initial state distribution, and $L : \mathcal{S} \times \mathcal{A} \rightarrow \mathbb{R}$ is a norm-based stage cost, typically $L(\bm{\mathrm{s}}, \bm{\mathrm{a}}) = \|\bm{\mathrm{s}}-\bm{\mathrm{s}}_\star\|_2^2$, which penalizes the instantaneous deviation of the state from a target state $\bm{\mathrm{s}}_\star$.
For stability analysis, costs replace the rewards common in reinforcement learning; the two formulations are equivalent up to negation.
A deterministic causal state-feedback policy $\pi: \mathcal{S} \rightarrow \mathcal{A}$ sets $\bm{\mathrm{a}}_k = \pi(\bm{\mathrm{s}}_k)$, and $\Pi$ denotes the set of all such policies.
For fixed $\pi$, $\tau$, and $\rho_0$, the state distributions $\{\rho_k\}_{k=0}^\infty$, with $\bm{\mathrm{s}}_k \sim \rho_k$, are uniquely determined by the closed-loop distribution transition operator $\mathcal{T}_{\tau,\pi}:M(\mathcal{S})\rightarrow M(\mathcal{S})$,
\begin{equation}
    \mathcal{T}_{\tau,\pi}\rho(\mathcal{B}) := \int_{\mathcal{S}} \tau(\mathcal{B}\mid \bm{\mathrm{s}},\pi(\bm{\mathrm{s}}))\, \rho(\mathsf{d}\bm{\mathrm{s}}) \quad \text{for every measurable } \mathcal{B}\subseteq\mathcal{S}.
\end{equation}
A standard reinforcement learning objective~\citep{Sutton2018} is the expected discounted cumulative cost:
\begin{equation}
\label{eq:prob_norm_based_mdp}
\tag{$\mathsf{TRL}$}
     \min_{\pi \in \Pi} \ \sum_{k=0}^{\infty} \gamma^k \mathbb{E}_{\bm{\mathrm{s}}_k \sim \rho_k} \left[L(\bm{\mathrm{s}}_k, \pi(\bm{\mathrm{s}}_k))\right] \quad
    \text{s.t.} \quad  \rho_{k+1} = \mathcal{T}_{\tau,\pi}\rho_k.
\end{equation}
Let $\pi_{\star}$ be an optimal policy of \namedref{Problem}{eq:prob_norm_based_mdp}, and, for $\rho_{k+1}=\mathcal{T}_{\tau,\pi_{\star}}\rho_k$, let
\begin{equation}
    \ell_{\star,k}(\rho_0) := \mathbb{E}_{\bm{\mathrm{s}}_k\sim\rho_k}\left[L(\bm{\mathrm{s}}_k, \pi_\star(\bm{\mathrm{s}}_k))\right] = \int_{\mathcal{S}} L(\bm{\mathrm{s}},\pi_\star(\bm{\mathrm{s}}))\, \rho_k(\mathsf{d}\bm{\mathrm{s}})
\end{equation}
be the expected stage cost at time $k$.
If $\{\ell_{\star,k}(\rho_0)\}_{k=0}^\infty$ is lower bounded and nonincreasing along the closed-loop distributional dynamics, it converges to a finite limit $\ell_\infty$.
However, this scalar sequence depends on $\rho_k$ only through the integral of $L$, whereas distributional stability requires convergence of the entire measure $\rho_k$.
Hence, convergence of $\{\ell_{\star,k}(\rho_0)\}_{k=0}^\infty$ does not, in general, imply that $\rho_k$ converges to a probability measure, as \namedref{Appendix}{appendix:counterexample} shows.
Convergence of the distribution itself can instead be established through dissipativity: \citet{Gros_2022} extended the dissipativity theory of economic model predictive control (MPC) to undiscounted infinite-horizon MDPs and showed that, under a strict dissipativity condition with respect to a dissimilarity measure, the state distribution under the optimal policy converges, in that dissimilarity measure, to the optimal steady-state probability measure defined by the stage cost.

\subsubsection{A Counterexample: Convergent Cost without Distributional Convergence}
\label{appendix:counterexample}
Let $\mathcal{S}=\mathbb{R}^{d_{\mathsf{s}}}$, $\mathcal{A}=\{\bm{\mathrm{a}}_0\}$, and $L(\bm{\mathrm{s}},\bm{\mathrm{a}})=\|\bm{\mathrm{s}}-\bm{\mathrm{s}}_\star\|_2^2$.
Then, $\pi_\star(\bm{\mathrm{s}})=\bm{\mathrm{a}}_0$ is the only admissible, hence optimal, policy of \namedref{Problem}{eq:prob_norm_based_mdp}.
The transition kernel is deterministic, $\tau(\mathcal{B}\mid \bm{\mathrm{s}},\bm{\mathrm{a}}_0)=\mathbf{1}_{\mathcal{B}}(G(\bm{\mathrm{s}}))$ for every measurable set $\mathcal{B}\subseteq\mathcal{S}$, with a measurable map $G:\mathcal{S}\rightarrow\mathcal{S}$.
The map $G$ preserves the norm around $\bm{\mathrm{s}}_\star$, i.e., $\|G(\bm{\mathrm{s}})-\bm{\mathrm{s}}_\star\|_2=\|\bm{\mathrm{s}}-\bm{\mathrm{s}}_\star\|_2$ for all $\bm{\mathrm{s}}\in\mathcal{S}$, and has finite order $m\geq 2$, i.e., $G^m$ is the identity map.
The pushforward of a probability measure $\rho$ by $G$ is $G_{\#}\rho(\mathcal{B}):=\rho(G^{-1}(\mathcal{B}))$ for every measurable set $\mathcal{B}\subseteq\mathcal{S}$.
The initial distribution $\rho_0$ has finite second moment, $\int_{\mathcal{S}}\|\bm{\mathrm{s}}-\bm{\mathrm{s}}_\star\|_2^2\,\rho_0(\mathsf{d}\bm{\mathrm{s}})<\infty$, and its orbit $\rho_0,\ G_{\#}\rho_0,\ \ldots,\ G_{\#}^{m-1}\rho_0$ contains at least two distinct probability measures.
Such $G$ and $\rho_0$ exist: the point reflection $G(\bm{\mathrm{s}})=2\bm{\mathrm{s}}_\star-\bm{\mathrm{s}}$ preserves the norm around $\bm{\mathrm{s}}_\star$ and has order $m=2$, and every initial distribution with finite second moment that is not symmetric about $\bm{\mathrm{s}}_\star$, e.g., a Gaussian distribution whose mean differs from $\bm{\mathrm{s}}_\star$, has two distinct measures in its orbit.

The closed-loop operator reduces to the pushforward, $\mathcal{T}_{\tau,\pi_\star}\rho(\mathcal{B})=\int_{\mathcal{S}}\mathbf{1}_{\mathcal{B}}(G(\bm{\mathrm{s}}))\,\rho(\mathsf{d}\bm{\mathrm{s}})=\rho(G^{-1}(\mathcal{B}))=G_{\#}\rho(\mathcal{B})$, so $\rho_k=G_{\#}^{k}\rho_0$ for all $k\geq 0$.
By the change-of-variables formula for pushforward measures and the norm preservation of $G$,
\begin{equation}
    \ell_{\star,k}(\rho_0)
    = \int_{\mathcal{S}} \|\bm{\mathrm{s}}-\bm{\mathrm{s}}_\star\|_2^2\, \rho_k(\mathsf{d}\bm{\mathrm{s}})
    = \int_{\mathcal{S}} \|G^k(\bm{\mathrm{s}})-\bm{\mathrm{s}}_\star\|_2^2\, \rho_0(\mathsf{d}\bm{\mathrm{s}})
    = \int_{\mathcal{S}} \|\bm{\mathrm{s}}-\bm{\mathrm{s}}_\star\|_2^2\, \rho_0(\mathsf{d}\bm{\mathrm{s}})
    = \ell_{\star,0}(\rho_0)
\end{equation}
for all $k\geq 0$.
This constant is finite by the second-moment assumption.
It is also positive: norm preservation gives $\|G(\bm{\mathrm{s}}_\star)-\bm{\mathrm{s}}_\star\|_2=0$, so $G$ fixes $\bm{\mathrm{s}}_\star$ and the Dirac measure $\delta_{\bm{\mathrm{s}}_\star}$ has a constant orbit, whereas the orbit of $\rho_0$ contains two distinct measures; hence, $\rho_0\neq\delta_{\bm{\mathrm{s}}_\star}$.
Thus, $\{\ell_{\star,k}(\rho_0)\}_{k=0}^{\infty}$ is lower bounded, nonincreasing, and convergent.

Since $G^m$ is the identity map, $\rho_{k+m}=\rho_k$ for all $k\geq0$.
For each $j\in\{0,\ldots,m-1\}$, the subsequence $\{\rho_{j+im}\}_{i=0}^{\infty}$ is constant and equal to $G_{\#}^{j}\rho_0$, and at least two of these constants differ.
Since weak limits are unique, $\{\rho_k\}_{k=0}^{\infty}$ does not converge weakly to any probability measure.
Moreover, let $\mathbb{D}$ be any dissimilarity measure with $\mathbb{D}(\rho\|\tilde{\rho})=0$ only if $\rho=\tilde{\rho}$.
For any probability measure $\tilde{\rho}$, the values $\mathbb{D}(\rho_k\|\tilde{\rho})$ cycle through the finitely many values $\mathbb{D}(G_{\#}^{j}\rho_0\|\tilde{\rho})$, $j=0,\ldots,m-1$, each occurring for infinitely many $k$.
If $\mathbb{D}(\rho_k\|\tilde{\rho})\rightarrow0$ held, all these values would vanish, and the separation property would give $G_{\#}^{j}\rho_0=\tilde{\rho}$ for every $j$, contradicting the assumption that the orbit contains two distinct measures.
Hence, $\{\rho_k\}_{k=0}^{\infty}$ does not converge in the sense of $\mathbb{D}$ either.

\namedref{Figure}{fig:counterexample} illustrates the example for the point reflection in one dimension with schematic densities.
Panel~(a) shows an initial distribution $\rho_0$ that is not symmetric about $\bm{\mathrm{s}}_\star$, its pushforward $G_{\#}\rho_0$, and the quadratic cost $\|\bm{\mathrm{s}}-\bm{\mathrm{s}}_\star\|_2^2$ (dotted), which is symmetric about $\bm{\mathrm{s}}_\star$.
In panel~(b), the state distributions alternate, $\rho_{2i}=\rho_0$ and $\rho_{2i+1}=G_{\#}\rho_0$ (each small axis is centered at $\bm{\mathrm{s}}_\star$), while the expected stage cost stays equal to $\ell_{\star,0}(\rho_0)$.

\begin{figure}[t]
    \centering
    \definecolor{figState}{HTML}{0072B2}
    \definecolor{figDist}{HTML}{D55E00}
    \definecolor{figPurple}{HTML}{CC79A7}
    \definecolor{figGray}{HTML}{595959}
    \pgfmathdeclarefunction{rhozero}{1}{\pgfmathparse{0.72*exp(-(((#1)-1.35)^2)/(2*0.33^2))+0.3*exp(-(((#1)-0.45)^2)/(2*0.45^2))}}
    \pgfmathdeclarefunction{rhoone}{1}{\pgfmathparse{rhozero(-(#1))}}
    \begin{tikzpicture}[
        font=\footnotesize,
        >={Stealth[length=4pt,width=3.4pt]},
        panel/.style={font=\footnotesize, anchor=west, inner sep=0pt},
        rhoA/.style={draw=figDist, line width=0.8pt},
        rhoB/.style={draw=figPurple!85!black, line width=0.8pt, dash pattern=on 2.6pt off 1.4pt},
        fillA/.style={fill=figDist!14},
        fillB/.style={fill=figPurple!16},
    ]
        \node[panel] at (0,2.75) {\textbf{(a)} point reflection $G(\bm{\mathrm{s}})=2\bm{\mathrm{s}}_\star-\bm{\mathrm{s}}$};
        \begin{scope}[xshift=3.2cm]
            \fill[fillA] plot[domain=-1.05:3, samples=100, smooth] (\x, {1.6*rhozero(\x)}) -- (3,0) -- (-1.05,0) -- cycle;
            \fill[fillB] plot[domain=-3:1.05, samples=100, smooth] (\x, {1.6*rhoone(\x)}) -- (1.05,0) -- (-3,0) -- cycle;
            \draw[black!45, line width=0.6pt, dash pattern=on 0.6pt off 1.2pt] plot[domain=-3:3, samples=60, smooth] (\x, {0.12*\x*\x});
            \draw[line width=0.4pt, black!75] (-3.1,0) -- (3.1,0);
            \draw[black!55, line width=0.4pt, dash pattern=on 1.5pt off 1.5pt] (0,0) -- (0,1.5);
            \draw[rhoA] plot[domain=-1.05:3, samples=100, smooth] (\x, {1.6*rhozero(\x)});
            \draw[rhoB] plot[domain=-3:1.05, samples=100, smooth] (\x, {1.6*rhoone(\x)});
            \node[anchor=north, inner sep=2pt] at (0,0) {$\bm{\mathrm{s}}_\star$};
            \node[text=figDist!80!black] at (1.2,0.5) {$\rho_0$};
            \node[text=figPurple!70!black] at (-1.2,0.5) {$G_{\#}\rho_0$};
            \node[text=black!60, anchor=south east, font=\scriptsize, inner sep=1pt] at (3.05,1.12) {$\|\bm{\mathrm{s}}-\bm{\mathrm{s}}_\star\|_2^2$};
            \draw[->, line width=0.6pt, black!75] (1.25,1.42) to[out=150, in=30] node[above, inner sep=1.5pt] {$G$} (-1.25,1.42);
        \end{scope}
        \node[panel] at (7.35,2.75) {\textbf{(b)} distributions $\rho_k=G_{\#}^{k}\rho_0$ and expected cost $\ell_{\star,k}(\rho_0)$};
        \foreach \k in {0,...,4} {
            \begin{scope}[xshift={8.35cm+\k*1.62cm}, yshift=1.05cm]
                \draw[line width=0.4pt, black!75] (-0.72,0) -- (0.72,0);
                \draw[black!55, line width=0.4pt] (0,-0.05) -- (0,0.05);
                \pgfmathparse{int(mod(\k,2))}
                \ifnum\pgfmathresult=0
                    \fill[fillA] plot[domain=-0.26:0.72, samples=60, smooth] (\x, {0.95*rhozero(\x/0.24)}) -- (0.72,0) -- (-0.26,0) -- cycle;
                    \draw[rhoA, line width=0.7pt] plot[domain=-0.26:0.72, samples=60, smooth] (\x, {0.95*rhozero(\x/0.24)});
                    \node[anchor=north, inner sep=2pt, text=figDist] at (0,0) {$\rho_0$};
                \else
                    \fill[fillB] plot[domain=-0.72:0.26, samples=60, smooth] (\x, {0.95*rhoone(\x/0.24)}) -- (0.26,0) -- (-0.72,0) -- cycle;
                    \draw[rhoB, line width=0.7pt] plot[domain=-0.72:0.26, samples=60, smooth] (\x, {0.95*rhoone(\x/0.24)});
                    \node[anchor=north, inner sep=2pt, text=figPurple!85!black] at (0,0) {$G_{\#}\rho_0$};
                \fi
            \end{scope}
        }
        \draw[line width=0.4pt, black!75] (7.55,-0.35) -- (16.05,-0.35);
        \node[anchor=west, inner sep=1pt] at (16.05,-0.35) {$k$};
        \draw[line width=0.7pt, black!75] (8.35,0.12) -- (8.35+4*1.62,0.12);
        \draw[line width=0.7pt, black!75, dash pattern=on 1pt off 1.5pt] (8.35+4*1.62,0.12) -- (15.95,0.12);
        \foreach \k in {0,...,4} {
            \fill[black!80] ({8.35+\k*1.62},0.12) circle (1.6pt);
            \draw[line width=0.4pt, black!75] ({8.35+\k*1.62},-0.39) -- ({8.35+\k*1.62},-0.31);
            \node[anchor=north, inner sep=2pt] at ({8.35+\k*1.62},-0.39) {$\k$};
        }
        \node[anchor=east, inner sep=2pt] at (7.8,0.12) {$\ell_{\star,k}(\rho_0)$};
    \end{tikzpicture}
    \caption{Counterexample for a point reflection.}
    \label{fig:counterexample}
\end{figure}

\subsubsection{Recovering Successful Norm-Based Stabilization}
\label{appendix:recovery_nmdp}
The example does not mean that a norm-based MDP can never stabilize a probability distribution, only that the scalar objective alone does not guarantee it.
\namedref{Corollary}{coro:dmdp_recovers_successful_nmdp} shows that the DMDP formulation recovers every successful norm-based stabilization whose distributional dynamics an admissible DMDP policy can represent.

\begin{mycoro}[\textbf{Recovery of successful norm-based distributional stabilization}]
    \label{coro:dmdp_recovers_successful_nmdp}
    Let $\pi_\star$ be an optimal policy of the norm-based MDP in \namedref{Problem}{eq:prob_norm_based_mdp}.
    Suppose that $\pi_\star$ is $\mathbb{D}$-asymptotically stable with respect to $\rho_\star$ on a set $\mathcal{P}\subseteq M(\mathcal{S})$ that is forward invariant under $\mathcal{T}_{\tau,\pi_\star}$.
    Suppose further that the measure-level dynamics induced by $\pi_\star$ are representable in the DMDP formulation, i.e., that there exists an admissible DMDP policy $\psi_\pi\in\bm{\Psi}$ such that
    \begin{equation}
        \mathcal{T}_{\tau,\psi_\pi(\rho)}\rho
        =
        \mathcal{T}_{\tau,\pi_\star}\rho,
        \quad \forall \rho\in\mathcal{P}.
        \label{eq:dmdp_represent_nmdp_transition_cor}
    \end{equation}
    Then, the DMDP admits a policy, namely $\psi_\pi$, that is $\mathbb{D}$-asymptotically stable with respect to $\rho_\star$ on $\mathcal{P}$.
\end{mycoro}

\begin{proof}
Let $\psi_\pi\in\bm{\Psi}$ be the policy in~\eqref{eq:dmdp_represent_nmdp_transition_cor}, so that
\begin{equation}
    \mathcal{T}_{\tau,\psi_\pi(\rho)}\rho = \mathcal{T}_{\tau,\pi_\star}\rho, \quad \forall \rho\in\mathcal{P}.
    \label{eq:proof_dmdp_embedding_nmdp_transition}
\end{equation}
Fix $\rho_0\in\mathcal{P}$, and let $\{\rho_k^{\pi_\star}\}_{k=0}^{\infty}$ and $\{\rho_k^{\psi_\pi}\}_{k=0}^{\infty}$ denote the state-distribution sequences generated from $\rho_0$ by $\pi_\star$ and $\psi_\pi$, respectively.
By induction, $\rho_k^{\psi_\pi}=\rho_k^{\pi_\star}$ holds for all $k\geq 0$.
The claim holds for $k=0$, since $\rho_0^{\psi_\pi}=\rho_0^{\pi_\star}=\rho_0$.
If the claim holds for some $k\geq 0$, then $\rho_k^{\pi_\star}\in\mathcal{P}$ by the forward invariance of $\mathcal{P}$ under $\mathcal{T}_{\tau,\pi_\star}$, and~\eqref{eq:proof_dmdp_embedding_nmdp_transition} gives
\begin{equation}
    \rho_{k+1}^{\psi_\pi}
    = \mathcal{T}_{\tau,\psi_\pi(\rho_k^{\psi_\pi})}\rho_k^{\psi_\pi}
    = \mathcal{T}_{\tau,\psi_\pi(\rho_k^{\pi_\star})}\rho_k^{\pi_\star}
    = \mathcal{T}_{\tau,\pi_\star}\rho_k^{\pi_\star}
    = \rho_{k+1}^{\pi_\star}.
\end{equation}
Hence, from every initial distribution in $\mathcal{P}$, both policies generate the same distribution sequence.
Since $\pi_\star$ is $\mathbb{D}$-asymptotically stable with respect to $\rho_\star$ on $\mathcal{P}$, its $\mathbb{D}$-stability carries over to $\psi_\pi$, and $\lim_{k\rightarrow\infty}\mathbb{D}(\rho_k^{\psi_\pi}\|\rho_\star)=\lim_{k\rightarrow\infty}\mathbb{D}(\rho_k^{\pi_\star}\|\rho_\star)=0$.
Therefore, $\psi_\pi$ is $\mathbb{D}$-asymptotically stable with respect to $\rho_\star$ on $\mathcal{P}$.
\end{proof}

\noindent
In contrast to the NMDP, which obtains distributional stability only indirectly through the closed-loop dynamics, the DMDP formulates distributional stabilization directly through the evolution of $\rho_k$ and the dissimilarity $\mathbb{D}(\rho_k\|\rho_\star)$.

\subsection{Relation to POMDPs and Belief-State MDPs}
\label{appendix:relation_to_existing_frameworks}

A POMDP agent is commonly decomposed into a state estimator, which updates a belief state from the previous belief, the executed action, and the new observation, and a policy, which maps the belief state to an action~\citep{Astrom1965,Kaelbling1998Planning}.
The belief state is a probability distribution over hidden states and a sufficient statistic of the past action-observation history.
This view underlies belief-dependent rewards~\citep{Araya2010}, belief-space planning~\citep{Platt2010}, and belief-like representations learned in deep reinforcement learning~\citep{Igl2018,Lambrechts2022,Luis2025}.
In the parameterized DMDP of \namedref{Problem}{eq:parameter_based}, the parameter $\omega_k$ represents $\rho_k$ through $\rho_{\omega_k}$, and the policy $\phi:\Omega\rightarrow\mathcal{A}$ acts on this representation, much as a POMDP policy acts on a belief state.
However, the two distributions differ intrinsically.
In a POMDP, the belief state represents epistemic uncertainty about an unobserved physical state, or about unknown model parameters included in the state, and serves as an information state for maximizing the expected return.
It remains a pathwise quantity, i.e., it is computed from the observations of a single realization, and the belief-state MDP is stochastic because the belief depends on future random observations.
In the DMDP, $\rho_k$ is the distribution of the state across realizations rather than a pathwise quantity, the action is common to all realizations, and $\rho_k$ evolves deterministically through the distribution transition operator in~\eqref{eq:def_state_transition_operator}.
This distributional state results from the stage cost, as \namedref{Appendix}{appendix:relation_measure_functional} explains, rather than from missing information, and $\rho_k$ is the controlled object, which the policy stabilizes toward the target distribution $\rho_\star$ in the sense of $\mathbb{D}(\rho_k\|\rho_\star)$.

\subsection{Relation to Measure-Functional MDPs}
\label{appendix:relation_measure_functional}

Measure-functional MDPs~\citep{Gros_2022,Kordabad2022} analyze the stability of stochastic MDPs on the space of probability measures rather than solely on the state space.
There, a state-feedback law $\bm{\mathrm{a}}_k=\pi(\bm{\mathrm{s}}_k)$ induces the measure sequence $\rho_{k+1}=\mathcal{T}_{\tau,\pi}\rho_k$, stage costs and storage functions are generalized to functionals of probability measures, and stability is characterized through a dissimilarity measure.
State feedback keeps these works within the standard MDP setting; it is a modeling choice rather than a restriction of their dissipativity theory, and \namedref{Table}{tab:related_work_comparison} therefore marks distribution feedback with (\cmark).
Whether state feedback suffices is decided by the stage cost: for an expected per-state cost, which is linear in the distribution, an optimal state-feedback policy exists under standard conditions~\citep{Puterman1994,HernandezLerma1996}, whereas a stage cost such as $\mathbb{D}(\rho_k\|\rho_\star)$ generally makes the optimal action depend on the whole distribution, e.g., on its variance, concentration, tail behavior, or multimodality.
Distribution feedback thus follows from the DMDP stage cost, not from a design preference, which is why the DMDP takes $\rho_k$ as the decision state of the policy $\bm{\mathrm{a}}_k=\psi(\rho_k)$, or $\bm{\mathrm{a}}_k=\phi(\omega_k)$ in the parameterized formulation.
Conversely, as formalized in \namedref{Corollary}{coro:dmdp_recovers_successful_nmdp}, an admissible DMDP policy that represents the stabilizing dynamics $\mathcal{T}_{\tau,\pi}$ on $M(\mathcal{S})$ induced by a state-feedback policy $\pi$ reproduces the same trajectory of probability measures.
Dissipativity and turnpike notions have also been extended to random variables and probability measures for stochastic nonlinear optimal control~\citep{Schiessl2025TAC} and, with a stationary stochastic process in place of the optimal steady state, for generalized linear-quadratic problems~\citep{Schiessl2025SICON}; for stochastic economic MPC, they yield closed-loop stability and performance guarantees~\citep{Schiessl2024CDC,Schiessl2025arXiv}.
These works analyze optimal solutions and MPC closed loops; the DMDP uses dissipativity through the cost rotation in \namedref{Section}{sec:cost_rotation}, and \namedref{Section}{sec:stablization_policy_learning} extends the Lyapunov certificates of \namedref{Section}{sec:lyapunov_certificates} to learned distribution-feedback policies.

\subsection{Relation to Distributional Reinforcement Learning}
\label{appendix:relation_distributional_rl}

Distributional reinforcement learning~\citep{Bellemare2017,Bellemare2023} usually models the distribution of the random return rather than only its expectation; this richer representation of long-term reward uncertainty can improve value learning or risk-sensitive decision-making~\citep{Dabney2018}.
The DMDP instead models and controls the distribution $\rho_k$ of the system state rather than a representation of the value function or return, with the stability objective $\mathbb{D}(\rho_k\|\rho_\star)\rightarrow 0$.
This distinction matters when reliability, safety, or personalization must be specified at the distribution level: two policies may have similar expected costs or return distributions while inducing very different state distributions, with different confidence regions, tail risks, or multimodal behaviors, and the DMDP captures these differences directly.

\subsection{Further Related Approaches}
\label{appendix:further_related}

Lyapunov functions of the state have been used to certify the stability or safety of learned policies~\citep{Perkins2002,Berkenkamp2017,Chow2018,Chang2019,Han2020}, and \citet{Busoniu2018} and \citet{Krasowski2023} survey stability and provable safety in reinforcement learning.
These certificates concern the realized state rather than its distribution.

Mean-field MDPs have the distribution of a large population of agents as their state~\citep{Motte2022,Bauerle2023}, and reinforcement learning methods for mean-field control take this population distribution as a policy input~\citep{Carmona2023,Pasztor2023}.
These methods do not establish its stability.

Distribution steering drives the state distribution of a linear stochastic system to a prescribed Gaussian distribution at a finite final time, either as a terminal constraint~\citep{Chen2016} or through a Wasserstein terminal cost~\citep{Balci2021}.
Optimal transport has also been used to coordinate large populations of controlled agents~\citep{LeCorre2026}.

Other objectives are defined on the distribution of the states, or state-action pairs, visited by a policy and aggregated over time, such as the entropy of the state-visitation distribution~\citep{Hazan2019}, the matching of a target state marginal distribution~\citep{Lee2019}, and convex functions of the stationary state-action distribution~\citep{Zahavy2021}.

\section{Distribution Estimation and Model Consistency}
\label{appendix:distribution_estimation}

\subsection{Estimation of the Distribution Parameter \texorpdfstring{$\hat{\omega}_k$}{omega-hat k}}
\label{appendix:omega_estimation_procedure}

One possible procedure for the parameter-estimation step of \namedref{Section}{sec:approximate_dynamics_model} is given below.
Fix $k\in\{0,1,\ldots,T\}$, and let $c_k$ be the context of time step $k$, e.g., the time index or another grouping variable that collects samples from the same state distribution $\rho_k=\rho_{\omega_k}$.
The density of $\rho_\omega$, $\omega\in\Omega$, is $p(\bm{\mathrm{s}};\omega)$.
\begin{enumerate}
    \item \textbf{Conditional density estimation.}
    From the offline dataset $\mathcal{D}_{\mathsf{off}}$ in~\eqref{eq:data_time_series}, estimate the conditional density $\hat{p}_{\bm{\mathrm{s}}}^{N}(\bm{\mathrm{s}}\mid c_k)$ of the state sample $\bm{\mathrm{s}}$ given $c_k$, e.g., by kernel conditional density estimation~\citep{Hyndman1996}, normalizing flows~\citep{Papamakarios2021}, neural density estimators~\citep{Papamakarios2017}, or another consistent density estimator.
    \item \textbf{Sample generation.}
    Draw $S$ independent samples $\tilde{\bm{\mathrm{s}}}_{k}^{(1)},\ldots,\tilde{\bm{\mathrm{s}}}_{k}^{(S)}\sim\hat{p}_{\bm{\mathrm{s}}}^{N}(\cdot\mid c_k)$.
    \item \textbf{Parameter estimation.}
    Estimate $\omega_k$ by maximum likelihood:
    \begin{equation}
        \label{eq:prob_max_likelihood_omega}
        \hat{\omega}_{k}^{S,N}
        :=
        \arg\max_{\omega\in\Omega}
        \frac{1}{S}
        \sum_{j=1}^{S}
        \log p(\tilde{\bm{\mathrm{s}}}_{k}^{(j)};\omega).
    \end{equation}
\end{enumerate}
The superscript $(S,N)$ indicates the dependence on the offline dataset size $N$ and the number $S$ of generated samples; the main text writes $\hat{\omega}_k$.

\subsection{An Example Pipeline for Model Consistency}
\label{appendix:example_pipeline_model_consistency}

\namedref{Proposition}{prop:asymptotic_recovery_rotated_optimizers} assumes the uniform model consistency~\eqref{eq:uniform_model_consistency_main} on the compact set $\Omega_{\mathsf{com}}\times\mathcal{A}$.
The pipeline below, one possible implementation that does not restrict the framework, shows how standard statistical estimators meet this condition under suitable regularity and coverage assumptions.

\paragraph{Step 1: Conditional density estimation.}
For each context $c$, defined as in \namedref{Appendix}{appendix:omega_estimation_procedure}, the conditional density of the state sample $\bm{\mathrm{s}}$ in $\mathcal{D}_{\mathsf{off}}$ is $p_{\bm{\mathrm{s}}}(\bm{\mathrm{s}}\mid c)=p(\bm{\mathrm{s}};\omega_c)$ with the distribution parameter $\omega_c\in\Omega$.
Least-squares conditional density estimation (LS-CDE)~\citep{sugiyama10a} approximates $p_{\bm{\mathrm{s}}}(\bm{\mathrm{s}}\mid c)$ by $\bm{\beta}^\top\bm{\varphi}(\bm{\mathrm{s}},c)$, where $\bm{\varphi}(\bm{\mathrm{s}},c):=[\varphi_1(\bm{\mathrm{s}},c),\ldots,\varphi_b(\bm{\mathrm{s}},c)]^\top$ collects nonnegative basis functions and $\bm{\beta}\in\mathbb{R}^b$ is a coefficient vector.
For the flattened samples $\{(\bm{\mathrm{s}}_j,c_j)\}_{j=1}^{N}$ of $\mathcal{D}_{\mathsf{off}}$ and a regularization parameter $\lambda_N$, the regularized empirical LS-CDE objective is
\begin{equation}
    \widehat{\mathcal{J}}_{\lambda_N}(\bm{\beta})
    :=
    \bm{\beta}^\top\widehat{\bm{\mathrm{H}}}\bm{\beta}
    -
    2\widehat{\bm{\mathrm{h}}}^{\top}\bm{\beta}
    +
    \lambda_N\bm{\beta}^\top\bm{\beta},
\end{equation}
where $\widehat{\bm{\mathrm{H}}}:=\frac{1}{N}\sum_{j=1}^{N}\int\bm{\varphi}(\bm{\mathrm{s}},c_j)\bm{\varphi}(\bm{\mathrm{s}},c_j)^\top\mathsf{d}\bm{\mathrm{s}}$ and $\widehat{\bm{\mathrm{h}}}:=\frac{1}{N}\sum_{j=1}^{N}\bm{\varphi}(\bm{\mathrm{s}}_j,c_j)$.
Its minimizer has the closed form $\tilde{\bm{\beta}}=(\widehat{\bm{\mathrm{H}}}+\lambda_N\bm{\mathrm{I}}_b)^{-1}\widehat{\bm{\mathrm{h}}}$, and the normalized conditional density estimator is
\begin{equation}
    \hat{p}_{\bm{\mathrm{s}}}^{N}(\bm{\mathrm{s}}\mid c)
    :=
    \frac{
    [\tilde{\bm{\beta}}^\top\bm{\varphi}(\bm{\mathrm{s}},c)]_{+}
    }{
    \int
    [\tilde{\bm{\beta}}^\top\bm{\varphi}(\bm{\mathrm{s}},c)]_{+}
    \mathsf{d}\bm{\mathrm{s}}
    },
\end{equation}
where $[x]_{+}:=\max\{x,0\}$ ensures nonnegativity.
Under the regularity conditions of LS-CDE, such as an appropriate choice of $\lambda_N$ and a sufficiently rich basis, $\|\hat{p}_{\bm{\mathrm{s}}}^{N}(\cdot\mid c)-p_{\bm{\mathrm{s}}}(\cdot\mid c)\|_{2}\rightarrow0$ in probability and, under strengthened consistency conditions, almost surely; the almost-sure version is assumed below.

\paragraph{Step 2: Parameter estimation by maximum likelihood.}
As in \namedref{Appendix}{appendix:omega_estimation_procedure}, $S$ independent samples $\tilde{\bm{\mathrm{s}}}_{c}^{(1)},\ldots,\tilde{\bm{\mathrm{s}}}_{c}^{(S)}\sim\hat{p}_{\bm{\mathrm{s}}}^{N}(\cdot\mid c)$ give the maximum-likelihood estimate
\begin{equation}
    \label{eq:mle_omega_pipeline}
    \hat{\omega}_{c}^{S,N}
    :=
    \arg\max_{\omega\in\Omega_{\mathsf{com}}}
    \frac{1}{S}
    \sum_{j=1}^{S}
    \log p(\tilde{\bm{\mathrm{s}}}_{c}^{(j)};\omega).
\end{equation}
Its consistency follows a standard route~\citep{Newey1994,vanderVaart1998}.
Define the population log-likelihoods
\begin{equation}
    Q_N(\omega)
    :=
    \int
    \log p(\bm{\mathrm{s}};\omega)
    \hat{p}_{\bm{\mathrm{s}}}^{N}(\bm{\mathrm{s}}\mid c)
    \mathsf{d}\bm{\mathrm{s}},
    \quad
    Q(\omega)
    :=
    \int
    \log p(\bm{\mathrm{s}};\omega)
    p_{\bm{\mathrm{s}}}(\bm{\mathrm{s}}\mid c)
    \mathsf{d}\bm{\mathrm{s}}.
\end{equation}
Assume that $\Omega_{\mathsf{com}}$ is compact, the family $\{p(\cdot;\omega)\mid\omega\in\Omega_{\mathsf{com}}\}$ is identifiable, the true parameter $\omega_c$ is the unique maximizer of $Q$, and $\log p(\bm{\mathrm{s}};\omega)$ is continuous in $\omega$.
Assume further that an envelope $E$, integrable under $p_{\bm{\mathrm{s}}}(\cdot\mid c)$, satisfies $|\log p(\bm{\mathrm{s}};\omega)|\leq E(\bm{\mathrm{s}})$ for all $\omega\in\Omega_{\mathsf{com}}$, and that the density estimator converges in the $E$-weighted $L^1$ sense, i.e., $\int E(\bm{\mathrm{s}})\,|\hat{p}_{\bm{\mathrm{s}}}^{N}(\bm{\mathrm{s}}\mid c)-p_{\bm{\mathrm{s}}}(\bm{\mathrm{s}}\mid c)|\,\mathsf{d}\bm{\mathrm{s}}\rightarrow0$ almost surely.
Convergence in $L^2$ alone does not suffice, because the estimated densities may place vanishing mass far in the tails, where $|\log p(\bm{\mathrm{s}};\omega)|$ is large.
Since this weighted $L^1$ distance bounds $|Q_N(\omega)-Q(\omega)|$ for every $\omega\in\Omega_{\mathsf{com}}$, it follows that $\sup_{\omega\in\Omega_{\mathsf{com}}}|Q_N(\omega)-Q(\omega)|\rightarrow0$ almost surely.
By the argmax theorem~\citep{vanderVaart1996}, the population maximizer of $Q_N$ converges to $\omega_c$.
Moreover, conditional on $\hat{p}_{\bm{\mathrm{s}}}^{N}(\cdot\mid c)$, the empirical log-likelihood in~\eqref{eq:mle_omega_pipeline} converges uniformly to $Q_N$ as $S\rightarrow\infty$ under the same compactness and envelope conditions, provided $E$ is integrable under $\hat{p}_{\bm{\mathrm{s}}}^{N}(\cdot\mid c)$.
Hence, if $S=S(N)$ grows fast enough for this Monte Carlo error to vanish almost surely as $N\rightarrow\infty$, then $\hat{\omega}_{c}^{S,N}\rightarrow\omega_c$ almost surely as $N\rightarrow\infty$.

\paragraph{Step 3: Learning the parameter transition map.}
The estimated parameters form the transition dataset $\mathcal{D}_{\omega}^{S,N}:=\{(\hat{\omega}_{j}^{S,N},\bm{\mathrm{a}}_j,\hat{\omega}_{j,+}^{S,N})\}_{j=1}^{n_{\mathsf{tr}}}$, where $\hat{\omega}_{j}^{S,N}$ and $\hat{\omega}_{j,+}^{S,N}$ estimate the distribution parameters $\omega_j$ and $\omega_{j,+}=f(\omega_j,\bm{\mathrm{a}}_j)$ before and after the action $\bm{\mathrm{a}}_j$.
Since $\hat{\omega}_{j}^{S,N}\rightarrow\omega_j$ and $\hat{\omega}_{j,+}^{S,N}\rightarrow\omega_{j,+}$ almost surely, $\mathcal{D}_{\omega}^{S,N}$ provides asymptotically correct training data for $f$.
Consider Gaussian process regression~\citep{Rasmussen2006} with a continuous universal kernel $\mathsf{Ker}$~\citep{Micchelli2006} on the compact domain $\mathcal{X}_{\mathsf{com}}:=\Omega_{\mathsf{com}}\times\mathcal{A}$, with one Gaussian process per component of $f:\Omega\times\mathcal{A}\rightarrow\Omega$.
Assume that each component of $f$ belongs to the reproducing kernel Hilbert space (RKHS) of $\mathsf{Ker}$, that the exact design points $\{(\omega_j,\bm{\mathrm{a}}_j)\}_{j=1}^{n_{\mathsf{tr}}}$ become dense in $\mathcal{X}_{\mathsf{com}}$ as $n_{\mathsf{tr}}\rightarrow\infty$, and that the input perturbations caused by parameter estimation vanish uniformly, i.e., $\max_{1\leq j\leq n_{\mathsf{tr}}}\|\hat{\omega}_{j}^{S,N}-\omega_j\|\rightarrow0$ almost surely.
Under these conditions and a suitable regularization and noise schedule, the regression is assumed to be uniformly consistent, i.e., the learned model $\hat{f}_{\mathsf{off}}^{S,N}$ satisfies~\eqref{eq:uniform_model_consistency_main} in place of $\hat{f}_{\mathsf{off}}^{N}$.
Error bounds for Gaussian process regression and kernel interpolation~\citep{Wendland2004,Srinivas2012,Lederer2019} are available for exactly observed inputs; since the inputs here are estimated, uniform consistency is assumed rather than derived from them.

\paragraph{Consequence.}
The pipeline reduces the model-consistency condition~\eqref{eq:uniform_model_consistency_main} to two requirements: (i) consistent estimates of the distribution parameters, which Steps 1 and 2 provide under the stated conditions; and (ii) uniform consistency of the regression with estimated inputs, which is assumed.
\namedref{Appendix}{appendix:toy_example} applies Steps 2 and 3 to a scalar linear example, in which the model error decreases as the data grow.

\section{A Computable Example}
\label{appendix:toy_example}

This appendix tests \namedref{Theorems}{theo:parameter_stability_means_distribution_stability}--\ref{theo:model_free_online_policy_learning} and \namedref{Proposition}{prop:asymptotic_recovery_rotated_optimizers} on a system whose hypotheses can all be checked, as \namedref{Table}{tab:toy_hypotheses} shows, and whose certificates can be computed because the model is known.

\begin{table}[!ht]
    \centering
    \caption{Hypotheses in the example.}
    \label{tab:toy_hypotheses}
    {\footnotesize\setlength{\tabcolsep}{3pt}
\begin{tabular}{@{}>{\raggedright\arraybackslash}p{0.10\linewidth}>{\raggedright\arraybackslash}p{0.42\linewidth}>{\raggedright\arraybackslash}p{0.14\linewidth}>{\raggedright\arraybackslash}p{0.30\linewidth}@{}}
\toprule
Result & Hypothesis & Status & Evidence \\
\midrule
All & closure; $\mathbb{D}$-compatibility with $\alpha_1(r)=\alpha_2(r)=r^2$ & exact & Gaussian AR(1); $\mathbb{D}=W_2^2=r^2$ \\
All & $\Omega_{\mathsf{com}}$ compact and positively invariant under every closed loop & exact; for $\widehat F_{\mathsf{off}}$ by construction & Lemma~\ref{lemma:toy_facts}; $\hat f_{\mathsf{off}}$ projected \\
Thm.~\ref{theo:parameter_stability_means_distribution_stability} & optimal $\phi_\star$; $U_\star^{\mathsf{R}}$ finite, continuous, Bellman identity; $F_{\phi_\star}(\omega_\star)=\omega_\star$; $\kappa_1(r)=r^2$, $\kappa_2(r)=2.56\,r^2$ & exact & Lemma~\ref{lemma:toy_facts}; $\max\tilde U/\mathbb{D}=1.57$ \\
Thms.~\ref{theo:stability_finite_samples}--\ref{theo:model_free_online_policy_learning} & $U_\star^{\mathsf{R}}$ and the sups over $\Omega_{\mathsf{com}}$ & numerical & $\tilde U$ with $\varepsilon_{\mathrm{num}}=1.4\times10^{-3}$; maxima over 11,822 points (20,016 for $e_f$); 80,000 independent points change $\Delta$, $\hat c$, $\Delta_{\mathsf{dec},n}$ of the deployed run by at most $1.9\%$ \\
Thm.~\ref{theo:stability_finite_samples} & $F_{\mathsf{off}}(\omega_\star)=\omega_\star$ & fails & offset $>0$; not used by~\eqref{eq:true_lyapunov_decrease_with_residual} \\
Thm.~\ref{theo:stability_finite_samples} & learned margin with $\sigma(r)=0.75\,r^2$; $\Delta\le\bar\Delta$ with probability $\ge0.9$ & computed; fails at $n_{\mathsf{tr}}=100$ & per run; maximum of 49 runs, confidence $0.994$ \\
Thms.~\ref{theo:online_plugin_stability}--\ref{theo:model_free_online_policy_learning} & estimation event for all $t\ge0$ with probability $\ge0.95$ & exact & growing $M_t$ \\
Thm.~\ref{theo:model_free_online_policy_learning} & $W_n=\mathbb{D}$: bounds and continuity & exact & -- \\
Thm.~\ref{theo:model_free_online_policy_learning} & empirical decrease with $\sigma_n(r)=0.25\,r^2$; $\Delta_{\mathsf{dec},n}\le\bar\Delta_{\mathsf{dec},n}$ with probability $\ge0.9$ & by construction; computed & capped estimate; maximum of 49 runs, with the true $f$ \\
Thms.~\ref{theo:stability_finite_samples}--\ref{theo:model_free_online_policy_learning} & asymptotic clause & fails & residual sets only \\
Prop.~\ref{prop:asymptotic_recovery_rotated_optimizers} & uniform model consistency & trend & $e_f$ \\
Prop.~\ref{prop:asymptotic_recovery_rotated_optimizers} & $\Theta$ compact, $\phi$ continuous; $J^{\mathsf{R}}$ finite and continuous & exact; fails & node actions; $J^{\mathsf{R}}$ infinite with an offset \\
\bottomrule
\end{tabular}}

\end{table}

\paragraph{System.}
The scalar state follows $s_{k+1}=a_{k,1}s_k+(1-a_{k,1})a_{k,2}+\epsilon_k$ with i.i.d.\ $\epsilon_k\sim\mathcal{N}(0,0.5^2)$ and an action $\bm{\mathrm{a}}_k=(a_{k,1},a_{k,2})\in\mathcal{A}=[0.70,0.95]\times[0.5,1.5]$ common to all realizations.
The Gaussian family $\rho_\omega=\mathcal{N}(\mu,\varsigma^2)$, $\omega=(\mu,\varsigma)\in\Omega=\mathbb{R}\times(0,\infty)$, is closed under this transition, so~\eqref{eq:parameter_dynamics} holds exactly with $f(\omega,\bm{\mathrm{a}})=\big(a_1\mu+(1-a_1)a_2,\,(a_1^2\varsigma^2+0.25)^{1/2}\big)$ for $\bm{\mathrm{a}}=(a_1,a_2)\in\mathcal{A}$.
Thus, every $\rho_k$ stays in the parameterized family, so the example does not test the approximation of distributions outside it, one of the most difficult aspects of the general problem; it tests the error terms and residual sets of \namedref{Section}{sec:stablization_policy_learning} under the errors of the learned parameter dynamics and value function, of the estimated decrease, and of the estimated parameters.

The target $\rho_\star=\rho_{\omega_\star}$, $\omega_\star=(1,0.8)$, is a steady state only under $\bm{\mathrm{a}}_\star=(\sqrt{39}/8,1)\approx(0.78,1)$.
The dissimilarity is the squared $2$-Wasserstein distance, $\mathbb{D}(\rho_\omega\|\rho_\star)=(\mu-1)^2+(\varsigma-0.8)^2=r^2$ with $r:=\|\omega-\omega_\star\|$; hence, \namedref{Assumption}{assump:D_compatible_parameterization} holds with $\alpha_1(r)=\alpha_2(r)=r^2$.
The operating set $\Omega_{\mathsf{com}}=[-1,3]\times[\varsigma_{\mathsf{lo}},\varsigma_{\mathsf{hi}}]$, bounded by the stationary standard deviations $\varsigma_{\mathsf{lo}}=0.5/\sqrt{0.51}\approx0.70$ at $a_1=0.70$ and $\varsigma_{\mathsf{hi}}=0.5/\sqrt{0.0975}\approx1.60$ at $a_1=0.95$, contains all steady states $\mu=a_2$, $\varsigma^2=0.25/(1-a_1^2)$.

With $\ell(\omega,\bm{\mathrm{a}})=\mathbb{D}(\rho_\omega\|\rho_\star)$, $\gamma=1$, and the storage function $\lambda=0$, the value $\ell_0=0$ in~\eqref{eq:parameter_based_optimal_steady} is attained only at $\omega_\star$, and $\ell^{\mathsf{R}}=\ell=r^2$.
The action $\bm{\mathrm{a}}_{\mathsf{db}}(\omega)$ with $a_{\mathsf{db},1}=\sqrt{0.39}/\varsigma$ and $a_{\mathsf{db},2}=(1-a_{\mathsf{db},1}\mu)/(1-a_{\mathsf{db},1})$ gives $f(\omega,\bm{\mathrm{a}}_{\mathsf{db}}(\omega))=\omega_\star$, and $\mathcal{R}:=\{\omega\in\Omega_{\mathsf{com}}\mid\bm{\mathrm{a}}_{\mathsf{db}}(\omega)\in\mathcal{A}\}$ is the set from which $\omega_\star$ is reached in one step.

\begin{mylemma}
    \label{lemma:toy_facts}
    (i) $f(\Omega_{\mathsf{com}}\times\mathcal{A})\subseteq\Omega_{\mathsf{com}}$.\\
    (ii) $\|f(\omega,\bm{\mathrm{a}}_\star)-\omega_\star\|^2\leq\frac{39}{64}\|\omega-\omega_\star\|^2$ for all $\omega\in\Omega$.\\
    (iii) $U_\star^{\mathsf{R}}$ is finite and continuous on $\Omega_{\mathsf{com}}$, with $r^2\leq U_\star^{\mathsf{R}}\leq2.56\,r^2$; an optimal stationary policy $\phi_\star$ exists; the Bellman identity~\eqref{eq:rotated_bellman_identity_parameter} holds on $\Omega_{\mathsf{com}}$; and $F_{\phi_\star}(\omega_\star)=\omega_\star$ with $\phi_\star(\omega_\star)=\bm{\mathrm{a}}_\star$.\\
    (iv) $U_\star^{\mathsf{R}}=\mathbb{D}$ and $\phi_\star=\bm{\mathrm{a}}_{\mathsf{db}}$ on $\mathcal{R}\supseteq\{\omega\in\Omega_{\mathsf{com}}\mid r\leq0.09\}$.
\end{mylemma}

\begin{proof}
(i) The new mean is a convex combination of $\mu\in[-1,3]$ and $a_2\in[0.5,1.5]$; the new standard deviation increases in $a_1$ and $\varsigma$ and has the fixed points $\varsigma_{\mathsf{lo}}$ at $a_1=0.70$ and $\varsigma_{\mathsf{hi}}$ at $a_1=0.95$.

\noindent (ii) For $(\mu^+,\varsigma^+):=f(\omega,\bm{\mathrm{a}}_\star)$, $\mu^+-1=a_{\star,1}(\mu-1)$ and $(\varsigma^+)^2-0.64=a_{\star,1}^2(\varsigma^2-0.64)$.
Since $\varsigma^+\geq a_{\star,1}\varsigma$ and $a_{\star,1}<1$, it follows that $|\varsigma^+-0.8|=a_{\star,1}^2|\varsigma-0.8|(\varsigma+0.8)/(\varsigma^++0.8)\leq a_{\star,1}|\varsigma-0.8|$, and $a_{\star,1}^2=39/64$ gives (ii).

\noindent (iii) Let $\mathbb{D}_k:=\mathbb{D}(\rho_{\omega_k}\|\rho_\star)$, and let $V_m(\omega)$ be the minimum of $\sum_{k<m}\mathbb{D}_k$ over action sequences in $\mathcal{A}^m$ from $\omega_0=\omega$.
$V_m$ is continuous by Berge's maximum theorem, since $f$ and $\mathbb{D}$ are continuous, $\mathcal{A}^m$ is compact, and (i) removes any state constraint.
By (ii), applying $\bm{\mathrm{a}}_\star$ forever costs at most $r^2/(1-39/64)=2.56\,r^2$, so $r^2\leq V_m\leq2.56\,r^2$.
An optimal $m$-step sequence has a step $k'<m$ with $\mathbb{D}_{k'}\leq V_m/m$; following it up to $k'$ and $\bm{\mathrm{a}}_\star$ afterwards costs at most $(1+2.56/m)V_m$.
Hence, $V_m$ increases in $m$ and converges uniformly to a continuous $\bar{U}\leq(1+2.56/m)V_m$.
The Bellman operator $(\Gamma U)(\omega):=\mathbb{D}(\rho_\omega\|\rho_\star)+\min_{\bm{\mathrm{a}}\in\mathcal{A}}U(f(\omega,\bm{\mathrm{a}}))$ is monotone and nonexpansive in the sup norm, and $V_{m+1}=\Gamma V_m$, so $\bar{U}=\Gamma\bar{U}$.
For a policy $\phi_\star$ that selects a minimizer in $\Gamma\bar{U}$, iterating along $F_{\phi_\star}$ gives $\sum_{k<H}\mathbb{D}_k+\bar{U}(\omega_H)=\bar{U}(\omega_0)$, so $\mathbb{D}_k\rightarrow0$, $\bar{U}(\omega_H)\leq2.56\,\mathbb{D}_H\rightarrow0$, and $\phi_\star$ attains $\bar{U}$.
Since every policy costs at least $V_m$ for every $m$, $U_\star^{\mathsf{R}}=\bar{U}$, $\phi_\star$ is optimal, and the Bellman identity holds.
At $\omega_\star$, this identity gives $U_\star^{\mathsf{R}}(F_{\phi_\star}(\omega_\star))=0$; hence, $F_{\phi_\star}(\omega_\star)=\omega_\star$, which only $\bm{\mathrm{a}}_\star$ achieves.

\noindent (iv) On $\mathcal{R}$, $\mathbb{D}\leq U_\star^{\mathsf{R}}\leq\mathbb{D}+U_\star^{\mathsf{R}}(\omega_\star)=\mathbb{D}$, so the Bellman identity requires $F_{\phi_\star}(\omega)=\omega_\star$, which only $\bm{\mathrm{a}}_{\mathsf{db}}$ achieves.
If $r\leq0.09$, then $a_{\mathsf{db},1}\in[0.70,0.88]$ and $|a_{\mathsf{db},2}-1|=a_{\mathsf{db},1}|\mu-1|/(1-a_{\mathsf{db},1})\leq0.38$, so $\bm{\mathrm{a}}_{\mathsf{db}}(\omega)\in\mathcal{A}$.
\end{proof}

\noindent
Hence, \namedref{Theorem}{theo:parameter_stability_means_distribution_stability} applies with $\Omega_{\mathcal{P}}=\Omega_{\mathsf{com}}$, $\kappa_1(r)=r^2$, and $\kappa_2(r)=2.56\,r^2$: $\phi_\star$ renders the example $\mathbb{D}$-asymptotically stable and, by (iv), reaches $\rho_\star$ exactly in finitely many steps.

\paragraph{Computations.}
Value iteration on a grid yields an approximation $\tilde{U}$ of $U_\star^{\mathsf{R}}$.
The grid has $\omega_\star$ as a node and is refined near it; $\bm{\mathrm{a}}_\star$ is among the candidate actions; and $U/\mathbb{D}$ is interpolated bilinearly.
The error indicator $\varepsilon_{\mathrm{num}}$ is the largest of the Bellman residual, the gap to a rigorous upper bound on $U_\star^{\mathsf{R}}$ (the $25$-step cost of the greedy policy plus $2.56\,\mathbb{D}$ at the end), and the difference from a coarser grid.

For \namedref{Theorem}{theo:stability_finite_samples}, a data level $(B,M)\in\{(20,50),(40,200),(80,800),(160,3200)\}$ consists of $B$ independent datasets~\eqref{eq:data_time_series} with $T=5$, each with $\omega_0$ and actions drawn uniformly from $\Omega_{\mathsf{com}}$ and $\mathcal{A}$, so that $n_{\mathsf{tr}}=BT$ and $N=BTM$.
The maximum-likelihood estimate $\hat{\omega}_k$ from the samples of $\rho_k$ and the model $\hat{f}_{\mathsf{off}}$, one Gaussian process per output, are projected onto $\Omega_{\mathsf{com}}$, so $\Omega_{\mathsf{com}}$ is invariant under $\widehat{F}_{\mathsf{off}}$.
The estimate does not need the density step of \namedref{Appendix}{appendix:omega_estimation_procedure}.

The learner solves $\hat{f}_{\mathsf{off}}(\omega_\star,\hat{\bm{\mathrm{a}}})=\omega_\star$ for $\hat{\bm{\mathrm{a}}}$, since otherwise the learned value with $\gamma=1$ would be infinite, and then solves \namedref{Problem}{eq:parameter_based_approximation} by value iteration with actions centered at $\hat{\bm{\mathrm{a}}}$, never at $\bm{\mathrm{a}}_\star$.
Bilinear interpolation of the node actions $\theta$ gives $\hat{\phi}_{\mathsf{off}}$, with $\hat{\phi}_{\mathsf{off}}(\omega_\star)=\hat{\bm{\mathrm{a}}}$.
The proof of \namedref{Theorem}{theo:stability_finite_samples} uses only the learned margin, the invariance, and the definitions of $\Delta_U$ and $\Delta_f$, so it applies verbatim to the computed $\widehat{U}_{\mathsf{off}}^{\mathsf{R}}$.

A run passes if $\hat{\bm{\mathrm{a}}}$ is found, the value iteration converges, and $\hat{c}:=\min_{\omega\neq\omega_\star}[\widehat{U}_{\mathsf{off}}^{\mathsf{R}}(\omega)-\widehat{U}_{\mathsf{off}}^{\mathsf{R}}(\widehat{F}_{\mathsf{off}}(\omega))]/r^2\geq0.75$, i.e., if~\eqref{eq:learned_lyapunov_decrease_margin_finite_sample} holds with the preset $\sigma(r)=0.75\,r^2$.
The score is $S=\Delta$ for a passing run and $S=\infty$ otherwise, so $\{S\leq\bar{\Delta}\}$ is contained in the event $\mathcal E_{\mathsf{off}}$ of the theorem.
Each level has $50$ independent runs; $\bar{\Delta}$ is the largest score of runs $0$--$48$, and run $49$ is deployed.
Hence, with confidence $1-0.9^{49}\approx0.994$ over the calibration runs, $S_{49}\leq\bar{\Delta}$ holds with probability at least $0.9$, i.e., $\alpha=0.1$.

Since the proofs of \namedref{Theorems}{theo:stability_finite_samples}--\ref{theo:model_free_online_policy_learning} are pathwise, a run satisfies their conclusions with its own $\Delta$ or $\Delta_{\mathsf{dec},n}$ in place of $\bar{\Delta}$ or $\bar{\Delta}_{\mathsf{dec},n}$ if it satisfies the learned margin and, for \namedref{Theorems}{theo:online_plugin_stability} and~\ref{theo:model_free_online_policy_learning}, lies on the estimation event.
The residual set~\eqref{eq:finite_sample_residual_set} is $\Omega_{\mathsf{com}}\cap\{r\leq\bar{r}\}$ with $\bar{r}:=(\bar{\Delta}/0.75)^{1/2}$, and $e_f$ is the supremum in~\eqref{eq:uniform_model_consistency_main}.

For \namedref{Theorem}{theo:online_plugin_stability}, the deployed $\hat{\phi}_{\mathsf{off}}$ with $n_{\mathsf{tr}}=800$ is used, with $\sigma_{\mathsf{on}}=\sigma$ and $\bar{\Delta}_{\mathsf{off}}=\bar{\Delta}$.
At step $t$, $\hat{\omega}_t$ is estimated as above from $M=M_t$ fresh samples; before the projection, its mean and standard deviation are $\mu_t+\varsigma_tZ/\sqrt{M}$ and $\varsigma_t\sqrt{X/M}$ with independent $Z\sim\mathcal{N}(0,1)$ and $X\sim\chi^2_{M-1}$.

As the projection is nonexpansive, $\|\hat{\omega}_t-\omega_t\|\leq\varsigma_tG_{M_t}$ with $G_M:=\|(Z/\sqrt{M},\sqrt{X/M}-1)\|$.
Since the same draws $(Z,X)$ are used for every $\omega_0$ and policy, $G_{M_t}$ depends on neither.
With $\delta_{\mathsf{on}}=\alpha_{\mathsf{on}}=0.05$, $M_t$ is the smallest $M$ with $P(G_M>\delta_{\mathsf{on}}/\varsigma_{\mathsf{hi}})\leq\alpha_{\mathsf{on}}/((t+1)(t+2))$, computed by quadrature, so a union bound over $t$ gives~\eqref{eq:online_estimator_uniform_error} for every $\omega_0$ with probability at least $1-\alpha_{\mathsf{on}}$.
By \namedref{Lemma}{lemma:toy_facts}(i), $F_{\mathsf{on}}$ keeps $\Omega_{\mathsf{com}}$ invariant; with the calibration, the conclusion holds with probability at least $0.9\times0.95$.

For \namedref{Theorem}{theo:model_free_online_policy_learning}, $W_n=\mathbb{D}$ satisfies~\eqref{eq:model_free_Wn_bounds} exactly, and the learner runs $160$ episodes of $5$ steps from uniformly drawn $\omega_0$ with $20{,}000$ samples per step.
After $n=20,40,80,160$ episodes, i.e., $N_n=100,\ldots,800$ in~\eqref{eq:model_free_online_dataset}, a Gaussian process $\widehat{W}^{+}_n(\omega,\bm{\mathrm{a}})$ is fitted to an unbiased estimate of the next dissimilarity, based on the unbiased sample variance, with the variance of this estimate as the noise variance.
The policy $\phi_n$ interpolates the minimizers of $\widehat{W}^{+}_n$ over $\mathcal{A}$ at the grid nodes.

The learned decrease $\widehat{\mathcal{D}}_n(\omega):=\min\{\widehat{W}^{+}_n(\omega,\phi_n(\omega))-\mathbb{D}(\rho_\omega\|\rho_\star),\,-\sigma_n(r)\}$ with $\sigma_n(r)=0.25\,r^2$ satisfies~\eqref{eq:model_free_empirical_decrease_certificate} by construction.
The cap raises $\Delta_{\mathsf{dec},n}$ by at most $e_n\geq0$, the amount by which the uncapped estimate violates~\eqref{eq:model_free_empirical_decrease_certificate}.
The constant $\bar{\Delta}_{\mathsf{dec},n}$ is calibrated like $\bar{\Delta}$, and the same $M_t$ are used with $\delta_{\omega,n}=0.05$ and $\alpha_n=1-0.9\times0.95$.

All suprema and the minimum in $\hat{c}$ are taken over finite sets of test points, as \namedref{Table}{tab:toy_hypotheses} shows.
The smallest residual, i.e., the smallest constant that can replace the total residual in~\eqref{eq:true_lyapunov_decrease_with_residual}, \eqref{eq:online_true_lyapunov_decrease_residual}, or~\eqref{eq:model_free_true_decrease_with_residual}, is evaluated with $\tilde{U}$ and $f$ on these test points for \namedref{Theorems}{theo:stability_finite_samples} and~\ref{theo:model_free_online_policy_learning} and on simulated plug-in steps for \namedref{Theorems}{theo:online_plugin_stability} and~\ref{theo:model_free_online_policy_learning}.

If $|\tilde{U}-U_\star^{\mathsf{R}}|\leq\varepsilon_{\mathrm{num}}$, as the indicators suggest, replacing $U_\star^{\mathsf{R}}$ by $\tilde{U}$ changes $\Delta_U$ by at most $\varepsilon_{\mathrm{num}}$; $\Delta_f$, $\Delta_{\mathsf{on}}$, and the smallest residual by at most $2\varepsilon_{\mathrm{num}}$; and $\Delta$ by at most $4\varepsilon_{\mathrm{num}}$, which is small compared with $\bar{\Delta}$ but of the order of $\Delta_{\mathsf{on}}$.

\paragraph{Results.}
In \namedref{Table}{tab:toy_results}, medians are over the $50$ runs, $\bar{r}$ is the radius of the residual set~\eqref{eq:finite_sample_residual_set}, \eqref{eq:online_residual_set}, or~\eqref{eq:model_free_residual_set}, and the smallest residual and the limit distance refer to the deployed run.
The limit distance is $\|\omega-\omega_\star\|$ where the closed loop settles: the largest value after $300$ steps from eight starts or, for \namedref{Theorem}{theo:online_plugin_stability}, the median over $200$ plug-in runs of the largest value in steps $30$--$39$.
``Failed'' counts the runs that do not pass for \namedref{Theorem}{theo:stability_finite_samples}, the plug-in runs whose estimation event fails within $40$ steps for \namedref{Theorem}{theo:online_plugin_stability}, and the runs with $e_n>0$ for \namedref{Theorem}{theo:model_free_online_policy_learning}.

\begin{table}[!ht]
    \centering
    \caption{Certificates of the learned policies.}
    \label{tab:toy_results}
    {\footnotesize
\begin{tabular*}{\linewidth}{@{\extracolsep{\fill}}lccccccc@{}}
\toprule
Data & \begin{tabular}[b]{@{}c@{}}$e_f$ or\\$\Delta_{\mathsf{dec},n}$ (median)\end{tabular} & \begin{tabular}[b]{@{}c@{}}$\bar\Delta$ or\\$\bar\Delta_{\mathsf{dec},n}$\end{tabular} & \begin{tabular}[b]{@{}c@{}}$\Delta_{\mathsf{on}}$ or\\$\Delta_{\omega,n}$\end{tabular} & \begin{tabular}[b]{@{}c@{}}residual\\radius $\bar r$\end{tabular} & \begin{tabular}[b]{@{}c@{}}smallest\\residual\end{tabular} & \begin{tabular}[b]{@{}c@{}}limit\\distance\end{tabular} & failed \\
\midrule
\multicolumn{8}{@{}l}{\emph{Thm.~\ref{theo:stability_finite_samples}, offline}} \\
$n_{\mathsf{tr}}=100$ & $0.22$ & $\infty$ & -- & $\infty$ & $1.0\times10^{-3}$ & $0.032$ & 5/50 \\
$n_{\mathsf{tr}}=200$ & $0.089$ & $1.3$ & -- & $1.3$ & $1.9\times10^{-5}$ & $4.2\times10^{-3}$ & 0/50 \\
$n_{\mathsf{tr}}=400$ & $0.036$ & $0.56$ & -- & $0.86$ & $1.2\times10^{-6}$ & $1.0\times10^{-3}$ & 0/50 \\
$n_{\mathsf{tr}}=800$ & $0.014$ & $0.20$ & -- & $0.52$ & $1.0\times10^{-6}$ & $9.2\times10^{-4}$ & 0/50 \\
\midrule
\multicolumn{8}{@{}l}{\emph{Thm.~\ref{theo:online_plugin_stability}, online plug-in}} \\
$n_{\mathsf{tr}}=800$ & -- & $0.20$ & $2.7\times10^{-3}$ & $0.52$ & $4.4\times10^{-4}$ & $9.7\times10^{-3}$ & 2/200 \\
\midrule
\multicolumn{8}{@{}l}{\emph{Thm.~\ref{theo:model_free_online_policy_learning}, estimated decrease}} \\
$N_n=100$ & $0.14$ & $0.40$ & $0.020$ & $1.3$ & $9.4\times10^{-3}$ & $0.040$ & 34/50 \\
$N_n=200$ & $0.058$ & $0.16$ & $6.6\times10^{-3}$ & $0.81$ & $3.2\times10^{-4}$ & $6.1\times10^{-3}$ & 23/50 \\
$N_n=400$ & $0.039$ & $0.089$ & $6.3\times10^{-3}$ & $0.62$ & $3.1\times10^{-4}$ & $5.6\times10^{-3}$ & 8/50 \\
$N_n=800$ & $0.024$ & $0.068$ & $6.4\times10^{-3}$ & $0.54$ & $3.2\times10^{-4}$ & $5.3\times10^{-3}$ & 5/50 \\
\bottomrule
\end{tabular*}}

\end{table}

In \namedref{Figure}{fig:toy_certificates}, panel (a) shows level sets of $\tilde{U}$, the set $\mathcal{R}$ (dashed), and three trajectories of $\phi_\star$; panels (b)--(d) use the deployed policies with $n_{\mathsf{tr}}=N_n=800$, and panel (d) starts from $(-1,\varsigma_{\mathsf{hi}})$, with dashed lines at $\mathbb{D}=\bar{r}^2$.

\begin{figure}[!ht]
    \centering
    \includegraphics[width=\linewidth]{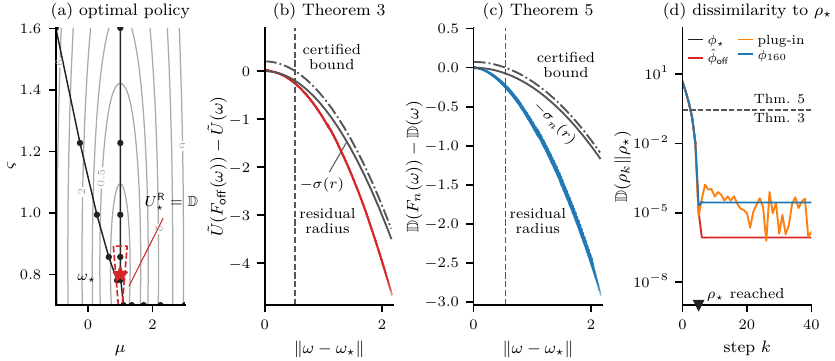}
    \caption{Certificates in the example.}
    \label{fig:toy_certificates}
\end{figure}

In \namedref{Figure}{fig:toy_certificates}a and~d, $\phi_\star$ reaches $\rho_\star$ in a few steps, and, as \namedref{Table}{tab:toy_hypotheses} shows, $\tilde{U}$ stays within the bounds of \namedref{Lemma}{lemma:toy_facts}(iii), well below the upper one.

For \namedref{Theorem}{theo:stability_finite_samples}, $e_f$, $\bar{\Delta}$, and $\bar{r}$ shrink as the data grow; with the least data, some runs fail, so $\bar{\Delta}=\infty$ and there is no certificate.
In \namedref{Figure}{fig:toy_certificates}b, the true decrease stays below the certified bound at every test point.
The smallest residual is orders of magnitude below $\bar{\Delta}$, and the limit distance is far below $\bar{r}$.

Since $\hat{\bm{\mathrm{a}}}\neq\bm{\mathrm{a}}_\star$, the offset $\|F_{\mathsf{off}}(\omega_\star)-\omega_\star\|$ is positive, so the fixed-point hypothesis of \namedref{Theorem}{theo:stability_finite_samples} fails; its proof, however, derives~\eqref{eq:true_lyapunov_decrease_with_residual} without it.
Every learned closed loop settles at a positive distance from $\omega_\star$, so the asymptotic clauses of \namedref{Theorems}{theo:stability_finite_samples}--\ref{theo:model_free_online_policy_learning} do not apply, and $J^{\mathsf{R}}$ of \namedref{Proposition}{prop:asymptotic_recovery_rotated_optimizers} is infinite for the learned $\theta$; the example reflects the proposition through the decrease of $e_f$.

For \namedref{Theorem}{theo:online_plugin_stability}, $\Delta_{\mathsf{on}}$ is small compared with $\bar{\Delta}_{\mathsf{off}}$, so the online residual set is barely larger than the offline one.
The estimation event holds in all but a few runs, consistent with $\alpha_{\mathsf{on}}$, and every step of these runs satisfies~\eqref{eq:online_true_lyapunov_decrease_residual}.
In \namedref{Figure}{fig:toy_certificates}d, the plug-in closed loop stays far inside the residual set, although estimation noise keeps it farther from $\omega_\star$ than the exact-parameter one.

For \namedref{Theorem}{theo:model_free_online_policy_learning}, $\bar{\Delta}_{\mathsf{dec},n}$ and $\bar{r}$ shrink as $N_n$ grows.
In \namedref{Figure}{fig:toy_certificates}c, the true decrease stays below the certified bound.
At equal numbers of transitions, the radius of \namedref{Theorem}{theo:model_free_online_policy_learning} is smaller than that of \namedref{Theorem}{theo:stability_finite_samples} with little data and about equal to it with the most, although the learner of \namedref{Theorem}{theo:model_free_online_policy_learning} uses more samples per step and the two certificates use different Lyapunov functions and margins.
The uncapped estimate violates~\eqref{eq:model_free_empirical_decrease_certificate} in most runs with the least data and in few runs with the most, always at points well inside the residual set.
From $N_n=200$ on, the smallest residual of the deployed run comes from the plug-in steps, and both this residual and the run's limit distance change little, whereas the limit distance over all runs keeps shrinking.
The policy $\phi_n$ settles far inside the residual set, although farther from $\omega_\star$ than $\hat{\phi}_{\mathsf{off}}$.

\section{Implementation Details}
\label{appendix:implementation}

\namedref{Figure}{fig:implementation_pipeline} shows the closed loop of \namedref{Section}{sec:practical_implementation}, and \namedref{Table}{tab:methods} lists the compared methods.

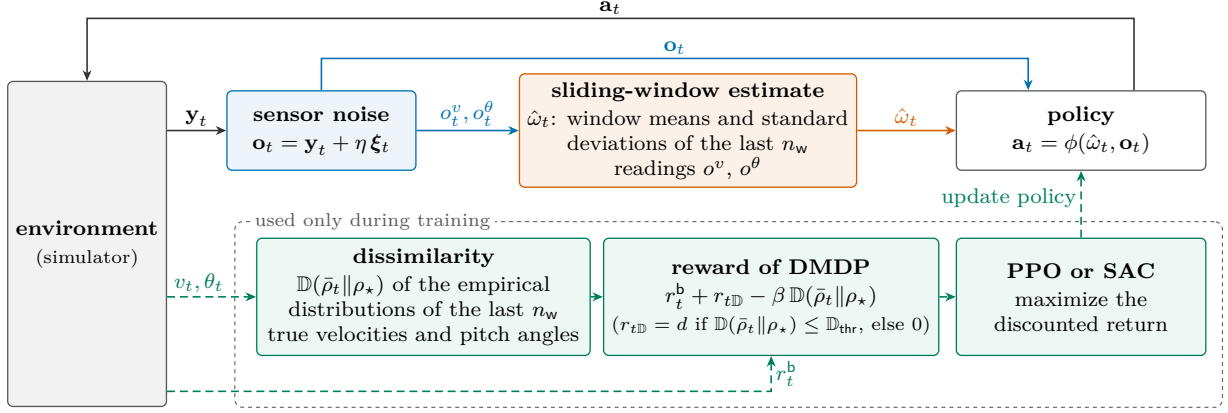
\begin{figure}[t]
    \centering
    \definecolor{figState}{HTML}{0072B2}
    \definecolor{figDist}{HTML}{D55E00}
    \definecolor{figGreen}{HTML}{009E73}
    \definecolor{figGray}{HTML}{595959}
    \begin{tikzpicture}[
        font=\footnotesize,
        >={Stealth[length=4pt,width=3.4pt]},
        box/.style={rounded corners=2pt, line width=0.5pt, align=center, inner sep=3pt, draw=figGray, fill=white, minimum height=1.05cm},
        obsbox/.style={box, draw=figState, fill=figState!7},
        estbox/.style={box, draw=figDist, fill=figDist!9},
        trainbox/.style={box, draw=figGreen!80!black, fill=figGreen!7, minimum height=1.55cm},
        flow/.style={->, line width=0.6pt, draw=black!80},
        train/.style={->, line width=0.6pt, draw=figGreen!80!black, dash pattern=on 3pt off 1.8pt},
        lab/.style={font=\footnotesize, inner sep=1.5pt},
    ]
        \node[box, fill=black!5, minimum width=2.1cm, minimum height=4.3cm] (env) at (1.05,-1.5) {\textbf{environment}\\[1pt] {\scriptsize (simulator)}};
        \node[obsbox, minimum width=2.5cm] (noise) at (4.15,0) {\textbf{sensor noise}\\[1pt] $\bm{\mathrm{o}}_t=\bm{\mathrm{y}}_t+\eta\,\bm{\xi}_t$};
        \node[estbox, minimum width=3.7cm] (est) at (9.0,0) {\textbf{sliding-window estimate}\\[1pt] $\hat{\omega}_t$: window means and standard\\ deviations of the last $n_{\mathsf{w}}$\\ readings $o^{v}$, $o^{\theta}$};
        \node[box, minimum width=3.3cm] (pol) at (14.2,0) {\textbf{policy}\\[1pt] $\bm{\mathrm{a}}_t=\phi(\hat{\omega}_t,\bm{\mathrm{o}}_t)$};
        \draw[flow] (env.east |- noise) -- node[lab, above] {$\bm{\mathrm{y}}_t$} (noise);
        \draw[flow, draw=figState] (noise) -- node[lab, above, text=figState] {$o^{v}_t,o^{\theta}_t$} (est);
        \draw[flow, draw=figDist] (est) -- node[lab, above, text=figDist] {$\hat{\omega}_t$} (pol);
        \draw[flow, draw=figState] (noise.north) -- ++(0,0.42) -| node[lab, pos=0.25, above, text=figState] {$\bm{\mathrm{o}}_t$} ([xshift=-0.7cm]pol.north);
        \draw[flow] ([xshift=0.7cm]pol.north) -- ++(0,0.95) -| node[lab, pos=0.25, above] {$\bm{\mathrm{a}}_t$} (env.north);
        \node[trainbox, minimum width=4.1cm, text width=4.2cm] (dis) at (5.5,-2.2) {\textbf{dissimilarity}\\[1pt] $\mathbb{D}(\bar{\rho}_t\|\rho_\star)$ of the empirical\\ distributions of the last $n_{\mathsf{w}}$\\ true velocities and pitch angles};
        \node[trainbox, minimum width=3.3cm] (shaped) at (10.1,-2.2) {\textbf{reward of DMDP}\\[1pt] $r^{\mathsf{b}}_t+r_{t\mathbb{D}}-\beta\,\mathbb{D}(\bar{\rho}_t\|\rho_\star)$\\[1pt] {\scriptsize ($r_{t\mathbb{D}}=d$ if $\mathbb{D}(\bar{\rho}_t\|\rho_\star)\leq\mathbb{D}_{\mathsf{thr}}$, else $0$)}};
        \node[trainbox, minimum width=3.3cm] (learner) at (14.2,-2.2) {\textbf{PPO or SAC}\\[1pt] maximize the\\ discounted return};
        \draw[train] (env.east |- dis) -- (dis);
        \node[lab, anchor=south west, text=figGreen!80!black] at ([xshift=1pt]env.east |- dis) {$v_t,\theta_t$};
        \draw[train] (dis) -- (shaped);
        \draw[train] (shaped) -- (learner);
        \coordinate (bus) at (env.east |- 0,-3.45);
        \draw[train] (bus) -| node[lab, right, pos=0.75, text=figGreen!80!black] {$r^{\mathsf{b}}_t$} (shaped.south);
        \draw[train] (learner.north) -- node[lab, left, pos=0.6, text=figGreen!80!black] {update policy} (pol.south);
        \begin{scope}[on background layer]
            \coordinate (busL) at ([xshift=-2pt]dis.west |- 0,-3.45);
            \node[draw=figGray, line width=0.4pt, rounded corners=3pt, dash pattern=on 1.5pt off 1.5pt, fit=(dis)(shaped)(learner)(busL), inner sep=6pt] (trainframe) {};
        \end{scope}
        \node[font=\scriptsize, text=figGray, fill=white, inner sep=1.5pt, anchor=west] at ([xshift=6pt]trainframe.north west) {used only during training};
    \end{tikzpicture}
    \caption{Practical implementation of DMDP.}
    \label{fig:implementation_pipeline}
\end{figure}

\begin{table}[tb]
    \centering
    \caption{Compared methods.}
    \label{tab:methods}
    \small
    \begin{tabular}{@{}lllll@{}}
        \toprule
        Method & Policy input & Reward & Reward critic & Target \\
        \midrule
        PPO / SAC & $\bm{\mathrm{o}}_t$ & $r^{\mathsf{b}}_t$ & expected return & -- \\
        PPO-Lag / SAC-Lag & $\bm{\mathrm{o}}_t$ & $r^{\mathsf{b}}_t-\lambda c_t$ & expected return & -- \\
        \addlinespace
        Distributional PPO / SAC & $\bm{\mathrm{o}}_t$ & $r^{\mathsf{b}}_t$ & $32$ quantiles of the return & -- \\
        Distributional PPO-Lag / SAC-Lag & $\bm{\mathrm{o}}_t$ & $r^{\mathsf{b}}_t-\lambda c_t$ & $32$ quantiles of the return & -- \\
        \addlinespace
        Implicit DMDP (PPO/SAC) & $\hat{\omega}_t$ and $\bm{\mathrm{o}}_t$ & $r^{\mathsf{b}}_t$ & $32$ quantiles of the return & -- \\
        Implicit DMDP-Lag (PPO/SAC) & $\hat{\omega}_t$ and $\bm{\mathrm{o}}_t$ & $r^{\mathsf{b}}_t-\lambda c_t$ & $32$ quantiles of the return & -- \\
        \addlinespace
        DMDP (PPO/SAC) & $\hat{\omega}_t$ and $\bm{\mathrm{o}}_t$ & $r^{\mathsf{b}}_t+r_{t\mathbb{D}}-\beta\,\mathbb{D}_t$ & expected return & $\rho_\star$ \\
        DMDP-Lag (PPO/SAC) & $\hat{\omega}_t$ and $\bm{\mathrm{o}}_t$ & $r^{\mathsf{b}}_t+r_{t\mathbb{D}}-\beta\,\mathbb{D}_t-\lambda c_t$ & expected return & $\rho_\star$ \\
        \bottomrule
    \end{tabular}
\end{table}

The regulated distribution $\rho_k$ is the joint distribution of the forward velocity $v_k$ and the torso pitch $\theta_k$, with marginals $\rho^{v}_k$ and $\rho^{\theta}_k$.
The parameterized family is a product of two uniform location-scale families,
\begin{equation}
    \rho_\omega=U\big[\mu^{v}-\sqrt{3}\varsigma^{v},\mu^{v}+\sqrt{3}\varsigma^{v}\big]\otimes U\big[\mu^{\theta}-\sqrt{3}\varsigma^{\theta},\mu^{\theta}+\sqrt{3}\varsigma^{\theta}\big],
\end{equation}
where $\omega=(\mu^{v},\varsigma^{v},\mu^{\theta},\varsigma^{\theta})$ collects the means and standard deviations of both features.
\namedref{Table}{tab:task} gives the target $\rho_\star=\rho_{\omega_\star}$, whose velocity marginal $\rho^{v}_\star$ is supported on the velocity band $[0.70\,v_{\mathsf{lim}},0.95\,v_{\mathsf{lim}}]$.

\begin{table}[tb]
    \centering
    \caption{Task.}
    \label{tab:task}
    \small
    \begin{tabular}{@{}ll@{}}
        \toprule
        Environment & \texttt{SafetyHalfCheetahVelocity-v1}~\citep{Ji2023}, MuJoCo~\citep{Todorov2012}; $H=1000$ \\
        $\bm{\mathrm{y}}_t$, $\bm{\mathrm{a}}_t$ & $\bm{\mathrm{y}}_t\in\mathbb{R}^{17}$; $\bm{\mathrm{a}}_t\in[-1,1]^{6}$ \\
        $v_t$, $\theta_t$ & torso velocity from $\bm{\mathrm{s}}_t$ to $\bm{\mathrm{s}}_{t+1}$; torso pitch of $\bm{\mathrm{s}}_{t+1}$ in $(-\pi,\pi]$ \\
        $r^{\mathsf{b}}_t$, $c_t$ & $v_t-0.1\|\bm{\mathrm{a}}_t\|^2$; $\mathbf{1}[v_t>v_{\mathsf{lim}}]$, $v_{\mathsf{lim}}=3.2096$ \\
        $\rho_\star$ & $U[0.70\,v_{\mathsf{lim}},0.95\,v_{\mathsf{lim}}]\otimes U[-0.2,0.2]$; $\omega_\star\approx(2.65,0.23,0,0.12)$ \\
        $\xi_{t,i}$, $\eta$ & i.i.d.\ $\tfrac{1}{2}\mathcal{N}(-0.8,0.6^2)+\tfrac{1}{2}\mathcal{N}(0.8,0.6^2)$; training: $\eta\sim U[0,0.3]$ per episode \\
        $n_{\mathsf{w}}$ & $10$ \\
        $r^{\mathsf{eva}}_t$ & $r^{\mathsf{b}}_t+r_{t\mathbb{D}}$ if $v_t\leq v_{\mathsf{lim}}$; $-v_{\mathsf{lim}}$ if $v_t>v_{\mathsf{lim}}$ \\
        \bottomrule
    \end{tabular}
\end{table}

The pitch target prescribes an upright torso, because the benchmark leaves the posture free and, without noise, a flipped gait has a wide velocity distribution, not concentrated in the velocity band, as \namedref{Figure}{fig:sac_stable_laws} shows.
\namedref{Figure}{fig:halfcheetah} labels the torques $a_{t,1},\ldots,a_{t,6}$ at the hip, knee, and foot joints of the back and front legs and shows DMDP-Lag (PPO) on the left and a flipped run of SAC on the right, both at one step of a noise-free evaluation episode with seed $0$.
A run is flipped if more than half of its steps $k\geq500$ have $|\theta_k|>\pi/2$, as in the right panel; it is upright if its mean absolute pitch for $k\geq500$ is at most $0.2$ rad, and tilted otherwise.

\begin{figure}[tb]
    \centering
    \includegraphics[width=0.78\linewidth]{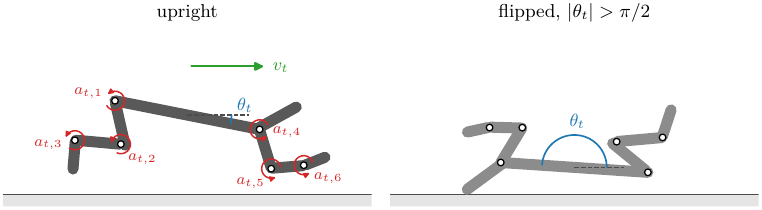}
    \caption{The HalfCheetah robot.}
    \label{fig:halfcheetah}
\end{figure}

The training dissimilarity sums the squared $2$-Wasserstein distances of the marginals,
\begin{equation}
    \label{eq:practical_w2}
    \mathbb{D}(\rho\|\rho_\star)=W_2^2\big(\rho^{v},\rho^{v}_\star\big)+W_2^2\big(\rho^{\theta},\rho^{\theta}_\star\big),
    \qquad
    W_2^2(\rho,\tilde{\rho})=\int_0^1\big(F^{-1}(u)-\tilde{F}^{-1}(u)\big)^2\,\mathsf{d}u,
\end{equation}
where $F^{-1}$ and $\tilde{F}^{-1}$ are the quantile functions of $\rho$ and $\tilde{\rho}$ on the real line~\citep{Villani2009}.
It satisfies \namedref{Definition}{def:dissimilarity_measure} and, on the family $\rho_\omega$, vanishes only at $\omega_\star$.

The window dissimilarity is $\mathbb{D}_t=\mathbb{D}(\bar{\rho}_t\|\rho_\star)$, with $\bar{\rho}_t$ the empirical distribution of the true velocities and pitch angles of the last $n_{\mathsf{w}}$ steps.
The threshold reward is $r_{t\mathbb{D}}=d\,\mathbf{1}[\mathbb{D}_t\leq\mathbb{D}_{\mathsf{thr}}]$ with a positive threshold $\mathbb{D}_{\mathsf{thr}}$, because even a window drawn from $\rho_\star$ has a positive dissimilarity.
The penalty $\beta\,\mathbb{D}_t$ gives a learning signal far from the target, where $r_{t\mathbb{D}}=0$ gives none.

The method-of-moments estimate $\hat{\omega}_t$ replaces the maximum-likelihood step of \namedref{Appendix}{appendix:omega_estimation_procedure} and uses the entries $o^{v}_t$ and $o^{\theta}_t$ of $\bm{\mathrm{o}}_t$, which read the instantaneous forward velocity and pitch angle at the start of step $t$.
At the start of an episode, the window of $\hat{\omega}_t$ is filled with the first reading, and that of $\bar{\rho}_t$ holds only the steps so far.

\namedref{Table}{tab:hyperparameters} gives the hyperparameters.

\begin{table}[!ht]
    \centering
    \caption{Hyperparameters.}
    \label{tab:hyperparameters}
    \footnotesize
    \begin{tabular}{@{}lll@{}}
        \toprule
        & PPO & SAC \\
        \midrule
        Steps & $1464\times8\times256\approx3\cdot10^{6}$ & $10^{6}$; $4$ envs; $1$ update/step; $5000$ random \\
        Batches & $10$ epochs $\times$ $32$ minibatches & buffer $10^{6}$; minibatch $256$ \\
        Networks & $2\times64$, $\tanh$; normalized inputs & $2\times256$, ReLU \\
        Policy & Gaussian, $\log\sigma_0=-1$ & $\tanh$-Gaussian, $\log\sigma\in[-5,2]$ \\
        Critics & $V_r$, $V_c$; GAE $0.95$~\citep{Schulman2016} & $Q_1$, $Q_2$, $Q_c$; Polyak $0.005$ \\
        Actor & clip $0.2$ on $\frac{\hat{A}^{r}_t-\lambda\hat{A}^{c}_t}{1+\lambda}$~\citep{Schulman2017,Ray2019} & $\frac{\alpha\log\pi-\min_{i}Q_i+\lambda Q_c}{1+\lambda}$~\citep{Ha2020} \\
        Coefficients & entropy $0$; value $0.5$; grad.\ norm $0.5$ & $\alpha_0=1$; target entropy $-6$ \\
        Learning rate & Adam, $3\times10^{-4}$, linear decay & Adam, $3\times10^{-4}$ \\
        $\lambda$ update & per update; $\hat{J}_c$: $32$ episodes & per $2048$ steps; $\hat{J}_c$: $16$ episodes \\
        Quantile critic & $V_r$; quantile Huber, $\kappa=1$~\citep{DabneyQR2018} & $Q_1$, $Q_2$; quantile Huber, $\kappa=1$ \\
        \midrule
        & \multicolumn{2}{l@{}}{PPO and SAC} \\
        \midrule
        $\gamma$ & \multicolumn{2}{l@{}}{$0.99$; bootstrap at truncation~\citep{Pardo2018}} \\
        Reward scale & \multicolumn{2}{l@{}}{$\div$ running std.\ of the discounted return; clip $[-10,10]$} \\
        $\lambda$ & \multicolumn{2}{l@{}}{Adam ($0.035$) on $-\lambda(\hat{J}_c-\bar{c})$; $\lambda\geq0$} \\
        $\beta$, $d$, $\mathbb{D}_{\mathsf{thr}}$, $\bar{c}$ & \multicolumn{2}{l@{}}{$0.5$, $10$, $0.1$, $10$} \\
        Pilot & \multicolumn{2}{l@{}}{PPO, other seeds, $\mathbb{D}_{\mathsf{thr}}=0.14$: $\beta\in\{0.5,1,2\}$, $\bar{c}\in\{10,5,2.5\}$} \\
        \bottomrule
    \end{tabular}
\end{table}

\FloatBarrier

\section{Additional Results}
\label{appendix:additional_results}

This appendix contains additional results with SAC and all results with PPO.
Unless stated otherwise, values and curves are means over the five seeds, and $\pm$ values and bands are standard errors.
A difference between two methods is clear if it exceeds the larger of their two standard errors; otherwise, it is within the seed variation.
The end-of-training values use the training episodes of the last $10\%$ of the steps, with stochastic actions, and the evaluation values use the deterministic evaluation episodes at the stated noise level.
The $10$--$90\%$ range of $R$ is the difference between the $10$th and $90$th percentiles of $R$ pooled over the seeds, and the line $k_0$ in the stability curves marks $k=500$.

\subsection{Additional Results with SAC}
\label{appendix:sac_results}

\namedref{Figure}{fig:sac_training_curves} shows the training curves discussed in \namedref{Section}{sec:exp_sac}.

\begin{figure}[!h]
    \centering
    \includegraphics[width=\linewidth]{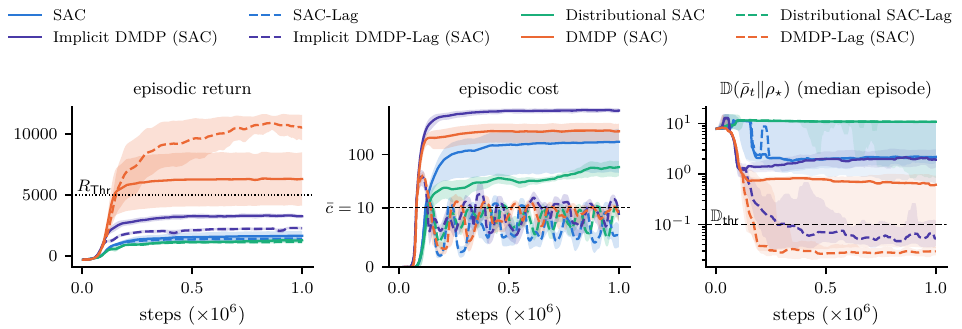}
    \caption{Training curves with SAC.}
    \label{fig:sac_training_curves}
\end{figure}

\namedref{Table}{tab:sac_shift} shows that sensor noise shifts the stable distributions of all methods, the more so the larger $\eta$.
The shift is the $1$-Wasserstein distance $\mathbb{D}(\hat{\rho}_\eta\|\hat{\rho}_0)$ between the stable velocity distribution $\hat{\rho}_\eta$ at noise level $\eta$ and the noise-free one $\hat{\rho}_0$.
Distributional SAC-Lag, four of whose runs are flipped, shifts the least, together with DMDP-Lag (SAC) at $\eta=0.1$; up to $\eta=0.4$, the shift of DMDP-Lag (SAC) stays within its seed variation of the smallest shift.
Implicit DMDP (SAC) shifts clearly the most at every noise level.

\begin{table}[!h]
    \centering
    \caption{Shift of the stable distribution under sensor noise with SAC.}
    \label{tab:sac_shift}
    \small
    \begin{tabular}{lccccc}
\toprule
 & \multicolumn{5}{c}{$\mathbb{D}(\hat{\rho}_\eta\|\hat{\rho}_0)$ $\downarrow$} \\
\cmidrule(lr){2-6}
Method & $\eta=0.1$ & $\eta=0.2$ & $\eta=0.3$ & $\eta=0.4$ & $\eta=0.5$ \\
\midrule
SAC & $0.04\pm0.01$ & $0.12\pm0.03$ & $0.21\pm0.03$ & $0.48\pm0.13$ & $0.80\pm0.26$ \\
SAC-Lag & $0.04\pm0.02$ & $0.11\pm0.03$ & $0.21\pm0.06$ & $0.37\pm0.08$ & $0.57\pm0.11$ \\
Distributional SAC & $0.05\pm0.02$ & $0.13\pm0.05$ & $0.28\pm0.11$ & $0.40\pm0.15$ & $0.56\pm0.19$ \\
Distributional SAC-Lag & $\bm{0.03\pm0.01}$ & $\bm{0.09\pm0.03}$ & $\bm{0.17\pm0.05}$ & $\bm{0.27\pm0.08}$ & $\bm{0.39\pm0.10}$ \\
Implicit DMDP (SAC) & $0.20\pm0.07$ & $0.52\pm0.15$ & $0.89\pm0.19$ & $1.37\pm0.22$ & $2.10\pm0.16$ \\
Implicit DMDP-Lag (SAC) & $0.09\pm0.01$ & $0.21\pm0.03$ & $0.35\pm0.05$ & $0.55\pm0.08$ & $0.85\pm0.10$ \\
DMDP (SAC) & $0.06\pm0.03$ & $0.24\pm0.09$ & $0.48\pm0.17$ & $0.70\pm0.23$ & $0.99\pm0.25$ \\
DMDP-Lag (SAC) & $\bm{0.03\pm0.01}$ & $0.12\pm0.07$ & $0.25\pm0.13$ & $0.41\pm0.16$ & $0.65\pm0.17$ \\
\bottomrule
\end{tabular}

\end{table}

\subsection{Results with PPO}
\label{appendix:ppo_results}

The PPO results have the format of \namedref{Section}{sec:exp_sac}.
Each logged training value averages the $32$ (PPO) or $16$ (SAC) most recently completed training episodes, and the curves of \namedref{Figures}{fig:sac_training_curves} and~\ref{fig:training_curves} are smoothed over $10$ logged values; their window-dissimilarity panels show the median and interquartile range, over the recent episodes of all five seeds, of the per-episode median of $\mathbb{D}_t$.
\namedref{Figures}{fig:sac_stable_laws} and~\ref{fig:stable_laws} pool $v_k$ for $k\geq500$ without noise over all evaluation episodes and seeds, with the horizontal axis ending at $v=6$, and \namedref{Figures}{fig:sac_stability_curves} and~\ref{fig:stability_curves} average the stability distance over $10$ consecutive steps.

The stability distance compares estimates of $\rho_k$ and $\hat{\rho}$ from different halves of the evaluation episodes and is averaged over both assignments of the halves.
The stationary reference value is the mean stability distance for $k\geq500$ of $10$ surrogates with independent random circular shifts of the segment $k\geq500$ of every episode.
The ratio of the stability distance to this value is near $1$ for a stationary $\rho_k$ and well above $1$ if $\rho_k$ drifts or the run is synchronized.
A run is synchronized, i.e., its episodes keep a common phase, if the standard deviation over $k\geq500$ of the mean velocity across its episodes exceeds twice the standard error of a mean of $64$ independent velocities.

In \namedref{Table}{tab:sampling_floor}, the ratio is close to $1$ for every method of both learners from $\eta=0.4$ on, and with SAC from $\eta=0.2$ on, except for Implicit DMDP (SAC) and DMDP (SAC), which remain well above $1$ at $\eta=0.2$.
Without noise, $10$ of the $40$ runs with PPO and $21$ of the $40$ runs with SAC, including all five runs of DMDP (SAC), are synchronized, so their $\rho_k$ oscillates and their episodes differ only in their initial states; at $\eta=0.1$, $14$ of the $21$ synchronized runs with SAC stay synchronized.

\begin{table}[tbp]
    \centering
    \caption{Ratio to the stationary reference value.}
    \label{tab:sampling_floor}
    \small
    \begin{tabular*}{\linewidth}{@{\extracolsep{\fill}}l*{6}{c}@{}}
\toprule
Method & $\eta=0$ & $\eta=0.1$ & $\eta=0.2$ & $\eta=0.3$ & $\eta=0.4$ & $\eta=0.5$ \\
\midrule
\multicolumn{7}{@{}l}{\emph{with PPO}} \\
PPO & $1.32$ & $1.05$ & $1.12$ & $1.05$ & $1.00$ & $1.02$ \\
PPO-Lag & $1.62$ & $1.26$ & $1.04$ & $1.05$ & $1.02$ & $1.00$ \\
Distributional PPO & $1.10$ & $1.03$ & $1.04$ & $1.07$ & $1.04$ & $1.03$ \\
Distributional PPO-Lag & $1.27$ & $1.05$ & $1.04$ & $1.05$ & $1.04$ & $1.04$ \\
Implicit DMDP (PPO) & $1.37$ & $1.32$ & $1.10$ & $1.09$ & $1.03$ & $1.06$ \\
Implicit DMDP-Lag (PPO) & $1.29$ & $1.04$ & $1.07$ & $1.11$ & $1.05$ & $1.00$ \\
DMDP (PPO) & $1.06$ & $1.03$ & $1.00$ & $1.04$ & $1.01$ & $1.00$ \\
DMDP-Lag (PPO) & $1.23$ & $1.01$ & $1.05$ & $1.04$ & $1.03$ & $1.01$ \\
\midrule
\multicolumn{7}{@{}l}{\emph{with SAC}} \\
SAC & $1.43$ & $1.16$ & $0.92$ & $1.07$ & $1.02$ & $0.98$ \\
SAC-Lag & $1.13$ & $1.04$ & $1.06$ & $0.99$ & $0.96$ & $0.97$ \\
Distributional SAC & $1.11$ & $1.08$ & $0.92$ & $1.02$ & $1.07$ & $1.00$ \\
Distributional SAC-Lag & $1.29$ & $1.14$ & $1.05$ & $0.95$ & $0.96$ & $1.00$ \\
Implicit DMDP (SAC) & $2.58$ & $2.46$ & $1.88$ & $1.31$ & $1.02$ & $1.00$ \\
Implicit DMDP-Lag (SAC) & $1.15$ & $1.13$ & $1.03$ & $0.99$ & $1.00$ & $1.00$ \\
DMDP (SAC) & $3.79$ & $3.60$ & $2.40$ & $1.11$ & $0.96$ & $0.97$ \\
DMDP-Lag (SAC) & $1.70$ & $1.40$ & $1.01$ & $0.99$ & $1.00$ & $1.02$ \\
\bottomrule
\end{tabular*}

\end{table}

\namedref{Table}{tab:main_results} shows that, without noise, DMDP-Lag (PPO) is the best of all methods with PPO: its stable distribution is clearly the closest to the target, and it has the highest evaluation return and the largest fraction of steps with $\mathbb{D}_t\leq\mathbb{D}_{\mathsf{thr}}$, within the seed variation of PPO-Lag in the return and of PPO-Lag and DMDP (PPO) in the fraction.
At $\eta=0.2$, Distributional PPO-Lag leads in these three columns, with PPO-Lag and DMDP (PPO) within its seed variation in the target distance and DMDP-Lag (PPO) clearly farther; its return and fraction lie within the seed variation of those of the DMDP methods and the other two constrained baselines.
The cost constraint brings every baseline clearly closer to the target at both noise levels, and DMDP (PPO) is within the seed variation of the constrained baselines even without it.

\begin{table}[!h]
    \centering
    \caption{Evaluation results with PPO.}
    \label{tab:main_results}
    \resizebox{\linewidth}{!}{\begin{tabular}{llcccccc}
\toprule
Method & $\eta$ & \begin{tabular}[b]{@{}c@{}}return \\ $r^{\mathsf{eva}}_t$ $\uparrow$\end{tabular} & cost $\downarrow$ & \begin{tabular}[b]{@{}c@{}}$\mathbb{D}(\rho_k\|\hat{\rho})$ \\ $k<200$ $\downarrow$\end{tabular} & \begin{tabular}[b]{@{}c@{}}$\mathbb{D}(\rho_k\|\hat{\rho})$ \\ $k\geq500$ $\downarrow$\end{tabular} & \begin{tabular}[b]{@{}c@{}}$\mathbb{D}(\hat{\rho}\|\rho^{v}_\star)$ \\ $\downarrow$\end{tabular} & \begin{tabular}[b]{@{}c@{}}$\mathbb{P}[\mathbb{D}(\bar{\rho}_t\|\rho_\star)$ \\ $\leq\mathbb{D}_{\mathsf{thr}}]$ $\uparrow$\end{tabular} \\
\midrule
PPO & $0$ & $-2207\pm468$ & $788.2\pm93.6$ & $0.92\pm0.07$ & $0.355\pm0.113$ & $1.86\pm0.32$ & $0.00\pm0.00$ \\
 & $0.2$ & $-454\pm445$ & $397.7\pm109.9$ & $0.92\pm0.19$ & $0.331\pm0.033$ & $1.33\pm0.24$ & $0.01\pm0.00$ \\
\midrule
PPO-Lag & $0$ & $9856\pm1966$ & $\bm{1.9\pm1.4}$ & $0.49\pm0.03$ & $0.174\pm0.055$ & $0.21\pm0.02$ & $0.78\pm0.19$ \\
 & $0.2$ & $6122\pm1340$ & $10.0\pm2.2$ & $\bm{0.46\pm0.03}$ & $0.205\pm0.058$ & $0.52\pm0.13$ & $0.45\pm0.13$ \\
\midrule
Distributional & $0$ & $-2293\pm550$ & $813.5\pm113.6$ & $1.08\pm0.13$ & $0.293\pm0.038$ & $1.93\pm0.29$ & $0.00\pm0.00$ \\
PPO & $0.2$ & $-898\pm450$ & $448.4\pm106.1$ & $1.29\pm0.21$ & $0.504\pm0.044$ & $1.68\pm0.18$ & $0.00\pm0.00$ \\
\midrule
Distributional & $0$ & $9055\pm1803$ & $3.6\pm2.9$ & $\bm{0.42\pm0.04}$ & $0.087\pm0.025$ & $0.24\pm0.04$ & $0.67\pm0.20$ \\
PPO-Lag & $0.2$ & $\bm{6926\pm1160}$ & $11.7\pm1.7$ & $0.48\pm0.05$ & $\bm{0.137\pm0.010}$ & $\bm{0.47\pm0.17}$ & $\bm{0.50\pm0.12}$ \\
\midrule
Implicit DMDP & $0$ & $-2563\pm30$ & $864.3\pm4.7$ & $1.32\pm0.08$ & $0.564\pm0.117$ & $2.19\pm0.08$ & $0.00\pm0.00$ \\
(PPO) & $0.2$ & $-225\pm416$ & $309.3\pm47.2$ & $1.41\pm0.18$ & $0.507\pm0.051$ & $1.68\pm0.15$ & $0.00\pm0.00$ \\
\midrule
Implicit DMDP-Lag & $0$ & $7841\pm2044$ & $2.0\pm1.8$ & $0.45\pm0.05$ & $\bm{0.075\pm0.031}$ & $0.32\pm0.07$ & $0.55\pm0.23$ \\
(PPO) & $0.2$ & $5929\pm1240$ & $\bm{9.1\pm1.1}$ & $0.54\pm0.08$ & $0.147\pm0.042$ & $0.79\pm0.28$ & $0.42\pm0.11$ \\
\midrule
DMDP (PPO) & $0$ & $8730\pm2326$ & $122.4\pm83.3$ & $0.51\pm0.06$ & $0.135\pm0.045$ & $0.30\pm0.17$ & $0.78\pm0.20$ \\
 & $0.2$ & $5655\pm1677$ & $75.0\pm30.9$ & $0.51\pm0.02$ & $0.212\pm0.016$ & $0.60\pm0.12$ & $0.46\pm0.16$ \\
\midrule
DMDP-Lag (PPO) & $0$ & $\bm{11417\pm350}$ & $4.0\pm3.2$ & $0.53\pm0.06$ & $0.125\pm0.042$ & $\bm{0.13\pm0.05}$ & $\bm{0.97\pm0.02}$ \\
 & $0.2$ & $6021\pm708$ & $16.2\pm3.7$ & $0.52\pm0.02$ & $0.213\pm0.014$ & $0.66\pm0.08$ & $0.46\pm0.07$ \\
\bottomrule
\end{tabular}

}
\end{table}

At both noise levels, the unconstrained baselines spend many steps above $v_{\mathsf{lim}}$, as their costs show, so their evaluation returns are negative and clearly below that of DMDP (PPO), whose cost is clearly lower than theirs but well above $\bar{c}$.
Without noise, the costs of the -Lag methods lie below $\bar{c}$ and within the seed variation of one another, the lowest being that of PPO-Lag; at $\eta=0.2$, Implicit DMDP-Lag (PPO) has the lowest cost of all methods, whereas DMDP-Lag (PPO) exceeds $\bar{c}$, clearly above the constrained baselines.

Feeding back $\hat{\omega}_t$ without $\rho_\star$ clearly lowers the cost at $\eta=0.2$, with and without the cost constraint, but does not bring the stable distribution closer to the target.
Knowing $\rho_\star$ does: DMDP (PPO) is clearly closer than Implicit DMDP (PPO) at both noise levels, and DMDP-Lag (PPO) than Implicit DMDP-Lag (PPO) without noise.

Without noise, every run of DMDP (PPO) is upright, although its run with seed $2$ never stays in the velocity band; one run of DMDP-Lag (PPO) is tilted.
At $\eta=0.2$, each unconstrained baseline has one flipped run, whereas the DMDP methods and the constrained baselines have none.

In \namedref{Figure}{fig:stable_laws}, the unconstrained baselines run mostly above $v_{\mathsf{lim}}$, the distributions of the constrained baselines peak at the upper edge of the velocity band, just below $v_{\mathsf{lim}}$, and those of the DMDP methods lie mostly within and spread over the band, with a small tail above $v_{\mathsf{lim}}$ for DMDP (PPO).

\begin{figure}[!h]
    \centering
    \includegraphics[width=\linewidth]{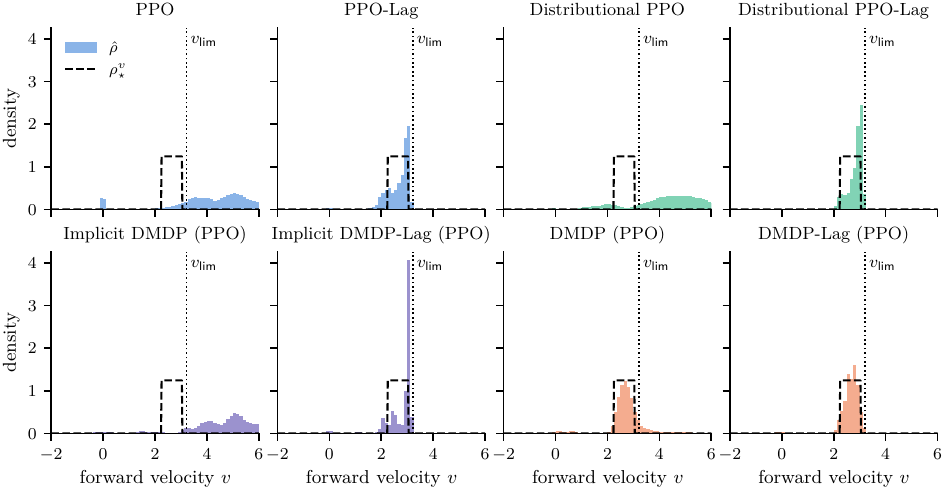}
    \caption{Stable velocity distributions with PPO.}
    \label{fig:stable_laws}
\end{figure}

In \namedref{Figure}{fig:training_curves}, the returns of the DMDP methods rise above $R_{\mathsf{Thr}}$ within the first $10^{6}$ training steps and remain there, whereas those of the baselines stay far below it.
In the second half of training, the costs of the constrained baselines and DMDP-Lag (PPO) stay near $\bar{c}$, those of the unconstrained baselines far above it, and that of DMDP (PPO) in between, well above $\bar{c}$; in this period, the median window dissimilarities of the DMDP methods lie lowest, below $\mathbb{D}_{\mathsf{thr}}$, closely followed by that of Distributional PPO-Lag, whereas those of the unconstrained baselines stay more than an order of magnitude above it.

\begin{figure}[!h]
    \centering
    \includegraphics[width=\linewidth]{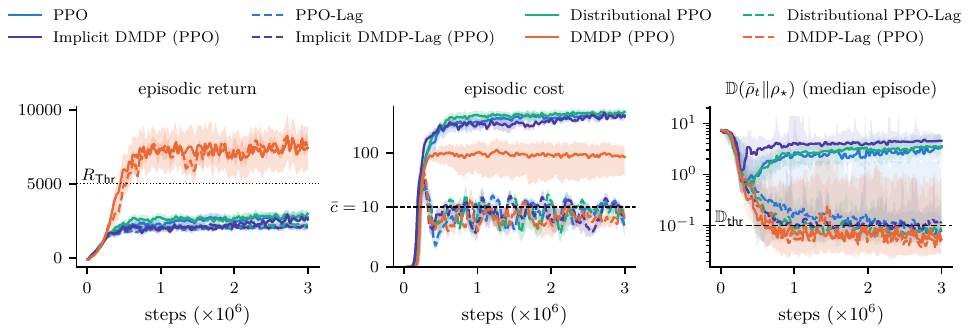}
    \caption{Training curves with PPO.}
    \label{fig:training_curves}
\end{figure}

In \namedref{Figure}{fig:stability_curves}, the stability distance of every method drops by about an order of magnitude within $100$ steps without noise; that of Distributional PPO then keeps decreasing over several hundred steps, whereas the others change much less.
In \namedref{Table}{tab:main_results}, Distributional PPO-Lag and Implicit DMDP-Lag (PPO) have the smallest stability distances for $k\geq500$ at both noise levels, within the seed variation of each other, and DMDP (PPO) has a clearly smaller one than the unconstrained baselines.

\begin{figure}[!h]
    \centering
    \includegraphics[width=\linewidth]{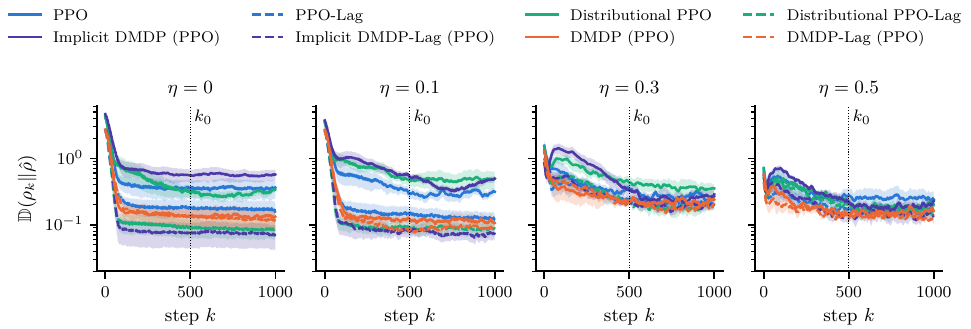}
    \caption{Stability curves with PPO.}
    \label{fig:stability_curves}
\end{figure}

In \namedref{Table}{tab:reliability}, DMDP-Lag (PPO) reaches $R_{\mathsf{Thr}}$ most often of all methods at the end of training and in the evaluation at both noise levels, within the seed variation of DMDP (PPO), and both reach it clearly more often than the baselines.
In each case, a DMDP method has the widest $10$--$90\%$ range of $R$.

\begin{table}[!h]
    \centering
    \caption{Reliability of the return with PPO.}
    \label{tab:reliability}
    \footnotesize\setlength{\tabcolsep}{3pt}
    \begin{tabular*}{\linewidth}{@{\extracolsep{\fill}}lcccccc@{}}
\toprule
 & \multicolumn{2}{c}{end of training} & \multicolumn{2}{c}{evaluation, $\eta=0$} & \multicolumn{2}{c}{evaluation, $\eta=0.2$} \\
\cmidrule(lr){2-3}\cmidrule(lr){4-5}\cmidrule(lr){6-7}
Method & $\mathbb{P}[R\geq R_{\mathsf{Thr}}]$ $\uparrow$ & 10--90\% range $\downarrow$ & $\mathbb{P}[R\geq R_{\mathsf{Thr}}]$ $\uparrow$ & 10--90\% range $\downarrow$ & $\mathbb{P}[R\geq R_{\mathsf{Thr}}]$ $\uparrow$ & 10--90\% range $\downarrow$ \\
\midrule
PPO & $0.00\pm0.00$ & $3653$ & $0.00\pm0.00$ & $2150$ & $0.00\pm0.00$ & $2828$ \\
PPO-Lag & $0.00\pm0.00$ & $\bm{1402}$ & $0.00\pm0.00$ & $592$ & $0.00\pm0.00$ & $\bm{1129}$ \\
Distributional PPO & $0.00\pm0.00$ & $3846$ & $0.00\pm0.00$ & $2437$ & $0.00\pm0.00$ & $3249$ \\
Distributional PPO-Lag & $0.00\pm0.00$ & $1535$ & $0.00\pm0.00$ & $\bm{383}$ & $0.00\pm0.00$ & $1496$ \\
Implicit DMDP (PPO) & $0.00\pm0.00$ & $3795$ & $0.02\pm0.02$ & $753$ & $0.00\pm0.00$ & $3140$ \\
Implicit DMDP-Lag (PPO) & $0.00\pm0.00$ & $2182$ & $0.00\pm0.00$ & $437$ & $0.00\pm0.00$ & $2091$ \\
DMDP (PPO) & $0.63\pm0.17$ & $10410$ & $0.78\pm0.20$ & $9003$ & $0.55\pm0.21$ & $9032$ \\
DMDP-Lag (PPO) & $\bm{0.68\pm0.05}$ & $10691$ & $\bm{0.97\pm0.02}$ & $691$ & $\bm{0.66\pm0.14}$ & $6713$ \\
\bottomrule
\end{tabular*}

\end{table}

In \namedref{Table}{tab:shift}, sensor noise shifts the stable distributions of all methods, the more so the larger $\eta$.
At every noise level, PPO-Lag shifts the least, and DMDP (PPO) is within its seed variation; the two DMDP methods shift within the seed variation of each other.

\begin{table}[!h]
    \centering
    \caption{Shift of the stable distribution under sensor noise with PPO.}
    \label{tab:shift}
    \small
    \begin{tabular}{lccccc}
\toprule
 & \multicolumn{5}{c}{$\mathbb{D}(\hat{\rho}_\eta\|\hat{\rho}_0)$ $\downarrow$} \\
\cmidrule(lr){2-6}
Method & $\eta=0.1$ & $\eta=0.2$ & $\eta=0.3$ & $\eta=0.4$ & $\eta=0.5$ \\
\midrule
PPO & $0.58\pm0.12$ & $1.85\pm0.20$ & $2.71\pm0.21$ & $3.23\pm0.29$ & $3.56\pm0.37$ \\
PPO-Lag & $\bm{0.09\pm0.01}$ & $\bm{0.53\pm0.10}$ & $\bm{1.32\pm0.19}$ & $\bm{1.82\pm0.20}$ & $\bm{2.05\pm0.14}$ \\
Distributional PPO & $0.95\pm0.27$ & $2.15\pm0.59$ & $2.94\pm0.52$ & $3.46\pm0.48$ & $3.70\pm0.47$ \\
Distributional PPO-Lag & $0.12\pm0.06$ & $0.57\pm0.18$ & $1.42\pm0.18$ & $2.11\pm0.11$ & $2.34\pm0.08$ \\
Implicit DMDP (PPO) & $0.94\pm0.32$ & $2.96\pm0.06$ & $3.82\pm0.18$ & $3.96\pm0.20$ & $4.12\pm0.17$ \\
Implicit DMDP-Lag (PPO) & $0.12\pm0.02$ & $0.78\pm0.24$ & $1.51\pm0.25$ & $2.12\pm0.10$ & $2.27\pm0.11$ \\
DMDP (PPO) & $0.12\pm0.04$ & $0.60\pm0.10$ & $1.46\pm0.19$ & $1.98\pm0.13$ & $2.18\pm0.08$ \\
DMDP-Lag (PPO) & $0.11\pm0.02$ & $0.64\pm0.06$ & $1.64\pm0.15$ & $2.03\pm0.09$ & $2.22\pm0.07$ \\
\bottomrule
\end{tabular}

\end{table}

\FloatBarrier

\end{document}